\documentclass[twoside]{article}

\usepackage[accepted]{aistats2022}

\usepackage{amsmath}
\usepackage{amssymb}
\usepackage{amsthm}
\usepackage{mathrsfs}
\usepackage{mathabx}
\usepackage{dsfont}
\usepackage{mathtools}
\usepackage{stmaryrd}

\usepackage{algorithm}
\usepackage{algorithmic}

\usepackage{booktabs}

\usepackage{microtype}

\usepackage[round]{natbib}

\DeclareMathOperator*{\Argmin}{arg\,min}% argmin
\DeclareMathOperator{\Exp}{exp}% normal exponential
\DeclareMathOperator{\Exps}{e}% small exponential
\DeclareMathOperator{\Expec}{\mathbb{E}}% symbol for expectation
\DeclareMathOperator{\Gaussian}{\mathcal{N}}% symbol for Gaussian random variable
\DeclareMathOperator{\Identity}{I}% identity matrix
\DeclareMathOperator{\Proba}{\mathbb{P}}% probability
\DeclareMathOperator{\Var}{Var}% symbole for the variance
\newcommand{\abs}[1]{\left\lvert#1\right\rvert}% absolute value
\newcommand{\smallabs}[1]{\lvert#1\rvert}% small absolute value

\newcommand{\card}[1]{\# #1}% card of a set without braces
\newcommand{\condexpec}[2]{\mathbb{E}\left[#1\middle|#2\right]}% conditional expectation
\newcommand{\condproba}[2]{\mathbb{P}\left(#1\middle|#2\right)}% conditional probability
\newcommand{\cov}[2]{\mathrm{Cov}(#1,#2)} % covariance of two random variables

\newcommand{\Diff}{\mathrm{d}}% differential operator
\newcommand{\defeq}{\vcentcolon =}% define equals
\renewcommand{\exp}[1]{\Exp\left(#1\right)}% exponential of something
\newcommand{\expec}[1]{\Expec\left[#1\right]}% expected value

\newcommand{\smallexpec}[1]{\Expec[#1]}% expected value, small version
\newcommand{\exps}[1]{\Exps^{#1}}% e to the power something
\newcommand{\gaussian}[2]{\Gaussian\left(#1,#2\right)}% Gaussian distribution
\newcommand{\Indic}{\mathds{1}}% indicator
\newcommand{\indic}[1]{\Indic_{#1}}% indicator function
\newcommand{\infnorm}[1]{\norm{#1}_{\infty}}% sup norm
\newcommand{\norm}[1]{\left\lVert#1\right\rVert}% norm of a vector
\newcommand{\proba}[1]{\Proba\left (#1\right )}% probability of an event
\newcommand{\smallproba}[1]{\Proba (#1)}% small proba
\newcommand{\Reals}{\mathbb{R}}% real numbers
\newcommand{\smallvar}[1]{\Var (#1)}% variance of a random variable
\newcommand{\var}[1]{\Var\left(#1\right)}% variance of a random variable
\newcommand{\bigo}[1]{\mathcal{O}\left(#1\right)}% big O notation
\renewcommand{\epsilon}{\varepsilon}

\theoremstyle{plain}
\newtheorem{theorem}{Theorem}[section]
\newtheorem{proposition}{Proposition}[section]
\newtheorem{lemma}{Lemma}[section]
\newtheorem{corollary}{Corollary}[section]

\theoremstyle{definition}
\newtheorem{assumption}{Assumption}[section]
\newtheorem{definition}{Definition}[section]
\newtheorem{remark}{Remark}[section]

\newcommand{\midim}{I_m}% middle of the image

\newcommand{\cc}{\mathcal{C}}% connected component

\newcommand{\graph}{\mathcal{G}}% the quickshift graph
\newcommand{\height}{H}% height of the image
\newcommand{\image}{I}% positions
\newcommand{\leftim}{I_{\ell}}% left part of the image in the bicolor case
\newcommand{\rightim}{I_r}% right part of the image
\newcommand{\normcst}{C}% normalization constant
\newcommand{\oc}{c}% the other color
\newcommand{\proj}{\widehat{P}} % hajek projection of the density estimate
\newcommand{\kernelsize}{k_s}% kernel size parameter
\newcommand{\kernelwidth}{k_w}% kernel width parameter
\newcommand{\maxdist}{d_m}% maximal distance between parents
\newcommand{\width}{W}% width of the image

\renewcommand{\arraystretch}{1.2}

\makeatletter
\def\hlinewd#1{%
	\noalign{\ifnum0=`}\fi\hrule \@height #1 %
	\futurelet\reserved@a\@xhline}
\makeatother

\makeatletter
\def\th@plain{%
	\thm@notefont{}% same as heading font
	\itshape % body font
}
\def\th@definition{%
	\thm@notefont{}% same as heading font
	\normalfont % body font
}
\makeatother

\newcommand{\ra}[1]{\renewcommand{\arraystretch}{#1}}

\begin{document}

\twocolumn[

\aistatstitle{How to scale hyperparameters for quickshift image segmentation}

\aistatsauthor{Damien Garreau}

\aistatsaddress{Universit\'e C\^ote d'Azur, Inria, CNRS, LJAD, France} ]

%%%%%%%%%%%%%%%%%%%%%%%%%%%%%%%%%%%%%%%%%%%%%%%%%%%%%%%%%%%%%%%%%%%%%%%%%%%%%%%%%

\begin{abstract}
Quickshift is a popular algorithm for image segmentation, used as a preprocessing step in many applications. 
Unfortunately, it is quite challenging to understand the hyperparameters' influence on the number and shape of superpixels produced by the method. 
In this paper, we study theoretically a slightly modified version of the quickshift algorithm, with a particular emphasis on homogeneous image patches with i.i.d. pixel noise and sharp boundaries between such patches. 
Leveraging this analysis, we derive a simple heuristic to scale quickshift hyperparameters with respect to the image size, which we check empirically. 
\end{abstract}

%%%%%%%%%%%%%%%%%%%%%%%%%%%%%%%%%%%%%%%%%%%%%%%%%%%%%%%%%%%%%%%%%%%%%%%%%%%%%%%%%

\section{INTRODUCTION}

Quickshift is a clustering algorithm, which is used in image processing to obtain superpixels. 
It proceeds by first computing a kernel density estimate at scale~$\kernelsize$, then connecting each pixel to the nearest neighbor with higher density. 
All connections further away than~$\maxdist$ are removed, and the superpixels are then defined as the connected components of the resulting graph. 
Proposed by \citet{vedaldi2008quick} as an efficient way to approximate the celebrated mean shift algorithm \citep{cheng1995mean,comaniciu2002mean}, the algorithm also appears in \citet{rodriguez2014clustering}. 
Together with SLIC \citep{achanta2012slic}, compact watershed \citep{neubert2014compact}, and Felzenszwalb \citep{felzenszwalb2004efficient}, it is often used as a preprocessing step in more complex computer vision tasks such as image compression \citep{compression2014} or semantic segmentation \citep{zhang2020semantic}. 
In the field of interpretability, quickshift is used as a default step when using LIME for images \citep{ribeiro2016should}, in order to create interpretable features. 

\begin{figure}[ht]
	\begin{center}
		\includegraphics[scale=0.22]{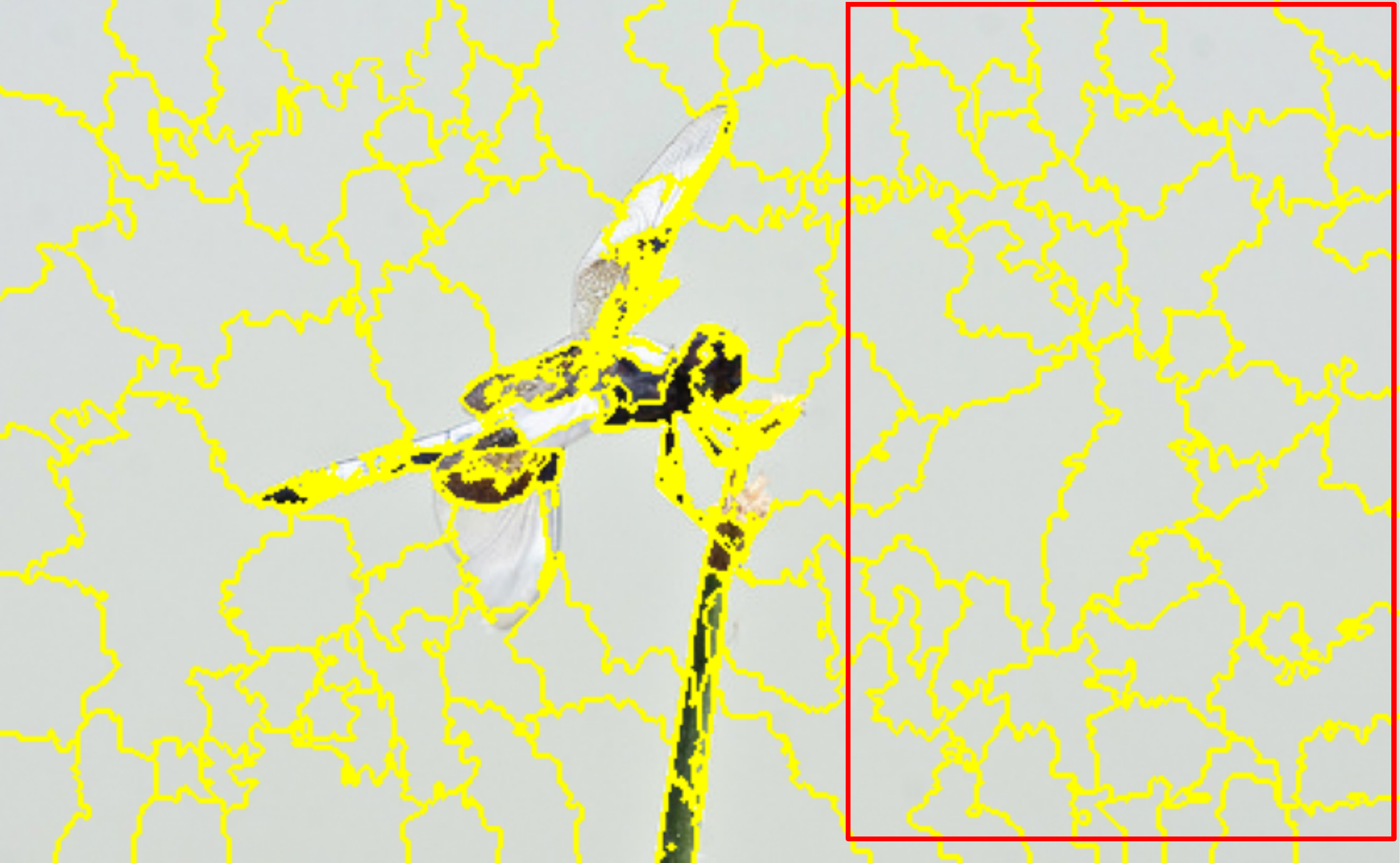}
	\end{center}
\vspace{-0.1in}
	\caption{\label{fig:motivation}Segmentation of an image with quickshift with default hyperparameters ($\kernelsize=5$, $\maxdist=10$). The superpixels are marked in yellow thick lines. The area inside the red rectangle is split up in $56$ superpixels despite being quite homogeneous. }
\end{figure}

In these applications, the number of superpixels produced by quickshift can be quite critical. 
For instance, in the case of LIME, the number of superpixels corresponds to the dimension of the linear model trained later on by the method \citep{garreau2021does}. 
Unfortunately, there is no way to know in advance the number of superpixels produced by quickshift as a function of the two main hyperparamters:~$\kernelsize$ and~$\maxdist$, in contrast for instance with SLIC. 
This can be particularly troublesome, in particular if the image at hand contains large, flat portions, where it is known that quickshift oversegments (see Figure~\ref{fig:motivation}).
In contrast, boundaries between patches of different colors are seemingly always well-identified. 

%%%%%%%%%%%%%%%%%%%%%%%%%%%%%%%%%%%%%%%%%%%%%%%%%%%%%%%%%%%%%%%%%%%%%%%%%%%%%%%%%

\subsection{Summary of the paper}

In this paper, we provide quantitative insights on the relationship between $\kernelsize$ and $\maxdist$, the two main hyperparameters of the method, and the number of superpixels produced when running the algorithm with this choice of hyperparameters. 
Our main theoretical contribution is the study of a modified version of quickshift on flat regions of the image. 
In that case, we show that this algorithm will find in average $\mathcal{A}/B(\kernelsize,\maxdist)$ local maxima, where $B(\cdot,\cdot)$ is $2$-homogeneous and $\mathcal{A}$ is the area of the flat region (Theorem~\ref{th:average-local-max}). 
Empirically, we demonstrate that these findings extend to the number of superpixels found by quickshift in real images. 
This yields the following heuristic: \textbf{for a given image, multiplying~$\kernelsize$ or~$\maxdist$ by $\kappa$ will roughly divide the number of superpixels by $\kappa^2$.}  
We also show how sharp boundaries between homogeneous patches are detected by quickshift (Theorem~\ref{th:decreasing-density-estimates-bicolor}). 

The paper is organized as follows: in Section~\ref{sec:quickshift}, we recall the detailed operation of quickshift. 
Section~\ref{sec:homogeneous} contains the analysis for the flat portions, Section~\ref{sec:bicolor} for boundaries. 
Only sketches of the proofs are provided, the complete version of the proofs can be found in the Appendix.  
Finally, we show some experimental results in Section~\ref{sec:experiments}.
All experiments are realized with the \texttt{scikit-image} implementation of quickshift and are publicly available.\footnote{\texttt{https://github.com/dgarreau/quickshift-scale}}

%%%%%%%%%%%%%%%%%%%%%%%%%%%%%%%%%%%%%%%%%%%%%%%%%%%%%%%%%%%%%%%%%%%%%%%%%%%%%%%%%%

\subsection{Related work}

Density-based clustering algorithms proceed by first computing a density estimate \citep{parzen1962estimation}, then either by looking at the level sets of the density or by hill-climbing. 
Quickshift belongs to the second category, and can be seen as a sample-based version of mean shift. 
This last method is analyzed in \citet{arias2016estimation}, which shows that the updates converge to the correct gradient steps. 
Regarding quickshift, \citet{jiang2017consistency} proves its consistency in the statistical sense, and gives the asymptotic size of $\kernelsize$ and $\maxdist$ to recover the underlying density when the number of points grows to infinity. 
Since the result is asymptotic, it is difficult to use to pick $\kernelsize$ and $\maxdist$ for a given image. 
Moreover, one of the main assumptions forbids the existence of flat regions in the theoretical analysis. 
A related analysis is proposed by \citet{verdinelli2018analysis}, with the same caveat. 
Also noting that density-based clustering algorithms tend to oversegment flat regions, \citet{jiang2018quickshift} proposes a way to solve this problem, but this comes at the cost of computing cluster cores \citep{jiang2017modal} of the density estimate. 

%%%%%%%%%%%%%%%%%%%%%%%%%%%%%%%%%%%%%%%%%%%%%%%%%%%%%%%%%%%%%%%%%%%%%%%%%%%%%%%%%

\section{QUICKSHIFT: A REFRESHER}
\label{sec:quickshift}

We now describe quickshift in more details, introducing notation along the way. 
In this description, we follow the \texttt{scikit-image} implementation \citep{van2014scikit}, which seems to be the most popular at the moment. 
In all the paper, we focus on a rectangular image $\xi$, of size $(\height,\width)$. 
We denote the pixel positions by $(i,j)\in\image\defeq [\height]\times [\width]$, where $[k]\defeq \{1,\ldots,k\}$. 
The pixels values are denoted by $\xi_{i,j}\in\Reals^3$. 
Strictly speaking, we work directly in the CIELAB space, though quickshift usually takes as input RGB images and converts them to the CIELAB space. 

%%%%%%%%%%%%%%%%%%%%%%%%%%%%%%%%%%%%%%%%%%%%%%%%%%%%%%%%%%%%%%%%%%%%%%%%%%%%%%

\subsection{Density estimation}
\label{sec:density-estimation}

Quickshift relies on a Gaussian kernel density estimate which we denote by $P_{i,j}$. 
The main idea is to see pixels as points in $\Reals^5$ (two space coordinates and three color coordinates). 
In practice, for each position pixel $(i,j)$, not all pixels of the image are considered to build this estimate, and only pixels $(u,v)$ close to $(i,j)$ are taken into account.
More precisely, only $(u,v)$ inside
\begin{equation}
\label{eq:def-square}
C_{i,j} \defeq \{(u,v)\in \image, \text{ s.t. } \abs{i-u}\vee\abs{j-v}\leq \kernelwidth\}
\, ,
\end{equation}
where $\kernelwidth$ is a positive scale hyperparameter (called \emph{kernel width} in the \texttt{scikit-image} implementation) are considered. 
Now we are able to define the density estimates, computed according to
\begin{equation}
\label{eq:def-density-estimate}
P_{i,j} \defeq \sum_{(u,v)\in C_{i,j}} \exps{\frac{-(i-u)^2-(j-v)^2-\norm{\xi_{i,j}-\xi_{u,v}}^2}{2\kernelsize^2}}
\, ,
\end{equation}
where $\kernelsize$ is a positive scale hyperparameter (called \emph{kernel size} and equal to $5$ by default in the \texttt{scikit-image} implementation). 
By default, $\kernelwidth = 3\kernelsize$, an assumption that we will always make from now on, though it is straightforward to adapt our analysis for another fixed~$\kernelwidth$. 
Finally, some i.i.d.\! $\gaussian{0}{\sigma_0^2}$ noise is added on each $P_{i,j}$ to break eventual ties (with $\sigma_0=10^{-5}$). 
This procedure is described in Algorithm~\ref{algo:density}. 
The \texttt{scikit\-learn} also has a \emph{ratio} hyperparameter, allowing to adjust the importance of the pixel values with respect to the pixel positions. 
We do not consider this hyperparameter in our analysis, since it simply amounts to multiplying the pixel values by a positive constant, which does not change anything for flat regions. 

Given $(i,j)\in\image$, let us define for all $(u,v)\in C_{i,j}$
\begin{equation}
	\label{eq:def-Xuv}
	X_{u,v} \defeq \exp{\frac{-\norm{\xi_{i,j}-\xi_{u,v}}^2}{2\kernelsize^2}}\cdot \delta_{u,v}
	\, ,
\end{equation}
where
\begin{equation}
	\delta_{u,v}\defeq \exp{\frac{-(i-u)^2-(j-v)^2}{2\kernelsize^2}}
	\, .
\end{equation}
With these notation in hand, we see that 
\[
\forall (i,j)\in\image, \quad P_{i,j}=\sum_{(u,v)\in C_{i,j}}X_{u,v}
\, ,
\]
and already we can spot two major difficulties in our analysis.
The first is that, even if we assume the $\xi_{i,j}$s to be independent random variables, the $X_{u,v}$s are not, and as a consequence \textbf{the $P_{i,j}$s are not independent}. 
As a consequence, looking at statements involving several $P_{i,j}$s is very challenging. 
Our approach in Section~\ref{sec:homogeneous} will be to simplify the problem by reducing~$P$ to its main component~$Q$, using standard tools from the theory of $U$-statistics. 
The independence of the $Q_{i,j}$s makes them much more convenient to work with. 
Our main challenge will be to show that nothing of value is lost when making this approximation. 

Second, let us define $\midim(\kernelwidth)$ the set of points of $\image$ that are further than $\kernelwidth$ away from the border of the image. 
By definition, for all $(i,j)\in\midim(\kernelwidth)$, $C_{i,j}$ is the set of points contained in a square of side $2\kernelwidth+1$ centered in $(i,j)$ (see Figure~\ref{fig:three-possibilities}). 
\textbf{The picture is more complicated near the borders} of the image: for instance, right next to the border, $C_{i,j}$ has roughly half its size, which drastically lowers the density estimates near the border, even in simple cases. 
As we will see in Section~\ref{sec:bicolor}, this has non-trivial consequences. 
Note that some implementations do normalize Eq.~\eqref{eq:def-density-estimate} (for instance in the GPU implementation of \citet{fulkerson2010really}), which removes this problem.

\begin{algorithm}[ht]
	\caption{Density estimation.}
	\label{algo:density}
	\begin{algorithmic}[1]% [x] means every x line numbered
		\REQUIRE Image $\xi\in \Reals^{\height\times\width}$, kernel size $\kernelsize$.
		\STATE \textbf{Set:} $\kernelwidth \leftarrow 3\kernelsize$
		\STATE \textbf{Initialize:} $P\in \Reals^{\height\times \width}$
		\FOR{$(i,j)\in\image$}
		\STATE $P_{i,j}\leftarrow 0$
		\FOR{$(u,v)\in C_{i,j}$}
		\STATE $P_{i,j}\leftarrow P_{i,j} + \exp{\frac{-(i-u)^2-(j-v)^2-\norm{\xi_{i,j}-\xi_{u,v}}^2}{2\kernelsize^2}}$
		\ENDFOR
		\STATE $P_{i,j} \leftarrow P_{i,j} + \gaussian{0}{\sigma_0^2}$
		\ENDFOR
		\RETURN $P$
	\end{algorithmic}
\end{algorithm}

%%%%%%%%%%%%%%%%%%%%%%%%%%%%%%%%%%%%%%%%%%%%%%%%%%%%%%%%%%%%%%%%%%%%%%%%%%%%%%%%%

\subsection{Graph construction}

We now turn to the graph construction, the heart of the quickshift algorithm, which we describe for any array $A\in\Reals^{\height\times\width}$.  
Intuitively, quickshift moves each pixel to the nearest neighbor with higher density. 
When this is no longer possible, we have found a local maximum of the density. 
Formally, quickshift produces a directed graph $\graph_0(A)$ with vertices $\image$ in the following way: first, all vertices are visited, and quickshift places an edge between $(i,j)$ and $(u,v)\in C_{i,j}$ if two conditions are satisfied: (i) $A_{i,j} < A_{u,v}$, and (ii) $(i-u)^2+(j-v)^2+\norm{\xi_{i,j}-\xi_{u,v}}^2$ is minimal among all $(u,v)\in C_{i,j}$ satisfying condition~(i). 
Second, all edges between pixels with ($5$-dimensional) distance greater than $\maxdist$ are removed from $\graph_0(A)$, where $\maxdist$ is a positive hyperparameter called \emph{maximal distance} in the \texttt{scikit-image} implementation and set to $10$ by default. 
Finally, the superpixels are defined as the \emph{connected components} of $\graph_0(A)$.  
The graph construction is summarized in Algorithm~\ref{algo:quickshift}.

%%%%%%%%%%%%%%%%%%%%%%%%%%%%%%%%%%%%%%%%%%%%%%%%%%%%%%%%%%%%%%%%%%%%%%%%%%%%%%

\begin{algorithm}[ht]
	\caption{Quickshift graph construction (following \texttt{skimage} implementation). }
	\label{algo:quickshift}
	\begin{algorithmic}[1]% [x] means every x line numbered
		\REQUIRE Image $\xi\in \Reals^{\height\times\width}$, kernel size $\kernelsize$, maximum distance $\maxdist$, array $A\in \Reals^{\height\times\width}$.
		\STATE \textbf{set:} $\kernelwidth \leftarrow 3\kernelsize$
		\FOR{$(i,j)\in\image$}
		\STATE \textbf{initialize:} $M\leftarrow +\infty$ with same shape as $C_{i,j}$
		\FOR{$(u,v)\in C_{i,j}$}
		\STATE $d\leftarrow (i-u)^2+(j-v)^2+\norm{\xi_{i,j}-\xi_{u,v}}^2$
		\IF{$A_{i,j} < A_{u,v}$ \textbf{and} $d\leq \maxdist$}
		\STATE $M_{u,v} \leftarrow d$
		\ENDIF
		\ENDFOR
		\STATE $(u,v)\leftarrow \Argmin M$
		\STATE $G_{(i,j),(u,v)}\leftarrow 1$
		\ENDFOR
		\RETURN $G$
	\end{algorithmic}
\end{algorithm}

%%%%%%%%%%%%%%%%%%%%%%%%%%%%%%%%%%%%%%%%%%%%%%%%%%%%%%%%%%%%%%%%%%%%%%%%%%%%%%

In definitive, we see quickshift as \textbf{the application of the graph construction procedure given by Algorithm~\ref{algo:quickshift} to the density estimate obtained by Algorithm~\ref{algo:density}}. 
That is, quickshift first computes~$P$ by means of Algorithm~\ref{algo:density} and then applies Algorithm~\ref{algo:quickshift} to $A=P$.
Other implementations exist, though they share the same basic ideas. 
For instance, the VLFeat library 
%\footnote{\texttt{vlfeat.org/}} 
normalizes the $P_{i,j}$s by $1/(2\pi\kernelsize)^{5+2}$.

%%%%%%%%%%%%%%%%%%%%%%%%%%%%%%%%%%%%%%%%%%%%%%%%%%%%%%%%%%%%%%%%%%%%%%%%%%%%%%

\subsection{Computational complexity}
\label{sec:computational-cost}

A naive implementation of the density estimation step (Algorithm~\ref{algo:density})  requires a full pass on the image, and for each pixel a pass on $C_{i,j}$. 
Therefore, the density estimation part of quickshift costs $\bigo{\height\width\kernelwidth^2}$. 
Regarding the graph construction (Algorithm~\ref{algo:quickshift}), first one has to parse the image once again and look into the $C_{i,j}$ windows, which costs $\bigo{\height\width\kernelwidth^2}$. 
Subsequently, we need to find the connected components of a directed graph with $\height\cdot\width$ vertices and the same number of edges, which costs $\bigo{\height\width}$ \citep{hopcroft1973algorithm}. 
In definitive, the total computational cost is linear in $\height$ and $\width$, and quadratic in $\kernelwidth$. 
This is something to consider when increasing the $\kernelwidth$ hyperparameter. 
Finally, note that the $\maxdist$ hyperparameter has no influence on the computational cost, though setting small values can lead to producing many superpixels. 
In particular, by the graph construction, all superpixels have geometric diameter smaller than $\maxdist$. 

%%%%%%%%%%%%%%%%%%%%%%%%%%%%%%%%%%%%%%%%%%%%%%%%%%%%%%%%%%%%%%%%%%%%%%%%%%%%%%%%%

\begin{figure}
\begin{center}
\includegraphics[scale=0.135]{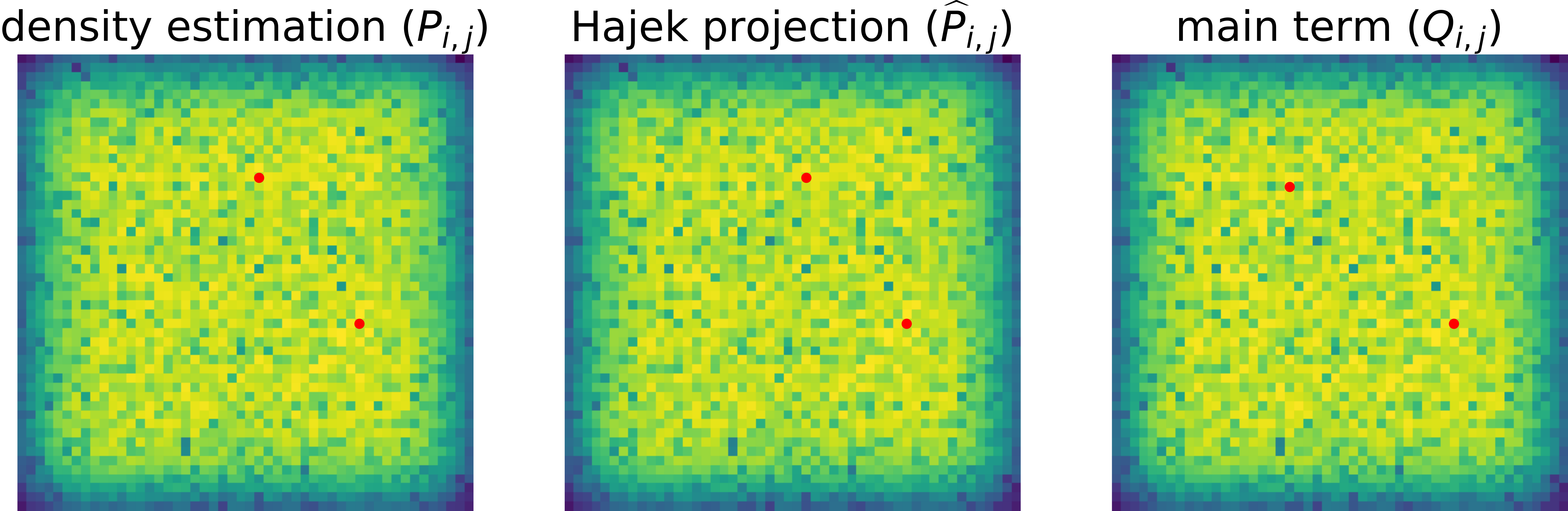}
\end{center}
\vspace{-0.1in}
\caption{\label{fig:density-estimation}Density estimation under Assumption~\ref{ass:flat-image} for a $50\times 50$ image with $\sigma=1.0$. \emph{Left:} using the definition,  Eq.~\eqref{eq:def-density-estimate}; \emph{middle:} using the H\'ajek projection, $\proj_{i,j}$; \emph{right:} using only the main term $Q_{i,j}$ (Eq.~\eqref{eq:def-Q}). As predicted by Theorem~\ref{th:P-close-to-Q}, $P$ and $Q$ are close to each other, and so are the associated local maxima (in red). 
Of course, their number and exact positions can differ.}
\end{figure}

%%%%%%%%%%%%%%%%%%%%%%%%%%%%%%%%%%%%%%%%%%%%%%%%%%%%%%%%%%%%%%%%%%%%%%%%%%%%%%%%%

\section{HOMOGENEOUS PATCHES}
\label{sec:homogeneous}

In this section, we focus on flat portions of the image. 
That is, subsets of $\image$ such that $\xi_{i,j}$ is approximately constant. 
As a simple model for this situation, we propose the following:
\begin{assumption}[Flat image]
	\label{ass:flat-image}
	For all $(i,j)\in\image$, the $\xi_{i,j}$ are i.i.d.\!  $\gaussian{\oc}{\sigma^2\Identity_3}$, where $\oc\in\Reals^3$ is an arbitrary color and $\sigma\leq \kernelsize / 5$. 
\end{assumption}

This assumption is a simple model of the random noise coming from an electronic sensor when photographing a mono-color item (see, \emph{e.g.}, \citet[Chapter 1.4,][]{pratt2001digital}). 
Under this assumption, we derive an approximate expression for $P$ in Section~\ref{sec:density}, before introducing in Section~\ref{sec:quickshift-modification} the modification of the quickshift algorithm that we investigate here. 
After precising the interplay between local maxima and superpixels in Section~\ref{sec:counting-superpixels}, we state our main result in Section~\ref{sec:average-local-max}. 

%%%%%%%%%%%%%%%%%%%%%%%%%%%%%%%%%%%%%%%%%%%%%%%%%%%%%%%%%%%%%%%%%%%%%%%%%%%%%%%%

\subsection{A closer look at density estimation}
\label{sec:density}

Under Assumption~\ref{ass:flat-image}, it is straightforward to compute the first moments of the $X_{u,v}$, and by extension those of $P_{i,j}$. 
Let us first define the normalization constant
\[
\normcst_p \defeq \left( \frac{\kernelsize^2}{\kernelsize^2+p\sigma^2}\right)^{3/2}
\, .
\]
Then we show in the Appendix that $\expec{P_{i,j}} = \normcst_2 \Delta_{i,j}$, where $\Delta_{i,j}\defeq \sum_{(u,v)\in C_{i,j}}\delta_{u,v}$. 
As noted in Section~\ref{sec:quickshift}, the $X_{u,v}$ are not independent. 
For instance, we show in the Appendix that $\cov{X_{u,v}}{X_{u',v'}}$ is always positive under~\ref{ass:flat-image}. 
Nonetheless, a careful reading of Eq.~\eqref{eq:def-Xuv} reveals that $\xi_{i,j}$ draws all values of the density estimate. 
More precisely, $\xi_{i,j}$ is present in all the $X_{u,v}$ (there are $\bigo{\kernelsize^2}$ of them), whereas each other individual $X_{u,v}$ is only present one time. 
Therefore, when $\kernelsize$ is large, we expect $P_{i,j}$ to behave roughly as its conditional expectation with respect to $\xi_{i,j}$ for a given $(i,j)$, which is given by 
\begin{equation}
\label{eq:def-Q}
Q_{i,j} \defeq \normcst_1 \cdot \exp{\frac{-\norm{\xi_{i,j}-\oc}^2}{2(\kernelsize^2+\sigma^2)}} \cdot \Delta_{i,j}
\, .
\end{equation}
We formalize this intuition by the following:

\begin{theorem}[$P_{i,j}$ is close to $Q_{i,j}$, w.h.p.]
\label{th:P-close-to-Q}
Assume that~\ref{ass:flat-image} holds. 
Suppose furthermore that $\kernelsize\geq 5$. 
Let $(i,j)\in\midim(\kernelwidth)$. 
Then, for any $\epsilon >0$, 
\[
\proba{\abs{P_{i,j} - Q_{i,j}} > \epsilon} 
\leq \frac{71\sigma^2}{\epsilon^2}
\, .
\]
\end{theorem}

The main consequence of this result is that \textbf{one can study $Q_{i,j}$ instead of $P_{i,j}$ when $\sigma$ is small}, at least for a fixed $(i,j)\in\midim(\kernelwidth)$. 
We refer to Figure~\ref{fig:density-estimation} for an illustration.
Theorem~\ref{th:P-close-to-Q} is the main motivation to study the graph construction step of quickshift on~$Q$ instead of~$P$. 
%
% limitations
Of course, looking at the graph construction implies looking simultaneously to several $P_{i,j}$s, a case which is not covered by Theorem~\ref{th:P-close-to-Q}. 
We conjecture that a uniform bound for all $(i,j)\in\midim(\kernelwidth)$ exists, though a straightforward approach \emph{via} a union bound argument does not yield a satisfying bound. 

\textit{Sketch of the proof of Theorem~\ref{th:P-close-to-Q}.} 
Following \citet{van2000asymptotic}, the key idea of the proof is to compute the H\'ajek projection of $P_{i,j}$ onto the set of random variables $\xi_{u,v}$, $(u,v)\in\image$ to capture the influence of each of the $\xi_{u,v}$.
We then show that the projection onto $\xi_{i,j}$, that is, $Q_{i,j}$, is the most prominent one. 
We start by computing the H\'ajek projection of $P_{i,j}$, which boils down to Gaussian integral computations under Assumption~\ref{ass:flat-image}, and we find, $\forall (i,j)\in\image$, 
\begin{equation}
\label{eq:hajek-projection}
	\proj_{i,j} = Q_{i,j} 
	+\, \normcst_1 \!\!\!\!\!\!\!\sum_{\substack{(u,v)\in C_{i,j} \\ (u,v)\neq (i,j)}}  \left(\exps{\frac{-\norm{\xi_{u,v}-\oc}^2}{2(\kernelsize^2+\sigma^2)}}- \left(\tfrac{\kernelsize^2+\sigma^2}{\kernelsize^2+2\sigma^2}\right)^{\frac{3}{2}}\right) \!\delta_{u,v}
	%\, .
\end{equation}
We are then able to show that $P_{i,j}$ and $\proj_{i,j}$ have similar variances when $\kernelsize$ is large enough. 
More precisely, 
\[
\abs{\frac{\smallvar{\proj_{i,j}}}{\var{P_{i,j}}} - 1} \leq \frac{4}{5\kernelsize^2}
\, .
\]
Therefore, an application of Theorem~11.2 in \cite{van2000asymptotic} yields:
\[
\proba{\abs{\frac{P_{i,j}-\expec{P_{i,j}}}{\sqrt{\var{P_{i,j}}}} -\frac{\proj_{i,j}-\smallexpec{\proj_{i,j}}}{\sqrt{\smallvar{\proj_{i,j}}}}} > \epsilon} \leq \frac{1}{\kernelsize^2\epsilon^2}
\, .
\]

Let us define $R_{i,j}\defeq \proj_{i,j}-Q_{i,j}$, the rest term. 
Then we are able to show that $\var{R_{i,j}}\leq \sigma^2/4$. 
Thus,  for any given $(i,j)$, when $\kernelsize$ is large enough and $\sigma$ is small with respect to $\kernelsize$, we can identify $P_{i,j}$ with $\proj_{i,j}$, which in turn can be identified with $Q_{i,j}$, concluding the proof of Theorem~\ref{th:P-close-to-Q}. 
\qed 

%%%%%%%%%%%%%%%%%%%%%%%%%%%%%%%%%%%%%%%%%%%%%%%%%%%%%%%%%%%%%%%%%%%%%%%%%%%%%%%%%%

\begin{figure}
	\begin{center}
		\includegraphics[scale=0.22]{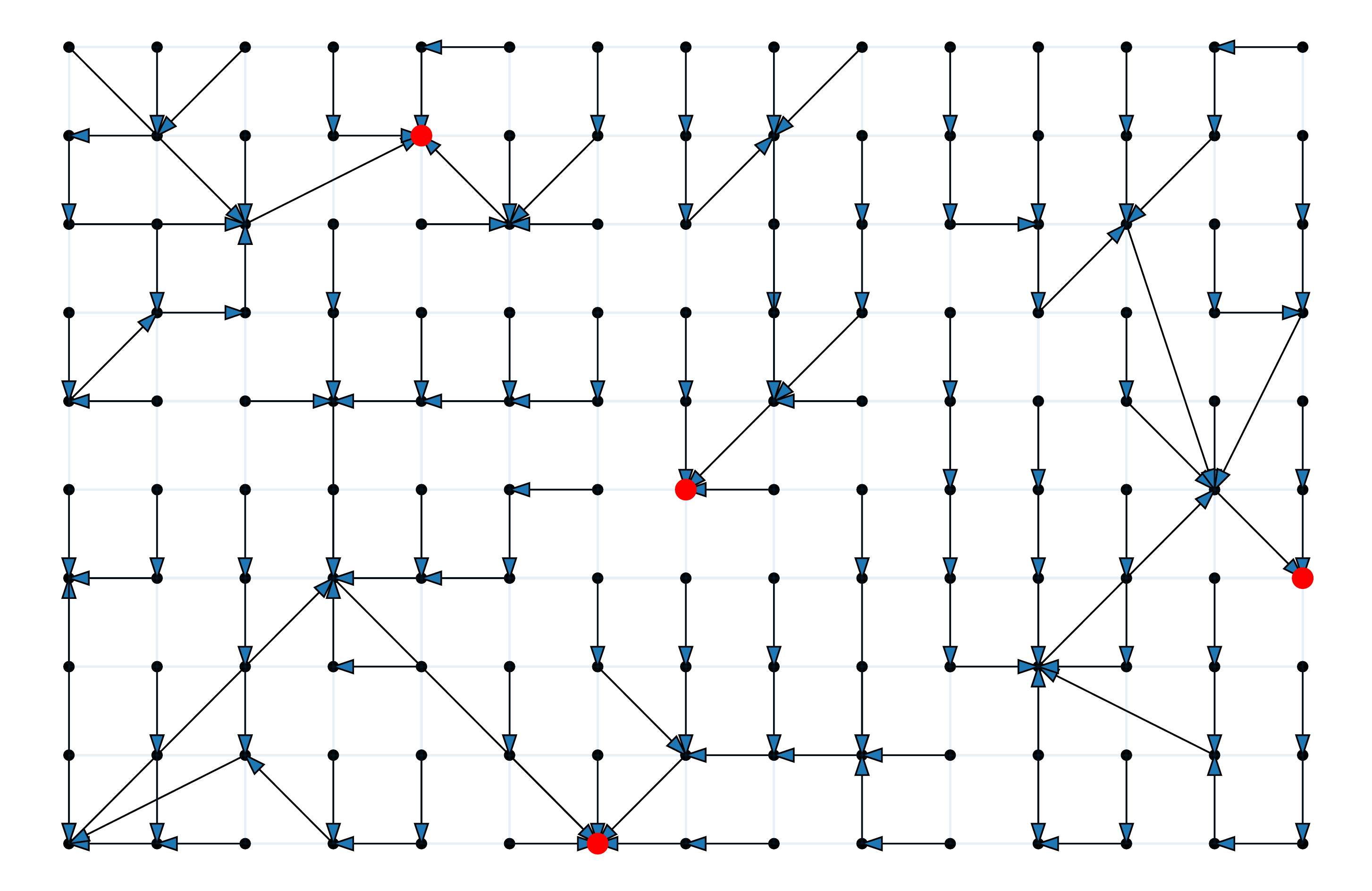}
	\end{center}
	\vspace{-0.2in}
	\caption{\label{fig:graph-construction}A sample graph $\graph(P)$ obtained by the running Algorithm~\ref{algo:simplified-quickshift} on the density estimate $P$ of an image of size $10\times 15$. Each arrow connects a pixel to the nearest pixel having a higher density estimate. The local maxima are marked in red: there is one such local maxima per connected component of $\graph$, as per Lemma~\ref{lemma:connected-components}. }
\end{figure}

%%%%%%%%%%%%%%%%%%%%%%%%%%%%%%%%%%%%%%%%%%%%%%%%%%%%%%%%%%%%%%%%%%%%%%%%%%%%%%%%%%

\subsection{Modified graph construction}
\label{sec:quickshift-modification}

We now present a slight modification in the graph construction. 
Namely, \textbf{we remove the $\norm{\xi_{i,j}-\xi_{u,v}}^2$ term from the distance computation} (line~5 of Algorithm~\ref{algo:quickshift}). 
Indeed, under Assumption~\ref{ass:flat-image}, $\smallexpec{\norm{\xi_{i,j}-\xi_{u,v}}^2}=\bigo{\sigma^2}$.
Since the Euclidean distance term is $\bigo{1}$, the color difference term is negligible whenever $\sigma$ is small or $\maxdist$ is large, and there is little difference between the output of Algorithm~\ref{algo:quickshift} and~\ref{algo:simplified-quickshift}. 
Of course, as soon as $\sigma$ is of the order of magnitude of $\maxdist$, notable differences appear. 
In particular, Algorithm~\ref{algo:simplified-quickshift} tends to produce much more superpixels in this situation.

A key observation is that, when using this modified version of quickshift, the graph construction simplifies greatly: instead of looking for values of~$A$ that are greater than $A_{i,j}$ inside $C_{i,j}$ and subsequently cut off all points with $5$-dimensional distance greater than $\maxdist$, we can look directly for points in $E_{i,j}\defeq C_{i,j}\cap D_{i,j}$, where $D_{i,j}$ is the set of points of $\image$ such that $\sqrt{(i-u)^2+(j-v)^2}\leq \maxdist$. 
There are three possibilities for $E_{i,j}$, depending on the relative position of $\kernelwidth$ and~$\maxdist$: we depict these three cases in Figure~\ref{fig:three-possibilities}. 
We call $\graph(A)$ this procedure, described formally in Algorithm~\ref{algo:simplified-quickshift} and illustrated in Figure~\ref{fig:graph-construction}. 

To conclude this section, we want to emphasize that the core analysis of this paper concerns $\graph(Q)$, not $\graph_0(P)$ which is the output of the default implementation. 

%%%%%%%%%%%%%%%%%%%%%%%%%%%%%%%%%%%%%%%%%%%%%%%%%%%%%%%%%%%%%%%%%%%%%%%%%%%%%%%%%%%%

\begin{figure}
	\begin{center}
		\includegraphics[scale=0.135]{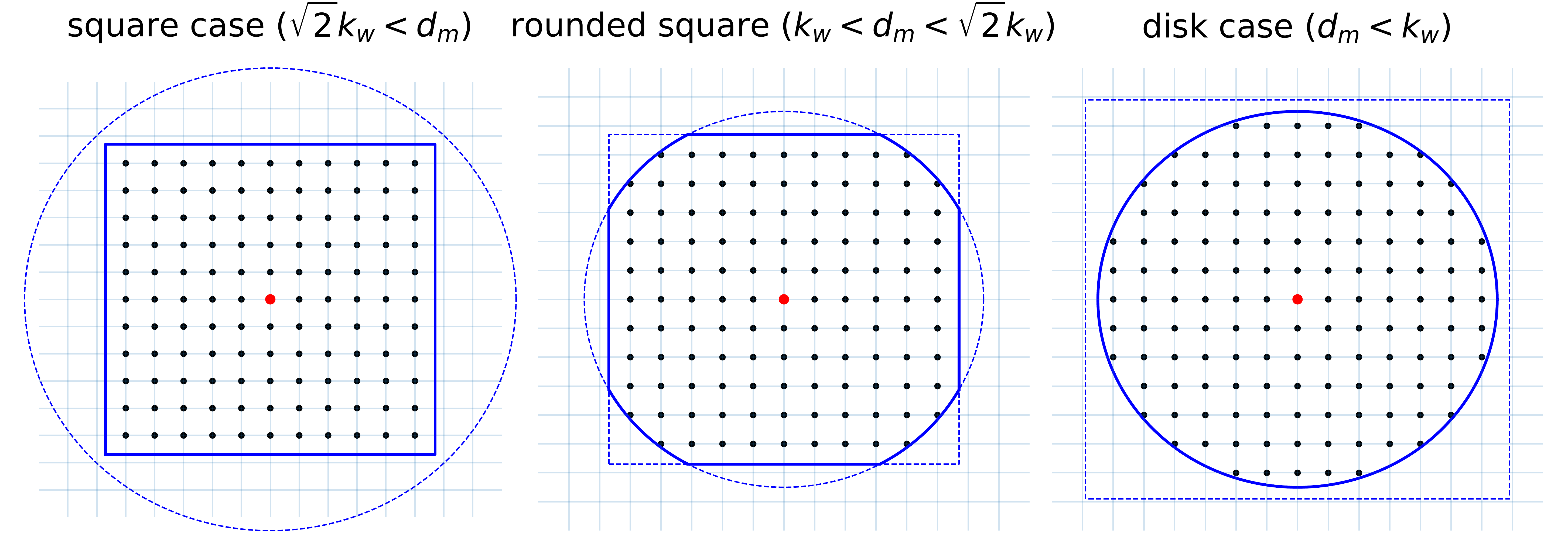}
	\end{center}
	\vspace{-0.2in}
	\caption{\label{fig:three-possibilities}Three possibilities for the lookout area of Algorithm~\ref{algo:simplified-quickshift}. At a given $(i,j)\in\image$ (in red), the algorithm looks at the nearest neighbor with higher value in $E_{i,j}=C_{i,j}\cap D_{i,j}$ (solid blue line). The shape of the intersection depends on the interplay between $\kernelsize$ and $\maxdist$. For instance, the default choice $\kernelsize=5$ and $\maxdist=10$ yields $\kernelwidth=15$, which is greater than $\maxdist$: we are in the disk case. This trichotomy remains when both $\kernelsize$ and $\maxdist$ are multiplied by a constant factor.}
\end{figure}

%%%%%%%%%%%%%%%%%%%%%%%%%%%%%%%%%%%%%%%%%%%%%%%%%%%%%%%%%%%%%%%%%%%%%%%%%%%%%%%%%%%%

\begin{algorithm}[ht]
	\caption{Simplified quickshift graph construction. 
	}
	\label{algo:simplified-quickshift}
	\begin{algorithmic}[1]% [x] means every x line numbered
		\REQUIRE Image $\xi\in \Reals^{\height\times\width}$, kernel size $\kernelsize$, maximum distance $\maxdist$, array $A\in \Reals^{\height\times\width}$.
		\STATE \textbf{set:} $\kernelwidth \leftarrow 3\kernelsize$
		\FOR{$(i,j)\in\image$}
		\STATE $(u,v)\leftarrow$ closest element of $E_{i,j}$ s.t. $A_{i,j}<A_{u,v}$
		\STATE $G_{(i,j),(u,v)}\leftarrow 1$
		\ENDFOR
		\RETURN $G$
	\end{algorithmic}
\end{algorithm}

%%%%%%%%%%%%%%%%%%%%%%%%%%%%%%%%%%%%%%%%%%%%%%%%%%%%%%%%%%%%%%%%%%%%%%%%%%%%%%%%%%

\subsection{Counting superpixels}
\label{sec:counting-superpixels}

We now return to our main focus: counting the number of superpixels. 
For a given region $R$ of the image, we set $K_R(A)$ the number of connected components of $\graph(A)$ intersecting $R$. 
It is challenging to investigate $K_R(A)$ directly, and we rather study the number of \emph{local maxima}:

\begin{definition}[Local maximum]
\label{def:local-maximum}
Let $A\in\Reals^{\height\times\width}$ and $(i,j)\in\image$. 
We say that $(i,j)$ is a local maximum of~$A$ if $\forall (u,v)\in E_{i,j}$, $A_{i,j}>A_{u,v}$. 
\end{definition}

We will use the notation $N_R(A)$ to denote the number of local maxima of $A$ inside a given region $R$. 
The main reason to study local maxima is that, on a global scale, there is an equivalence between the number of superpixels and the number of local maxima, as is demonstrated in Figure~\ref{fig:graph-construction}. 
This fact is formalized by our next lemma:

\begin{lemma}[Connected components and local maxima]
\label{lemma:connected-components}
Let $\graph(A)$ be the directed graph produced by Algorithm~\ref{algo:simplified-quickshift} applied to an array $A\in\Reals^{\height\times\width}$. 
Then to each connected component of $\graph(A)$ corresponds a \emph{unique} local maxima of $A$ in the sense of Definition~\ref{def:local-maximum}. 
\end{lemma}

In view of Lemma~\ref{lemma:connected-components}, we will now focus on $N_R(A)$, which is \emph{a priori} distinct from $K_R(A)$. 
Indeed, when restricting ourselves to~$R$, it can happen that we miss the local maxima corresponding to the connected component intersecting with~$R$. 
Namely, only the trivial lower bound
\begin{equation}
\label{eq:lower-bound-local-max}
N_R(A) \leq K_R(A)
\, 
\end{equation}
holds.
There is no corresponding upper bound in the general case, since a rectangle can intersect many superpixels without containing \emph{any} local maximum as depicted in Figure~\ref{fig:counter-example}. 
Nevertheless, empirically, $N_R(A)$ and $K_R(A)$ are comparable for large regions (and the bound given by Eq.~\eqref{eq:lower-bound-local-max} is reasonably tight).

\begin{figure}[ht]
	\begin{center}
		\includegraphics[scale=0.18]{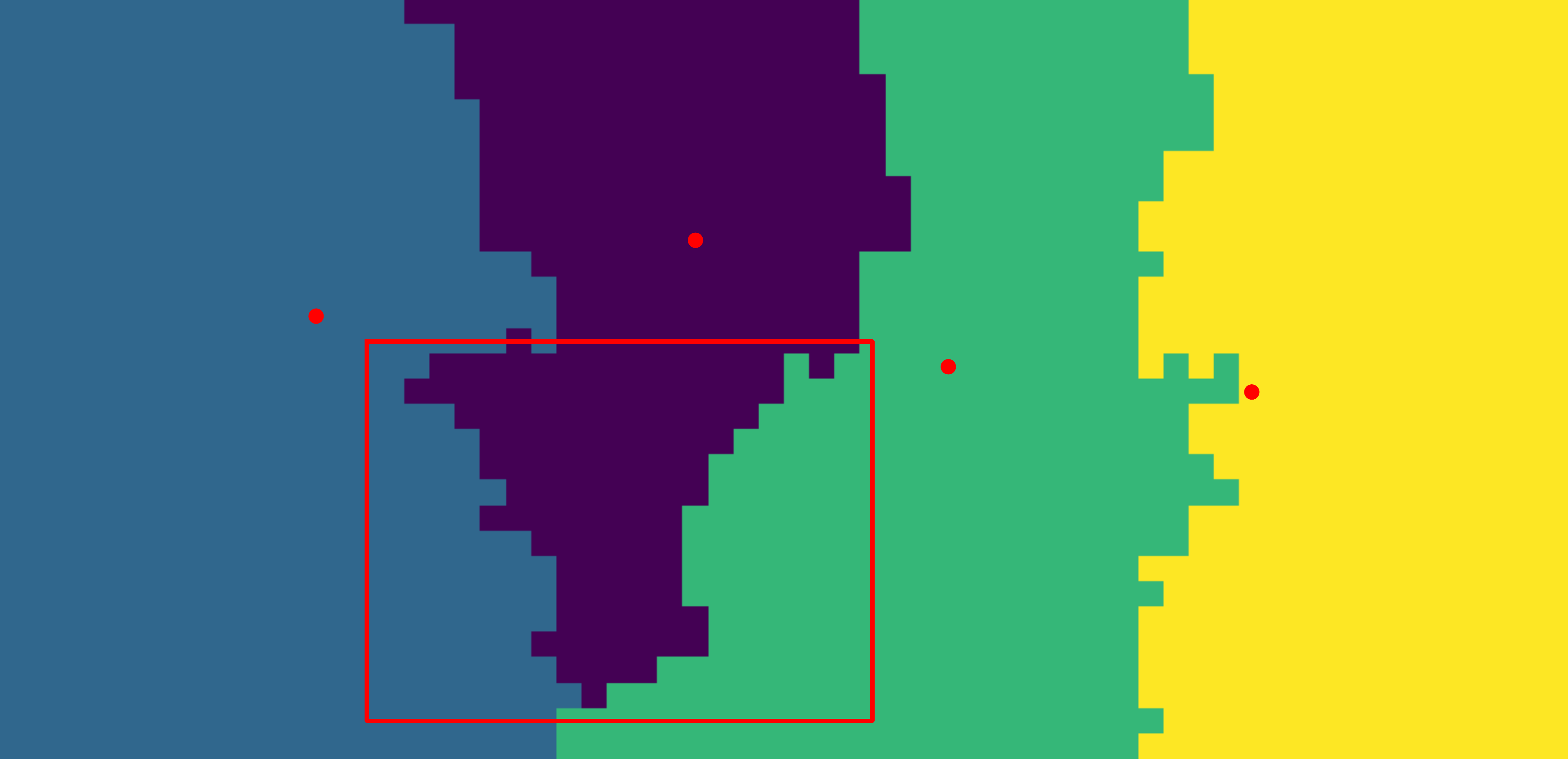}
	\end{center}
	\caption{\label{fig:counter-example}The number of local maxima inside a part of the image does not always coincide with the number of superpixels intersecting this area. In this example,~$R$ is a rectangle with boundaries marked in red. The number of local maxima inside $R$ is $N_R(A)=0$, whereas the number of superpixels intersecting $R$ is $K_R(A)=3$. }
\end{figure}

%%%%%%%%%%%%%%%%%%%%%%%%%%%%%%%%%%%%%%%%%%%%%%%%%%%%%%%%%%%%%%%%%%%%%%%%%%%%%%%%

\begin{figure*}[ht]
	\begin{center}
		\includegraphics[scale=0.28]{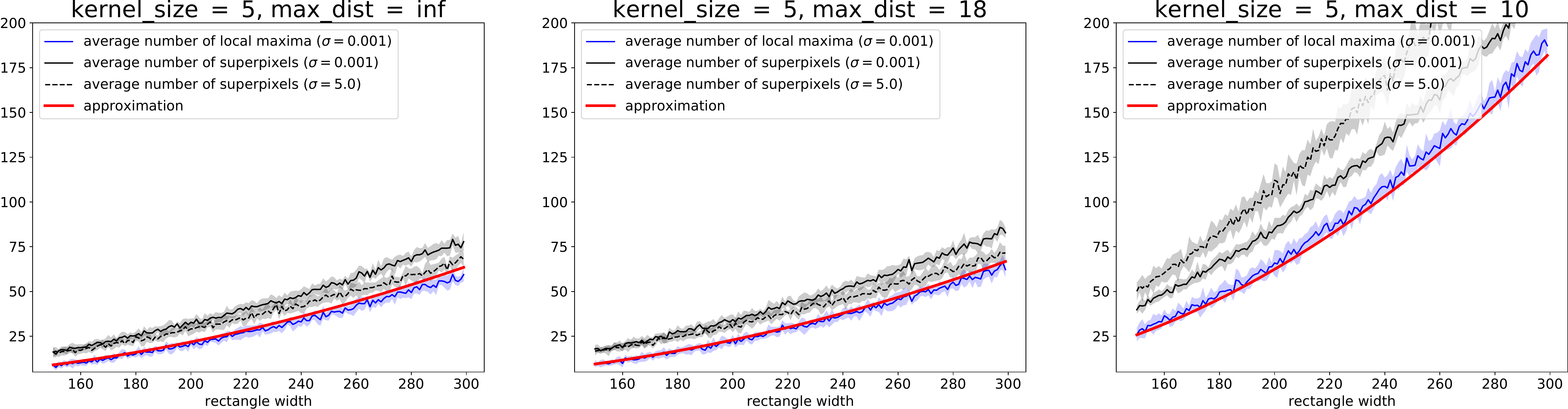}
	\end{center}
\vspace{-0.1in}
	\caption{\label{fig:evolution}Evolution of the number of superpixels created by quickshift under Assumption~\ref{ass:flat-image}. For each input shape, we run quickshift on ten images and count the number of local maxima and the number of superpixels in a rectangular patch far from the border of the image. We see that, when $\sigma$ is small, the prediction of Theorem~\ref{th:average-local-max} (in red) is close to the empirical number of local maxima (in blue). Moreover, the number of local maxima coincides roughly with the number of superpixels (in black). When $\sigma$ increases, as expected, this link weakens but we still observe a quadratic dependency in the size of the rectangle. 
The three choices of $(\kernelsize,\maxdist)$ correspond to the three different configurations (square, rounded square, and disk).}
\end{figure*}

%%%%%%%%%%%%%%%%%%%%%%%%%%%%%%%%%%%%%%%%%%%%%%%%%%%%%%%%%%%%%%%%%%%%%%%%%%%%%%%%%

\subsection{Average number of local maxima in flat regions}
\label{sec:average-local-max}

We are now ready to state and prove the main result of this section.

\begin{theorem}[Average number of local maxima]
\label{th:average-local-max}
Assume that~\ref{ass:flat-image} holds. 
Let $R\subseteq\image$ be a rectangle of height $h$ and width $w$ at distance greater than $2\kernelwidth$ from the border. 
Then 
\[
\expec{N_R(Q)} = 
\begin{cases}
	hw\cdot\left(\frac{1}{\pi \maxdist^2}+\bigo{\frac{1}{\maxdist}}\right) \text{ if } \maxdist \leq \kernelwidth \\
hw\cdot \left(\frac{1}{\pi(3\kernelwidth\maxdist-\kernelwidth^2-\maxdist^2)}+\bigo{\frac{1}{\kernelwidth}}\right) \\
\phantom{blablablabkablabla}\text{ if } \kernelwidth < \maxdist \leq \sqrt{2}\kernelwidth \\
hw\cdot \left(\frac{1}{4\kernelwidth^2}+\bigo{\frac{1}{\kernelwidth}}\right) \text{ otherwise. }
\end{cases}
\]
\end{theorem}

The main consequence of Theorem~\ref{th:average-local-max} is the scale of $N_R(Q)$ with respect to (i) the size of $R$, and (ii) the hyperparameters. 
In plain words, firstly, we find $N_R(Q)$ to be \textbf{proportional to the area of $R$}: a flat, rectangular part of the image will contain twice as many superpixels if its size is multiplied by two. 
Secondly, whatever the relative position of $\kernelsize$ and $\maxdist$ is, \textbf{$N_R(Q)$ is $(-2)$-homogeneous in $(\kernelsize,\maxdist)$}: doubling both $\kernelsize$ and $\maxdist$ will approximately divide $N_R(Q)$ by four. 
We illustrate the validity of Theorem~\ref{th:average-local-max} with respect to the non-modified version of quickshift (Algorithm~\ref{algo:quickshift} running on~$P$) in Figure~\ref{fig:evolution}. 

\textit{Proof of Theorem~\ref{th:average-local-max}.}
The key of the proof is the following result, which is the main motivation for finding an i.i.d.\! approximation of~$P$:

\begin{lemma}[Key lemma]
\label{lemma:key-lemma}
Let $A\in\Reals^{\height\times\width}$ be a random array such that the $A_{i,j}$ are i.i.d.\! with cumulative distribution function $G$ and associated density $g$. 
Let $R\subseteq \image$. 
Define $N_{i,j}\defeq \card{E_{i,j}}$. 
Then the expected number of local maxima in $R$ produced by Algorithm~\ref{algo:simplified-quickshift} applied to~$A$ is given by
\[
\expec{N_R(A)} = \sum_{(i,j)\in R} \frac{1}{N_{i,j}}
\, .
\]
\end{lemma}

Note that, as a consequence, Theorem~\ref{th:average-local-max} is not limited to a rectangular area, though it is less convenient to state for a general shape.

\textit{Proof of Lemma~\ref{lemma:key-lemma}.} 
We first write $N_R(A)$ as the sum of indicator that $(i,j)$ is a local maximum. 
Further, 
\begin{align*}
&\expec{N_R(A)} = \expec{\sum_{(i,j)\in R} \indic{\forall (u,v) \in E_{i,j}^\star, \,\, A_{i,j} > A_{u,v}}} \\
&= \sum_{(i,j)\in R} \proba{\forall (u,v) \in E_{i,j}^\star, \,\,A_{i,j} > A_{u,v}}
\, ,
\end{align*}
where we write $E_{i,j}^\star$ short for $E_{i,j}\setminus\{(i,j)\}$. 
Let us condition with respect to $A_{i,j}$. 
Using the i.i.d.\!\! assumption on the $A_{i,j}$s, by definition of the cumulative distribution function $G$,
\[
\condproba{\forall (u,v) \in E_{i,j}^\star,\,\, A_{i,j} > A_{u,v}}{A_{i,j}} = G(A_{i,j})^{N_{i,j}-1}
\, .
\]
Integrating the last display yields the result, since
\begin{align*}
\int G(t)^{N_{i,j}-1}g(t) \,\Diff t 
= \left[ \frac{G(t)^{N_{i,j}}}{N_{i,j}}\right]_{-\infty}^{+\infty} 
= \frac{1}{N_{i,j}}
\, .
\end{align*}
\qed 

Further, we see that the $Q_{i,j}$ are indeed i.i.d.\! if we are outside a band of width $2\kernelwidth$ from the border: here, $\Delta_{i,j}$ does not depend on $(i,j)$. 
Thus we can apply Lemma~\ref{lemma:key-lemma} to the random array $Q$. 
Finally, it is just a matter of counting the number of points inside $E_{i,j}$. 
For instance, let us assume that $\maxdist \leq \kernelwidth$. 
In that case, we have to count the number of lattice points inside a disk of radius $\maxdist$. 
This is known as the \emph{Gauss circle problem}, and an immediate bound (due to Gauss himself) is $N_{i,j} =\pi\maxdist^2+\bigo{\maxdist}$. 
Summing over all pixels in $R$ yields the last result in Theorem~\ref{th:average-local-max}. 
The other cases are similar and the details can be found in the Appendix. 
\qed

%%%%%%%%%%%%%%%%%%%%%%%%%%%%%%%%%%%%%%%%%%%%%%%%%%%%%%%%%%%%%%%%%%%%%%%%%%%%%%%%%

\section{SHARP EDGES}
\label{sec:bicolor}

In this short section, we focus on well-defined boundaries between homogeneous patches of the image.
More precisely, we assume the following:

\begin{assumption}[Bicolor image]
\label{ass:bicolor}
There exist $j_0\in [\width]$, $\oc_1,\oc_2\in\Reals^3$ such that the $\xi_{i,j}$ are i.i.d. $\gaussian{\oc_1}{\sigma^2\Identity_3}$ (resp. $\gaussian{\oc_2}{\sigma^2\Identity_3}$) on $\leftim\defeq [\height]\times\{1,\ldots,j_0\}$ (resp. $\rightim\defeq [\height]\times\{j_0,\ldots,\width\}$), with $\sigma\leq \kernelsize / 5$. 
\end{assumption}

In plain words, we consider a bicolor image with a vertical boundary at $j_0$ between them and i.i.d.\! Gaussian noise as in Section~\ref{sec:homogeneous}.
We refer to Figure~\ref{fig:bicolor-setting} for a visual depiction. 
\begin{figure}
	\begin{center}
		\includegraphics[scale=0.18]{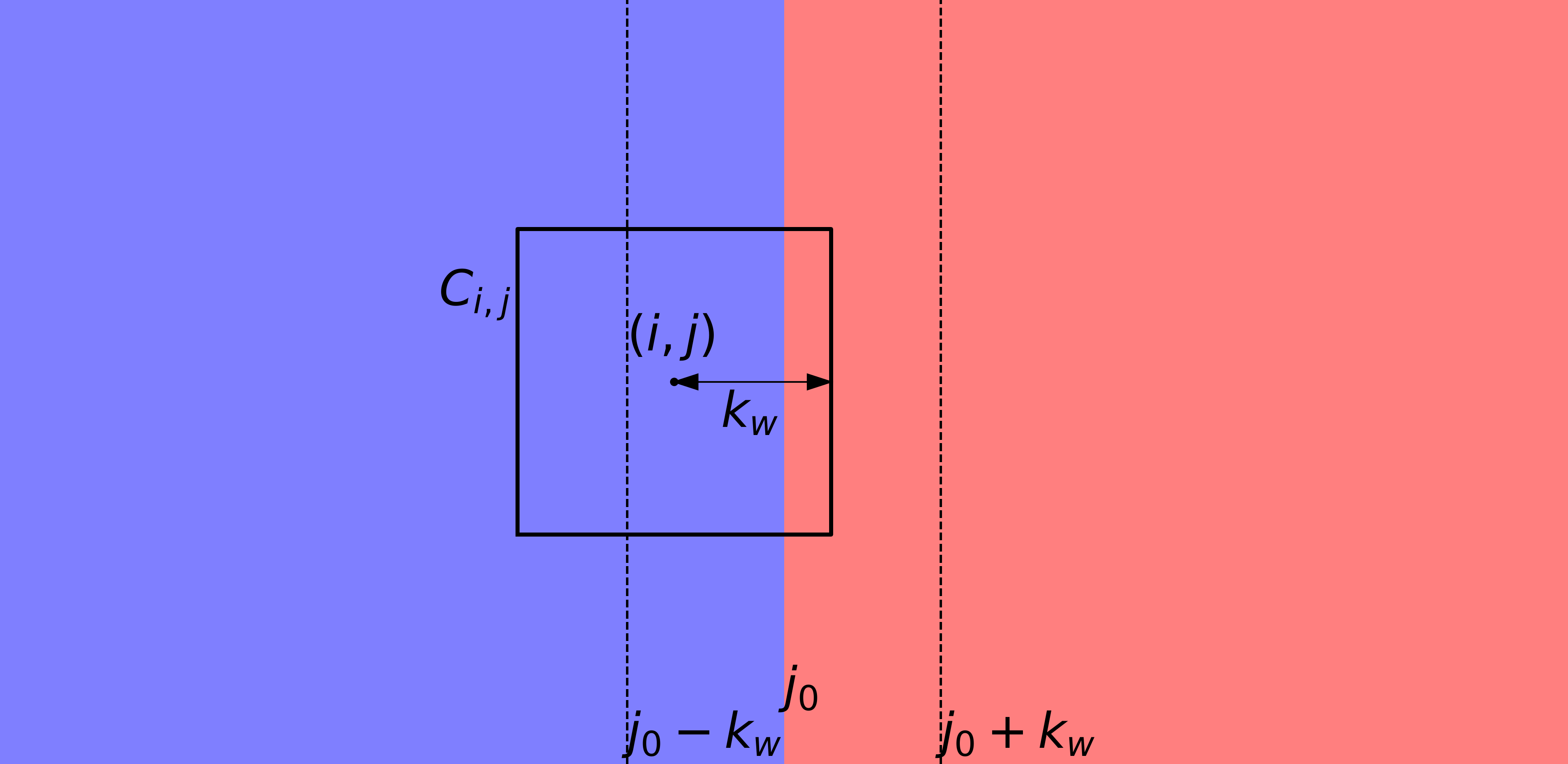}
	\end{center}
	\caption{\label{fig:bicolor-setting}The bicolor setting described by Assumption~\ref{ass:bicolor}. Here, $\oc_1$ is a light blue and $\oc_2$ a light red. The border between the two patches marked at $j_0$ is in the middle of the image.}
\end{figure}
Note that we can also work under Assumption~\ref{ass:bicolor}  to understand what happens at the border of a flat image: 
Let us consider $(i,j)$ a pixel near the boundary of the image. 
In that case, $C_{i,j}$ contains fewer points than in the center of the image since the square centered at $(i,j)$ ``overflows'' the image, and therefore the sum in Eq.~\eqref{eq:def-density-estimate} is missing some terms.  
Instead of thinking of $C_{i,j}$ as reduced, we can think of $C_{i,j}$ intersecting a color patch where $\norm{\xi_{u,v}-\xi_{i,j}}$ is very large for all $(u,v)$ in that patch. 
In that event, the corresponding terms in Eq.~\eqref{eq:def-density-estimate} vanish. 
Without further ado, we can state the main result of this section:

\begin{theorem}[Increasing density estimates]
	\label{th:decreasing-density-estimates-bicolor}
	Assume that~\ref{ass:bicolor} holds.  
	Assume further that $\kernelsize \geq 5$ and that $\norm{\oc_1-\oc_2}\geq 3\kernelsize$. 
	Then, for any $(i,j)\in\midim\cap \leftim$ such that $\abs{j-j_0}\leq \kernelwidth$, 
	\[
	\proba{P_{i,j} > P_{i,j+1}} \geq 1- 16\sigma^2
	\, .
	\]
\end{theorem}

In a nutshell, Theorem~\ref{th:decreasing-density-estimates-bicolor} states that, in the $\kernelwidth-$neighborhood of the boundary, the density estimates are increasing away from the border. 
The main consequence of Theorem~\ref{th:decreasing-density-estimates-bicolor} is the \textbf{absence of local maxima near the boundary} between two homogeneous color patches when Algorithm~\ref{algo:simplified-quickshift} is applied to~$P$. 
Indeed, the nearest neighbor of $(i,j+1)$ with highest density will be $(i,j)$ with high probability, and thus the edges in a band of width $\kernelwidth$ around the boundary are all pointing away from the boundary. 
An additional consequence is that \textbf{the boundary is well-recognized by Algorithm~\ref{algo:simplified-quickshift} provided that the color difference is large enough with respect to~$\kernelsize$.} 
Indeed, for all $i\in [\height]$, $(i,j_0)$ is linked to $(i,j_0-1)$ and $(i,j_0+1)$ to $(i,j_0+2)$ by symmetry, with high probability. 
Thus it is very unlikely that a point of $\leftim$ belongs to the superpixel corresponding to $\rightim$. 
We illustrate this phenomenon in Figure~\ref{fig:bicolor-local-maxima}. 
We note that it is straightforward to adapt the proof of Theorem~\ref{th:decreasing-density-estimates-bicolor} for a horizontal boundary. 

\textit{Sketch of the proof of Theorem~\ref{th:decreasing-density-estimates-bicolor}.}
As in the homogeneous case, we can compute the expected estimated density under Assumption~\ref{ass:bicolor}, which takes a slightly more involved expression:
\begin{align}
	&\expec{P_{i,j}} = \normcst_2 \cdot \sum_{(u,v)\in C_{i,j}\cap \leftim } \delta_{u,v} \label{eq:expec-bicolor}\\
	& + \normcst_2 \cdot \exp{\frac{-\norm{\oc_1-\oc_2}^2}{2(\kernelsize^2+2\sigma^2)}}\cdot\sum_{(u,v)\in C_{i,j}\cap \rightim } \delta_{u,v} \notag 
	\, .
\end{align}
Using Eq.~\eqref{eq:expec-bicolor}, we can look at the difference in expected estimated density between two neighbors. 
We prove that, if the color difference is large enough, the difference between the expected estimated density at two neighboring points is lower bounded by a term of order~$\kernelsize$. 
Namely, provided that $\norm{\oc_1-\oc_2}\geq 3\kernelsize$, 
\[
\expec{P_{i,j}} - \expec{P_{i,j+1}} \geq \frac{3\kernelsize}{2}
\, .
\]
Finally, since the same upper bound on the variance of $P_{i,j}$ as in the homogeneous case holds, we can use Chebyshev's inequality to find an event of large probability $\Omega$ on which both $\abs{P_{i,j}-\expec{P_{i,j}}}$ and $\abs{P_{i,j+1}-\expec{P_{i,j+1}}}$ are smaller than $3\kernelsize/4$. 
Therefore, on $\Omega$, by the triangle inequality, $P_{i,j}>P_{i,j+1}$. 
\qed

%%%%%%%%%%%%%%%%%%%%%%%%%%%%%%%%%%%%%%%%%%%%%%%%%%%%%%%%%%%%%%%%%%%%%%%%%%%%%%%%%

\begin{figure}[ht]
	\begin{center}
		\includegraphics[scale=0.2]{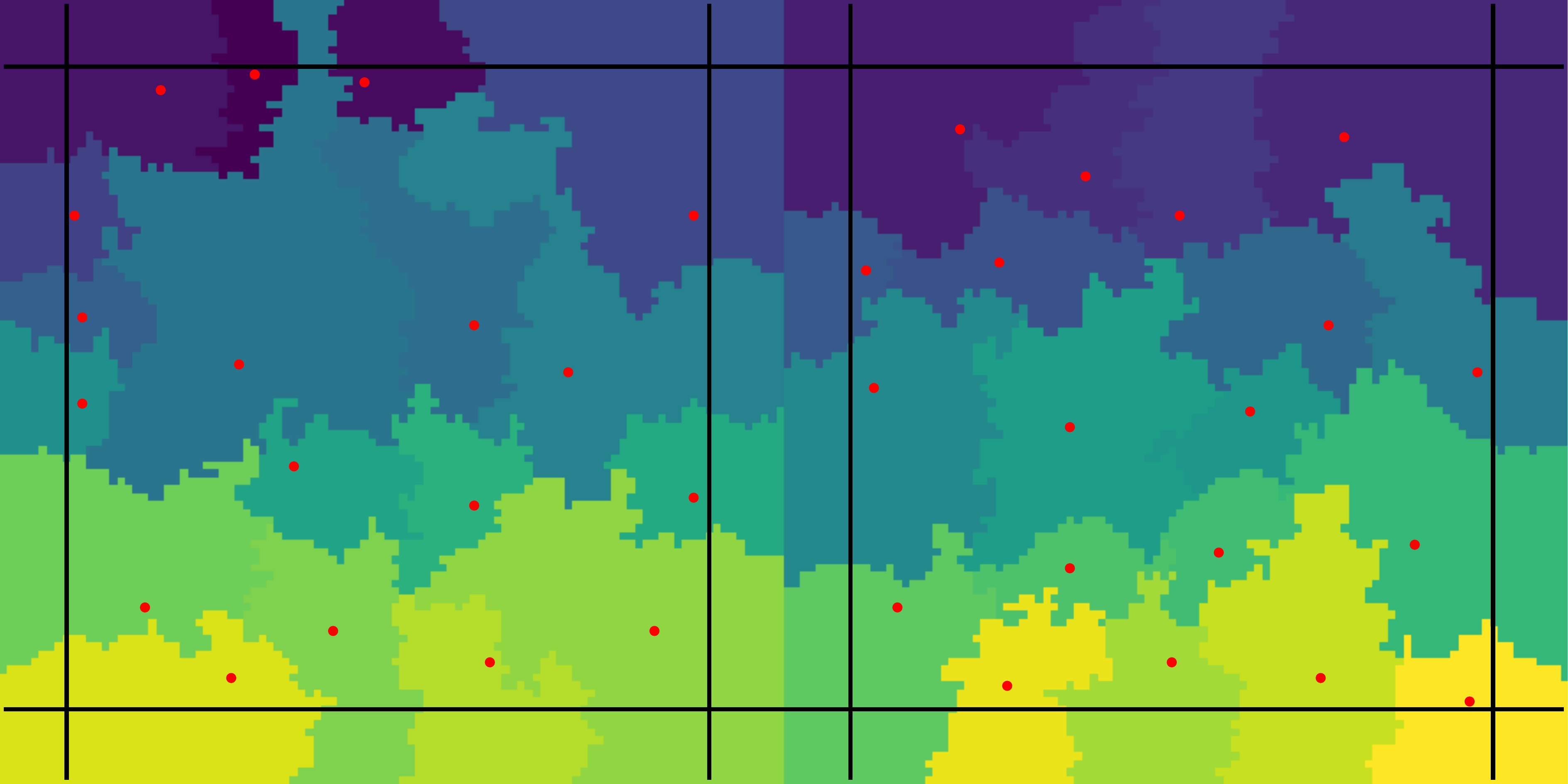}
	\end{center}
\vspace{-0.1in}
	\caption{\label{fig:bicolor-local-maxima}Illustration of Theorem~\ref{th:decreasing-density-estimates-bicolor}: quickshift segmentation of a bicolor image with a vertical border. In this setting, the density estimate is increasing away from the border between parts of the image of different colors, with high probability. As a consequence, quickshift does not find local maxima (the red dots) in this area (marked by the thick dark lines). }
\end{figure}

%%%%%%%%%%%%%%%%%%%%%%%%%%%%%%%%%%%%%%%%%%%%%%%%%%%%%%%%%%%%%%%%%%%%%%%%%%%%%%%%%

\section{EXPERIMENTS}
\label{sec:experiments}
 
In this section, we see how the claims of Section~\ref{sec:homogeneous} extend to real images. 
We consider three datasets:
(i) a subset of the ILSVRC2017 dataset \citep{ILSVRC15}. 
From the original $5500$ images, we manually picked $748$ of them where a large, rectangular part of the image is visually homogeneous in color. 
(ii) a subset of the CityScapes dataset \citep{Cordts2016Cityscapes}. 
More precisely, the $544$ images taken in Berlin from the \texttt{leftImg8bit\_trainvaltest} folder. 
(iii) a random subset of $500$ images from the \texttt{JPEGImages} folder of the Pascal VOC dataset \citep{everingham2015pascal}. 
Note that the images from the CityScapes dataset are much larger than those of the ILSVRC and PascalVOC dataset (typically $1024\times 2048$ \emph{vs} $300\times 500$). 
We first check the scaling with respect to the size of the image in Section~\ref{sec:patch-size} and the hyperparameters in Section~\ref{sec:parameters}, before concluding with a practical use-case in Section~\ref{sec:use-case}.

%%%%%%%%%%%%%%%%%%%%%%%%%%%%%%%%%%%%%%%%%%%%%%%%%%%%%%%%%%%%%%%%%%%%%%%%%%%%%%%%%

\subsection{Scaling with respect to the image size}
\label{sec:patch-size}

According to the discussion following Theorem~\ref{th:average-local-max}, for any choice of hyperparameters, the expected number of superpixels on a rectangular, homogeneous part of an image of size $h\times w$ should scale as $hw$. 
To test this, we conducted the following experiment. 
For a given choice of hyperparameters, we first segmented each of the images in our datasets with quickshift, yielding $n_\text{orig}$ superpixels. 
In a second step, we rescaled the image by a given ratio $\rho$, dividing both~$\height$ and~$\width$ by~$\rho$, and proceeded with the same experiment, yielding~$n_{\text{new}}$ superpixels.
According to Theorem~\ref{th:average-local-max}, one should observe $n_\text{orig}/n_\text{new}\approx \rho^2$. 
We report in Table~\ref{tab:results-scaling-size} the empirical average of $n_\text{orig}/n_\text{new}$ for $\rho\in\{2,3\}$: as expected, this empirical average is close to~$\rho^2$ in all the hyperparameter configurations that we tested. 
We note that this relationship is the weakest for small~$\maxdist$, as showed by the great variability of the ratio. 

%%%%%%%%%%%%%%%%%%%%%%%%%%%%%%%%%%%%%%%%%%%%%%%%%%%%%%%%%%%%%%%%%%%%%%%%%%%%%%%%%

\begin{table*}[t]
	\caption{\label{tab:results-scaling-size}Scaling with respect to the image size. We report the empirical average and standard deviation of $n_\text{orig}/n_\text{new}$ for all datasets and three hyperparameter configurations.}
	\vskip 0.15in
	\ra{1.0}
	\centering
\begin{tabular}{@{}cccccccc@{}} \toprule
	&  &\multicolumn{2}{c}{ILSVRC} & \multicolumn{2}{c}{Cityscapes} & \multicolumn{2}{c}{Pascal VOC} \\ 
	\cmidrule(rl){3-4} \cmidrule(rl){5-6} \cmidrule(rl){7-8} 
	$\kernelsize$ & $\maxdist$  & $\rho^2=4$ & $\rho^2=9$ & $\rho^2=4$ & $\rho^2=9$ & $\rho^2=4$ & $\rho^2=9$  \\ \midrule
	5 & $+\infty$ &  3.8 (0.7) & 7.9 (2.1) & 4.3 (0.2)&  10.2 (0.8) & 3.9 (0.7) & 8.2 (2.1) \\
	5 & 18 & 3.3 (0.8) & 6.3 (2.3) & 4.0 (0.2) & 9.0 (0.8) & 3.3 (0.7) & 6.4 (2.1)  \\
	5 & 10 & 4.3 (2.2) & 9.1 (7.7) & 3.9 (0.2) & 8.6 (0.9) & 4.3 (1.9) & 8.9 (5.6)  \\
	\bottomrule
\end{tabular}
	\vskip -0.1in
\end{table*}

%%%%%%%%%%%%%%%%%%%%%%%%%%%%%%%%%%%%%%%%%%%%%%%%%%%%%%%%%%%%%%%%%%%%%%%%%%%%%%%%%

\subsection{Scaling with respect to $\kernelsize$ and $\maxdist$}
\label{sec:parameters}

Next, we turn to the influence of the hyperparameters. 
We now consider a fixed image, and multiply both~$\kernelsize$ and~$\maxdist$ by a constant factor~$\kappa$. 
According to the discussion following Theorem~\ref{th:average-local-max}, one expects the average number of superpixels produced by quickshift on a flat portion of the image to be divided, roughly, by a factor $\kappa^2$. 
To test this hypothesis, we conducted a similar experiment to that of the previous section. 
First, we segmented the images of our dataset for a given choice of~$\kernelsize$ and~$\maxdist$, and counted the number of superpixels obtained (as in Section~\ref{sec:patch-size}, we call this number $n_\text{orig}$). 
We then segmented again all images with quickshift, this time with hyperparameters $(\kappa\kernelsize,\kappa\maxdist)$ and counted the number of superpixels obtained in that case ($n_\text{new}$). 
We report the average ratio $n_\text{orig}/n_{\text{new}}$ in Table~\ref{tab:results-scaling-parameters} for $\kappa\in\{0.5,2\}$. 
Again, we obtain values close to $1/\kappa^2$ as expected when $\maxdist$ is large with respect to~$\kernelsize$. 
For small~$\maxdist$, the ratio can be off by a constant factor. 

%%%%%%%%%%%%%%%%%%%%%%%%%%%%%%%%%%%%%%%%%%%%%%%%%%%%%%%%%%%%%%%%%%%%%%%%%%%%%%%%%

\begin{table*}[t]
	\caption{\label{tab:results-scaling-parameters}Scaling with respect to the hyperparameters. We report the empirical average and standard deviation of $n_\text{orig}/n_\text{new}$ for all datasets and three hyperparameter configurations.}
	\vskip 0.15in
	\ra{1.0}
	\centering
	\begin{tabular}{@{}cccccccc@{}} \toprule
	&  & \multicolumn{2}{c}{ILSVRC} & \multicolumn{2}{c}{Cityscapes} & \multicolumn{2}{c}{Pascal VOC} \\ 
	\cmidrule(rl){3-4} \cmidrule(rl){5-6} \cmidrule(rl){7-8} 
	$\kernelsize$ & $\maxdist$   & $\kappa^{-2}=0.25$ & $\kappa^{-2}=4$ & $\kappa^{-2}=0.25$ & $\kappa^{-2}=4$ & $\kappa^{-2}=0.25$ & $\kappa^{-2}=4$ \\ \midrule
	5 & $+\infty$ & 0.26 (0.04) & 3.7 (0.4) & 0.22 (0.01)&  3.8 (0.1) & 0.24 (0.05) & 3.8 (0.4) \\
	5 & 18 & 0.18 (0.05) & 10.9 (5.3) & 0.20 (0.01) & 5.5 (0.8) & 0.18 (0.04) & 10.9 (5.4) \\
	5 & 10  & 0.11 (0.05) & 14.4 (5.5)  & 0.17 (0.02) & 8.1 (1.3) & 0.11 (0.05) & 16.3 (7.6) \\
	\bottomrule
\end{tabular}
	\vskip -0.1in
\end{table*}

%%%%%%%%%%%%%%%%%%%%%%%%%%%%%%%%%%%%%%%%%%%%%%%%%%%%%%%%%%%%%%%%%%%%%%%%%%%%%%%%%

\subsection{A practical use case}
\label{sec:use-case}

Suppose that we have calibrated an image processing pipeline on downsized images with dimensions $(\height/\rho,\width/\rho)$ and now want to get back to the original size, in effect multiplying both $\height$ and $\width$ by a factor $\rho$, without the number of superpixels produced by quickshift changing too drastically. 
Combining the insights from Section~\ref{sec:patch-size} and~\ref{sec:parameters}, we see that running quickshift with hyperparameters $(\rho\kernelsize,\rho\maxdist)$ on the original image will yield approximately the same number of superpixels. 
This simple heuristic provides a way to scale quickshift hyperparameters. 
We illustrate this process is Figure~\ref{fig:rescaling}. 
In addition to recovering a similar number of superpixel, we note that their shapes are similar: increasing $\kernelsize$, in certain limits, does not seem to damage the density estimation step too much. 

As noted before, this heuristic works best when~$\maxdist$ is large, otherwise the deviations observed in Table~\ref{tab:results-scaling-size} and~\ref{tab:results-scaling-parameters} accumulate. 
We see two main reasons for this. 
First, we studied and derived a heuristic from a modified version of the quickshift algorithm. 
As explained in Section~\ref{sec:quickshift-modification}, these two versions can be quite different when $\maxdist$ is small. 
Second, even on flat portions of the image, the distribution of the pixel values does not quite satisfy Assumption~\ref{ass:flat-image}. 
For instance, the variance of the pixel values is typically much higher than in the regime where the closed-form expressions of Theorem~\ref{th:average-local-max} holds. 
There can also be a lot of spacial dependencies between pixel values, thus breaking the independent part of our assumption. 
We provide additional experiments in the Appendix (Section~\ref{sec:real-image}) to demonstrate this qualitatively. 

\begin{figure}[ht]
	\begin{center}
		\includegraphics[scale=0.23]{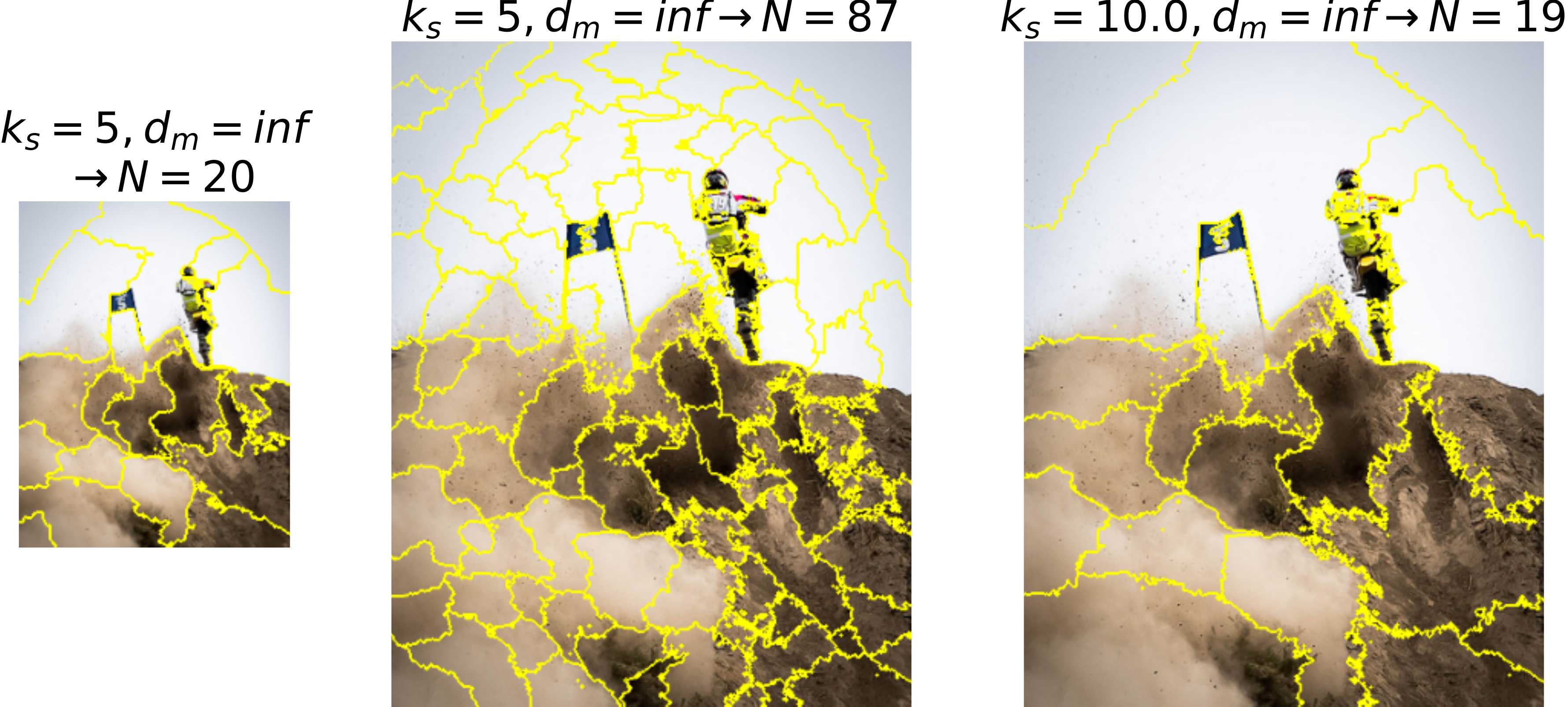}
	\end{center}
\vspace{-0.1in}
	\caption{\label{fig:rescaling}Scaling quickshift hyperparameters. \emph{Left:} segmentation of a downsized image (by a factor $\rho=2$) with $\kernelsize=5$ and $\maxdist=+\infty$, quickshift finds $20$ superpixels; \emph{middle:} segmentation of the original image with the same hyperparameters ($87$ superpixels); \emph{right:} segmentation of the original image with rescaled hyperparameters ($\kernelsize=2\times 5 = 10$), quickshift finds $19$ superpixels, approximately the same number as in the downsized image. }
\end{figure}

%%%%%%%%%%%%%%%%%%%%%%%%%%%%%%%%%%%%%%%%%%%%%%%%%%%%%%%%%%%%%%%%%%%%%%%%%%%%%%%%%

\section{CONCLUSION}
In this paper, we investigate the relationship between the number of superpixels produced by quickshift and the choice of the two main hyperparameters of the method:~$\kernelsize$ and~$\maxdist$. 
Theoretically, we find that, on a flat portion of the image and for a modified version of the algorithm, this number is proportional to the size of the patch and inversely proportional to the square of~$\kernelsize$ or~$\maxdist$ (depending on their relative position). 
Experimentally, we see that this scaling law is true to some extent for the original algorithm applied to real images. 
This provides a simple heuristic to keep the number of superpixels approximately constant when going from small to large images. 
We also show that quickshift accurately detects the borders of homogeneous patches provided that~$\kernelsize$ is large enough with respect to the color difference. 

As future work, our main focus is to tackle the original version of the algorithm, in order to capture better the behavior for small $\maxdist$. 
The main difficulty in doing so seems to be the extension of Lemma~\ref{lemma:key-lemma}, which seems challenging even considering i.i.d.\! inputs since the lookout area is no longer deterministic.  
We also want to obtain a uniform version of Theorem~\ref{th:P-close-to-Q}, which in turn would give a more rigorous justification in replacing~$P$ by~$Q$ in our analysis.

%%%%%%%%%%%%%%%%%%%%%%%%%%%%%%%%%%%%%%%%%%%%%%%%%%%%%%%%%%%%%%%%%%%%%%%%%%%%%%%%%

\subsubsection*{Acknowledgments}

The author wants to thank Elena Di Bernardino for constructive discussions during the writing of the paper. 
This work was supported by the NIM-ML project (ANR-21-CE23-0005-01). 

\vfill

%%%%%%%%%%%%%%%%%%%%%%%%%%%%%%%%%%%%%%%%%%%%%%%%%%%%%%%%%%%%%%%%%%%%%%%%%%%%%%%%%

\bibliographystyle{abbrvnat}
\bibliography{../biblio}

%%%%%%%%%%%%%%%%%%%%%%%%%%%%%%%%%%%%%%%%%%%%%%%%%%%%%%%%%%%%%%%%%%%%%%%%%%%%%%%

\clearpage
\appendix

\thispagestyle{empty}

\onecolumn
\aistatstitle{Supplementary material}

%%%%%%%%%%%%%%%%%%%%%%%%%%%%%%%%%%%%%%%%%%%%%%%%%%%%%%%%%%%%%%%%%%%%%%%%%%%%%%%

%\section*{ORGANIZATION OF THE APPENDIX}

In this appendix, we collect all missing proofs from the main paper and present some additional experimental results. 
It is organized as follows:
Theorem~\ref{th:P-close-to-Q}, stating that the density estimates can be approximated by the $Q_{i,j}$ term, is proved in Section~\ref{sec:density-approximation}. 
Section~\ref{sec:homogeneous-patches} is dedicated to the proof of Theorem~\ref{th:average-local-max}, which gives the expected number of local maxima in a flat portion of the image. 
In Section~\ref{sec:sharp-boundaries}, we prove Theorem~\ref{th:decreasing-density-estimates-bicolor} of the paper, stating that the density estimates are increasing away from the boundary between two homogeneous patches of the image. 
Truly technical results are collected in Section~\ref{sec:technical-results}. 
Finally, we present some additional experiments, mainly in relation to the use-case, in Section~\ref{sec:additional-experiments}.

%%%%%%%%%%%%%%%%%%%%%%%%%%%%%%%%%%%%%%%%%%%%%%%%%%%%%%%%%%%%%%%%%%%%%%%%%%%%%%%%

\section{DENSITY ESTIMATES APPROXIMATION}
\label{sec:density-approximation}

In this section, we provide a complete proof of Theorem~\ref{th:P-close-to-Q} of the paper. 
We follow the sketch of the proof provided in Section~\ref{sec:density} of the paper: after providing some elementary facts about $P_{i,j}$ in Section~\ref{sec:elementary}, we compute its H\'ajek projection onto the $\xi_{i,j}$s, $\proj_{i,j}$, in Section~\ref{sec:hajek-projection}. 
In Section~\ref{sec:proj-is-close}, we show that $P_{i,j}$ is close to $\proj_{i,j}$ with high probability, essentially proving that they have similar variances for large $\kernelsize$. 
Finally, we show in Section~\ref{sec:main-term} that $\proj_{i,j}$ is close to the main term $Q_{i,j}$, and we conclude. 

%%%%%%%%%%%%%%%%%%%%%%%%%%%%%%%%%%%%%%%%%%%%%%%%%%%%%%%%%%%%%%%%%%%%%%%%%%%%%%%%

\subsection{Elementary computations}
\label{sec:elementary}

Recall that, for any $(i,j)\in\image$, we defined the observation window
\begin{equation}
	\label{eq:def-Cij}
	C_{i,j} \defeq \{(u,v)\in \image, \text{ s.t. } \abs{i-u}\vee\abs{j-v}\leq \kernelwidth\}
	\, ,
\end{equation}
corresponding to all pixels of $\image$ located within a square of side $2\kernelwidth$ centered at $(i,j)$. 
Our main object of interest in this section is the density estimate $P_{i,j}=\sum_{(u,v)\in C_{i,j}} X_{u,v}$, where
\begin{equation*}
%	\label{eq:def-Xuv}
	X_{u,v}^{i,j} \defeq \exp{\frac{-\norm{\xi_{i,j}-\xi_{u,v}}^2}{2\kernelsize^2}}\delta_{u,v}
	\, ,
\end{equation*}
(Eq.~\eqref{eq:def-Xuv} in the paper) and
\begin{equation}
	\label{eq:def-delta}
	\delta_{u,v}^{i,j} \defeq \exp{\frac{-(i-u)^2-(j-v)^2}{2\kernelsize^2}}
	\, .
\end{equation}

As announced, we start our study by some elementary derivations, which we will use in the rest of this Appendix. 
Recall that we defined the normalization constant
\begin{equation}
	\label{eq:def-normcst}
	\forall p \geq 1, \qquad \normcst_p \defeq \left(\frac{\kernelsize^2}{\kernelsize^2+p\sigma^2}\right)^{\frac{3}{2}}
	\, .
\end{equation}
This normalization constant is important since it appears in most of the computations involving the expected value of the $X_{u,v}$ random variables under Assumption~\ref{ass:flat-image}. 
Indeed, we have the following: 

\begin{lemma}[Moments of $X_{u,v}$]
	\label{lemma:Xuv-moments-generic}
	Let $(i,j)\in I$ and $(u,v)\in C_{i,j}\setminus \{(i,j)\}$. 
	Assume that $\xi_{i,j}\sim \gaussian{\oc_1}{\sigma^2\Identity_3}$ and $\xi_{u,v}\sim \gaussian{\oc_2}{\sigma^2\Identity_3}$ with $\oc_1,\oc_2\in\Reals^3$. 
	Then, for any $p\geq 1$,
	\[
	\condexpec{X_{u,v}^p}{\xi_{i,j}} = \normcst_p\cdot \exp{\frac{-p\norm{\xi_{i,j} - \oc_2}^2}{2(\kernelsize^2 + p\sigma^2)}}\cdot \delta_{u,v}^p
	\, ,
	\]
	and
	\[
	\expec{X_{u,v}^p} = \normcst_{2p}
	\cdot \exp{\frac{-p\norm{\oc_1 - \oc_2}^2}{2(\kernelsize^2 + 2p\sigma^2)}}
	\cdot \delta_{u,v}^p
	\, .
	\]
\end{lemma}

\pagebreak

Under Assumption~\ref{ass:flat-image}, we can easily deduce from Lemma~\ref{lemma:Xuv-moments-generic} the expected value of $X_{u,v}$, that is,
\begin{equation}
	\label{eq:Xuv-expectation}
	\expec{X_{u,v}} = \normcst_2 \cdot \delta_{u,v}
	\, .
\end{equation}
Recall that we defined 
\begin{equation}
	\label{eq:def-Delta}
	\Delta_{i,j} \defeq \sum_{(u,v)\in C_{i,j}} \delta_{u,v}
	\, ,
\end{equation}
by linearity we find that 
\begin{equation}
	\label{eq:Pij-expected-value}
	\expec{P_{i,j}} = \normcst_2 \cdot \Delta_{i,j}
	\, .
\end{equation}

\begin{proof}
	We begin by the computation of the conditional expectation. 
	Conditionally to $\xi_{i,j}$, we note that $X_{u,v}^p$ can be written as a product of three independent random variables. 
	Namely,
	\begin{align}
		\condexpec{X_{u,v}^p}{\xi_{i,j}} &= \condexpec{\exp{\frac{-p\norm{\xi_{i,j}-\xi_{u,v}}^2}{2\kernelsize^2}}}{\xi_{i,j}}\cdot \delta_{u,v}^p \notag \\
		&= \prod_{k=1}^3 \condexpec{\exp{\frac{-p(\xi_{i,j,k}-\xi_{u,v,k})^2}{2\kernelsize^2}}}{\xi_{i,j,k}}\cdot \delta_{u,v}^p \label{eq:aux-key-gaussian-1}
		\, .
	\end{align}
	Let us fix $k\in\{1,2,3\}$.
	The inner term can be written
	\begin{align*}
		\condexpec{\exp{\frac{-(\xi_{i,j,k}-\xi_{u,v,k})^2}{2\kernelsize^2}}}{\xi_{i,j,k}} &= \int \exp{\frac{-p(\xi_{i,j,k}-x)^2}{2\kernelsize^2}} \cdot \exp{\frac{-(x-c_{2,k})^2}{2\sigma^2}} \frac{\Diff x}{\sigma\sqrt{2\pi}}
		\, ,
	\end{align*}
	since $\xi_{u,v}\sim \gaussian{c_2}{\sigma^2\Identity_3}$. 
	We apply Lemma~\ref{lemma:key-gaussian-computation} with $a=\xi_{i,j,k}$, $b=c_{2,k}$, $c=\kernelsize/\sqrt{p}$, and $d=\sigma$ to obtain
	\[
	\condexpec{\exp{\frac{-p(\xi_{i,j,k}-\xi_{u,v,k})^2}{2\kernelsize^2}}}{\xi_{i,j,k}} = \left(\frac{\kernelsize^2}{\kernelsize^2+p\sigma^2} \right)^{\frac{1}{2}} \cdot \exp{\frac{-p(\xi_{i,j,k}-c_{2,k})^2}{2(\kernelsize^2+p\sigma^2)}}
	\, .
	\]
	Coming back to Eq.~\eqref{eq:aux-key-gaussian-1}, we find the first statement of the lemma to be true. 
	
	% marginalizing
	To take the expectation with respect to $\xi_{i,j}$, we take the same road and first write
	\begin{equation}
		\label{eq:aux-key-gaussian-2}
		\expec{X_{u,v}^p} = \left(\frac{\kernelsize^2}{\kernelsize^2+p\sigma^2}\right)^{\frac{3}{2}} \cdot \prod_{k=1}^3 \expec{\exp{\frac{-p(\xi_{i,j,k}-c_{2,k})^2}{2(\kernelsize^2+p\sigma^2)}} } \cdot \delta_{u,v}^p
		\, .
	\end{equation}
	Fix $k\in\{1,2,3\}$. 
	The inner term can be written
	\[
	\expec{\exp{\frac{-p(\xi_{i,j,k}-c_{2,k})^2}{2(\kernelsize^2+p\sigma^2)}} } = \int \exp{\frac{-p(x-c_{2,k})^2}{2(\kernelsize^2+p\sigma^2)}}\cdot \exp{\frac{-(x-c_{1,k})^2}{2\sigma^2}} \frac{\Diff x}{\sigma\sqrt{2\pi}}
	\, .
	\]
	We apply Lemma~\ref{lemma:key-gaussian-computation} with $a=c_{2,k}$, $b=c_{1,k}$, $c^2=(\kernelsize^2+p\sigma^2)/p$, and $d^2=\sigma^2$ to obtain
	\[
	\expec{\exp{\frac{-p(\xi_{i,j,k}-c_{2,k})^2}{2(\kernelsize^2+p\sigma^2)}} } = \left(\frac{\kernelsize^2+p\sigma^2}{\kernelsize^2+2p\sigma^2} \right)^{\frac{1}{2}} \cdot \exp{\frac{-p(c_{1,k}-c_{2,k})^2}{2(\kernelsize^2+2p\sigma^2)}}
	\, .
	\]
	Coming back to Eq.~\eqref{eq:aux-key-gaussian-2}, we have proved the second statement of the lemma. 
\end{proof}

We now introduce two important functions for our study. 

\begin{definition}[$\psi$ functions]
	\label{def:psi-functions}
	For any $t\geq 0$, we let
	\[
	\psi_1(t) \defeq \frac{1}{(1+4t)^{\frac{3}{2}}} - \frac{1}{(1+2t)^3}
	\quad \text{and}\quad
	\psi_2(t) \defeq \frac{1}{(1+t)^{\frac{3}{2}}(1+3t)^{\frac{3}{2}}} - \frac{1}{(1+2t)^3}
	\, .
	\]
\end{definition}
The main reason in introducing these auxiliary functions is their appearance in the variance computations that are key to our analysis. 
For instance, Lemma~\ref{lemma:Xuv-moments-generic} implies that
\begin{equation}
	\label{eq:Xuv-variance}
	\var{X_{u,v}} = \left[\left(\frac{\kernelsize^2}{\kernelsize^2+4\sigma^2}\right)^{\frac{3}{2}} - \left(\frac{\kernelsize^2}{\kernelsize^2+2\sigma^2}\right)^3\right]\cdot \delta_{u,v}^2 = \psi_1\left(\frac{\sigma^2}{\kernelsize^2}\right)\cdot \delta_{u,v}^2
	\, .
\end{equation}
We also have the following:

\begin{lemma}[Covariance structure of the $X_{u,v}$]
	Under Assumption~\ref{ass:flat-image}, for any distinct $(u,v),(u',v')\in C_{i,j}$, 
	\begin{equation}
		\label{eq:Xuv-covariance}
		\cov{X_{u,v}}{X_{u',v'}} = \psi_2\left(\frac{\sigma^2}{\kernelsize^2}\right)\cdot \delta_{u,v}\delta_{u',v'}
		\, .
	\end{equation}
\end{lemma}

Some technical facts about $\psi_1$ and $\psi_2$ are collected in Section~\ref{sec:psi-functions-facts}. 
For instance, according to Lemma~\ref{lemma:psi-2-bounds}, $\psi_2(t)$ is positive for any $t\geq 0$. 
We deduce that, under Assumption~\ref{ass:flat-image}, $X_{u,v}$ and $X_{u',v'}$ are positively correlated for any $(u,v),(u',v')\in C_{i,j}$. 

\begin{proof}
	Let $(u,v),(u',v')\in C_{i,j}$ be distinct. 
	We write
	\begin{align*}
		\condexpec{X_{u,v}X_{u',v'}}{\xi_{i,j}} &= \condexpec{X_{u,v}}{\xi_{i,j}} \cdot \condexpec{X_{u',v'}}{\xi_{i,j}} \\
		&= \left(\frac{\kernelsize^2}{\kernelsize^2+\sigma^2}\right)^{3} \cdot \exp{\frac{-\norm{\xi_{i,j}-\oc}^2}{2((\kernelsize^2+\sigma^2)/2)}}\cdot \delta_{u,v}\delta_{u',v'} 
		\, ,
	\end{align*}
	where we used Lemma~\ref{lemma:Xuv-moments-generic} with $p=1$. 
	Marginalizing with respect to $\xi_{i,j}$, we obtain
	\[
	\expec{X_{u,v}X_{u',v'}} = \left(\frac{\kernelsize^2}{\kernelsize^2+\sigma^2}\right)^{3} \cdot \left(\frac{\kernelsize^2+\sigma^2}{\kernelsize^2+3\sigma^2} \right)^{\frac{3}{2}}\cdot \delta_{u,v}\delta_{u',v'} 
	\, ,
	\]
	again using Lemma~\ref{lemma:Xuv-moments-generic}. 
	Eq.~\eqref{eq:Xuv-expectation} combined with some straightforward algebra yields the promised result. 
\end{proof}

Putting together Eq.~\eqref{eq:Xuv-variance} and Eq.~\eqref{eq:Xuv-covariance}, we see that
\begin{align}
	\label{eq:Pij-variance}
	\var{P_{i,j}} &= \psi_2\left(\frac{\sigma^2}{\kernelsize^2}\right)\cdot \sum_{\substack{(u,v),(u',v')\in C_{i,j} \\ (u,v)\neq (u',v')}} \delta_{u,v}\delta_{u',v'} + \psi_1\left(\frac{\sigma^2}{\kernelsize^2}\right) \cdot \sum_{(u,v)\in C_{i,j}} \delta_{u,v}^2
	\, .
\end{align}

It can be cumbersome to work directly with this expression. 
We propose the following lower bound for the variance of $P_{i,j}$:

\begin{lemma}[Lower bound on $\var{P_{i,j}}$]
	\label{lemma:Pij-variance-lower-bound}
	Assume that \ref{ass:flat-image} holds. 
	Then, for any $(i,j)\in \midim(\kernelwidth)$,  
	\[
	\var{P_{i,j}} \geq \frac{\sigma^4}{\kernelsize^4}\cdot \Delta_{i,j}^2 
	\, .
	\]
\end{lemma}

\begin{proof}
	On $[0,1/25]$, both $\psi_1$ and $\psi_2$ are lower bounded by $t^2$ (Lemma~\ref{lemma:psi-1-bounds} and~\ref{lemma:psi-2-bounds}). 
	Therefore, according to Eq.~\eqref{eq:Pij-variance}, 
	\begin{align*}
		\var{P_{i,j}} &= \psi_2\left(\frac{\sigma^2}{\kernelsize^2}\right)\cdot \sum_{\substack{(u,v),(u',v')\in C_{i,j} \\ (u,v)\neq (u',v')}} \delta_{u,v}\delta_{u',v'} + \psi_1\left(\frac{\sigma^2}{\kernelsize^2}\right) \cdot \sum_{(u,v)\in C_{i,j}} \delta_{u,v}^2 \\
		&\geq \frac{\sigma^4}{\kernelsize^4} \cdot \left( \sum_{\substack{(u,v),(u',v')\in C_{i,j} \\ (u,v)\neq (u',v')}} \delta_{u,v}\delta_{u',v'} +\sum_{(u,v)\in C_{i,j}} \delta_{u,v}^2  \right) \\
		&= \frac{\sigma^4}{\kernelsize^4} \cdot \Delta_{i,j}^2
		\, .
	\end{align*}
\end{proof}

We also have an upper bound on the variance of $P_{i,j}$.

\begin{lemma}[Upper bound on $\var{P_{i,j}}$]
	\label{lemma:Pij-variance-upper-bound}
	Assume that~\ref{ass:flat-image} holds.
	In addition, suppose that $\kernelsize\geq 5$. 
	Then, for any $(i,j)\in \midim(\kernelwidth)$,  
	\[
	\var{P_{i,j}} \leq 107\sigma^4
	\, .
	\]
\end{lemma}

\begin{proof}
	Since $\sigma^2\leq \kernelsize^2/25$, according to Lemma~\ref{lemma:psi-1-bounds}, we have
	\[
	\var{X_{u,v}} \leq \frac{6\sigma^4}{\kernelsize^4}\cdot \delta_{u,v}^2 
	\, ,
	\]
	and according to Lemma~\ref{lemma:psi-2-bounds},
	\[
	\cov{X_{u,v}}{X_{u',v'}} \leq \frac{2\sigma^4}{\kernelsize^4}\cdot \delta_{u,v}\delta_{u',v'}
	\, .
	\]
	Therefore,
	\[
	\var{P_{i,j}} \leq \frac{6\sigma^4}{\kernelsize^4}\sum \delta_{u,v}^2 + \frac{2\sigma^4}{\kernelsize^4}\sum \delta_{u,v}\delta_{u',v'} = \frac{2\sigma^4}{\kernelsize^4}\Delta_{i,j}^2 + \frac{4\sigma^4}{\kernelsize^4}\sum\delta_{u,v}^2
	\, .
	\]
	Now, we use Lemma~\ref{lemma:Delta-bounds} to bound $\Delta_{i,j}$ and Eq.~\eqref{eq:sum-of-squares-upper-bound} to bound the sum of squares. 
	We obtain
	\begin{align*}
		\var{P_{i,j}} &\leq \frac{2\sigma^4}{\kernelsize^4}\cdot \frac{(5\kernelsize+2)^4}{2^4} + \frac{4\sigma^4}{\kernelsize^4}\cdot \frac{(9\kernelsize+5)^2}{25}
		\, .
	\end{align*}
	We conclude by studying the last display as a function of $\kernelsize$ on $[5,+\infty)$. 
\end{proof}

\begin{remark}
	The main message of this section is that one can compute the moments of $X_{u,v}$ (and thus $P_{i,j}$) under parametric assumptions on the pixel values. 
	That the distribution of the noise is Gaussian is not crucial, the same computations could be made with another p.d.f., leading to different expressions for the moments and thus the variance. 
\end{remark}

%%%%%%%%%%%%%%%%%%%%%%%%%%%%%%%%%%%%%%%%%%%%%%%%%%%%%%%%%%%%%%%%%%%%%%%%%%%%%%%%%%%

\subsection{H\'ajek projection of the density estimates}
\label{sec:hajek-projection}

In this section, we study the H\'ajek projection of $P_{i,j}$ under Assumption~\ref{ass:flat-image}. 
We refer to Chapter~11 in~\citet{van2000asymptotic} for an introduction to H\'ajek projections. 
We start by the derivation of the projection itself, which is given without proof in the paper as Eq.~\eqref{eq:hajek-projection}.

\begin{proposition}[Hajek projection of the density estimates]
	\label{lemma:density-estimate-hajek-projection}
	Under Assumption~\ref{ass:flat-image}, the Hajek projection of $P_{i,j}$ onto the set of random variables $\{\xi_{i,j},(i,j)\in\image\}$ is given by 
	\[
	\forall (i,j)\in\image, \quad
	\proj_{i,j} = \normcst_1 \cdot \exp{\frac{-\norm{\xi_{i,j}-\oc }^2}{2(\kernelsize^2+\sigma^2)}} \cdot \Delta_{i,j} + \normcst_1\cdot \sum_{\substack{(u,v)\in C_{i,j} \\ (u,v)\neq (i,j)}}  \left[\exp{\frac{-\norm{\xi_{u,v}-\oc}^2}{2(\kernelsize^2+\sigma^2)}}- \left(\frac{\kernelsize^2+\sigma^2}{\kernelsize^2+2\sigma^2}\right)^{\frac{3}{2}}\right]\cdot \delta_{u,v} 
	\, .
	\]
\end{proposition}

\begin{proof}
	According to Lemma~11.10 in \citet{van2000asymptotic},
	\begin{equation}
		\label{eq:hajek-master}
		\proj_{i,j} = \expec{P_{i,j}} + \sum_{(u,v)\in C_{i,j}} (\condexpec{P_{i,j}}{\xi_{u,v}} - \expec{P_{i,j}})
		\, ,
	\end{equation}
	since the $\xi_{u,v}$ with $(u,v)\in C_{i,j}$ are the only random variables from which $P_{i,j}$ depends. 
	By linearity, computing $\proj_{i,j}$ is thus a matter of computing $\condexpec{P_{i,j}}{\xi_{u,v}}$, for all $(u,v)\in C_{i,j}$. 
	
	% first case
	If $(u,v)=(i,j)$, then Lemma~\ref{lemma:Xuv-moments-generic} with $p=1$ gives us
	\[
	\condexpec{P_{i,j}}{\xi_{i,j}} = \normcst_1 \cdot \sum_{(u,v)\in C_{i,j}} \exp{\frac{-\norm{\xi_{i,j}-\oc}^2}{2(\kernelsize^2+\sigma^2)}}\cdot \delta_{u,v} = \normcst_1 \cdot \exp{\frac{-\norm{\xi_{i,j}-\oc}^2}{2(\kernelsize^2+\sigma^2)}} \cdot \Delta_{i,j}
	\, .
	\]
	
	% other cases
	Let us now assume that $(u,v)\neq (i,j)$. 
	By linearity of the conditional expectation, the main computation is thus $\condexpec{X_{u',v'}}{\xi_{u,v}}$. 
	There are two cases. 
	First, $(u,v)=(u',v')$. 
	Then similarly to the first part of the proof, we obtain
	\[
	\condexpec{X_{u,v}}{\xi_{u,v}} =
	\normcst_1  \cdot \exp{\frac{-\norm{\xi_{u,v}-\oc }^2}{2(\kernelsize^2+\sigma^2)}} \cdot \delta_{u,v}
	\, .
	\]
	Second, if $(u,v)\neq (u',v')$, then, by independence, $\condexpec{X_{u',v'}}{\xi_{u,v}}=\expec{X_{u',v'}}$. 
	Keeping in mind that $X_{i,j}=1$ a.s., we deduce that 
	\begin{align*}
		\condexpec{P_{i,j}}{X_{u,v}} &= \sum_{(u',v')\in C_{i,j}} \condexpec{X_{u',v'}}{X_{u,v}} \\
		&= 1 + \condexpec{X_{u,v}}{\xi_{u,v}} + \sum_{\substack{(u',v')\in C_{i,j} \\ (u',v')\neq (i,j) \\ (u',v')\neq (u,v)}} \expec{X_{u',v'}}
		\, .
	\end{align*}
	Therefore
	\[
	\condexpec{P_{i,j}}{X_{u,v}} - \expec{P_{i,j}} =
	\normcst_1 \cdot \left[\exp{\frac{-\norm{\xi_{u,v}-\oc}^2}{2(\kernelsize^2+\sigma^2)}}- \left(\frac{\kernelsize^2+\sigma^2}{\kernelsize^2+2\sigma^2}\right)^{\frac{3}{2}}\right] \cdot \delta_{u,v}
	\, .
	\]
	We conclude the proof by simplifying the $\expec{P_{i,j}}$ in Eq.~\eqref{eq:hajek-master} with the one in the $(u,v)=(i,j)$ term. 
\end{proof}

Next, we will show that the variance ratio $\smallvar{\proj_{i,j}}/\var{P_{i,j}}$ is close to $1$. 
We first derive the variance of $\proj_{i,j}$. 

\begin{lemma}[Variance of the Hajek projection]
	\label{lemma:hajek-projection-variance-computation}
	Under Assumption~\ref{ass:flat-image}, 
	\begin{align*}
		\smallvar{\proj_{i,j}} &= \psi_2\left(\frac{\sigma^2}{\kernelsize^2}\right)\cdot\left( \Delta_{i,j}^2 +\sum_{\substack{(u,v)\in C_{i,j} \\ (u,v)\neq (i,j)}} \delta_{u,v}^2\right)  \\
		&= 2\psi_2\left(\frac{\sigma^2}{\kernelsize^2}\right)\cdot \sum_{(u,v)\in C_{i,j}} \delta_{u,v}^2 + \psi_2\left(\frac{\sigma^2}{\kernelsize^2}\right)\cdot \sum_{\substack{(u,v),(u',v')\in C_{i,j} \\ (u,v)\neq (u',v')}} \delta_{u,v}\delta_{u',v'} + \psi_2\left(\frac{\sigma^2}{\kernelsize^2}\right)
		\, .
	\end{align*}
\end{lemma}

\begin{proof}
	By construction, the H\'ajek projection is a sum of independent variables, thus the main computation is that of the variance of the exponential term. 
	Using Lemma~\ref{lemma:Xuv-moments-generic}, we see that
	\[
	\var{\exp{\frac{-\norm{\xi_{i,j}-\oc}^2}{2(\kernelsize^2+\sigma^2)}}} = \left(\frac{\kernelsize^2+\sigma^2}{\kernelsize^2+3\sigma^2}\right)^{\frac{3}{2}} - \left(\frac{\kernelsize^2+\sigma^2}{\kernelsize^2+2\sigma^2}\right)^3
	\, .
	\]
	Multiplying the previous display by $\normcst_1^2$, we recognize $\psi_2(\sigma^2/\kernelsize^2)$. 
	Simple algebra yields the result, keeping in mind the definition of $\Delta_{i,j}$ and that $\delta_{i,j}=1$.
\end{proof}

%%%%%%%%%%%%%%%%%%%%%%%%%%%%%%%%%%%%%%%%%%%%%%%%%%%%%%%%%%%%%%%%%%%%%%%%%

\subsection{$P_{i,j}$ is close to $\proj_{i,j}$}
\label{sec:proj-is-close}

In this section, we show that $P_{i,j}$ is close to $\proj_{i,j}$ with high probability. 
The main difficulty here is to prove that the variance ratio is close to $1$, which is achieved by the next proposition. 

\begin{proposition}[Controlling the variance ratio]
	\label{prop:variance-ratio-convergence}
	Assume that~\ref{ass:flat-image} holds. 
	Assume further that $\kernelsize \geq 5$. 
	Then, for any $(i,j)\in\midim(\kernelwidth)$, 
	\[
	\abs{\frac{\smallvar{\proj_{i,j}}}{\var{P_{i,j}}} - 1} \leq \frac{4}{5\kernelsize^2}
	\, .
	\]
\end{proposition}

\begin{proof}
	% numerator
	Using Lemma~\ref{lemma:hajek-projection-variance-computation} and Eq.~\eqref{eq:Pij-variance}, we first write
	\begin{align*}
		\abs{\smallvar{\proj_{i,j}} - \var{P_{i,j}}} &\leq \abs{2\psi_2\left(\frac{\sigma^2}{\kernelsize^2}\right)-\psi_1\left(\frac{\sigma^2}{\kernelsize^2}\right)} \cdot \sum_{(u,v)\in C_{i,j}} \delta_{u,v}^2 + \psi_2 \left(\frac{\sigma^2}{\kernelsize^2}\right) \\
		&\leq  \frac{3\sigma^4}{\kernelsize^4} \cdot \sum_{(u,v)\in C_{i,j}} \delta_{u,v}^2 + \frac{3\sigma^4}{2\kernelsize^4}
		\, ,
	\end{align*}
	where we used the fact that $\abs{2\psi_2(t)-\psi_1(t)}\leq 3t^2$ and $\psi_2(t)\leq 3t^2/2$ for all $t\in [0,1/25]$ (a consequence of Lemma~\ref{lemma:psi-technical} and~\ref{lemma:psi-2-bounds}). 
	
	% denominator
	Now recall that, according to Lemma~\ref{lemma:Pij-variance-lower-bound}, $\var{P_{i,j}} \geq \frac{\sigma^4}{\kernelsize^4} \Delta_{i,j}^2$. 
	We deduce that 
	\begin{align*}
		\abs{\frac{\smallvar{\proj_{i,j}}}{\var{P_{i,j}}} - 1} &\leq \frac{3\sigma^4}{\kernelsize^4} \frac{\kernelsize^4\sum_{(u,v)\in C_{i,j}} \delta_{u,v}^2}{\sigma^4\Delta_{i,j}^2} + \frac{3\sigma^4}{2\kernelsize^4} \cdot \frac{\kernelsize^4}{\sigma^4\Delta_{i,j}^2} \\
		&\leq \frac{3}{4\kernelsize^2} + \frac{3}{2}\cdot \frac{1}{(2\kernelsize+1)^4} \tag{Lemma~\ref{lemma:sum-of-squares-negligible} and~\ref{lemma:Delta-bounds}} 
		\, .
	\end{align*}
	We deduce the result by studying the last display as a function of $\kernelsize$ on $[5,+\infty)$. 
\end{proof}

Transferring this control on the variance ratio to the random variable and its projection is a classical idea when dealing with H\'ajek projections:

\begin{corollary}[$P_{i,j}$ and $\proj_{i,j}$ are close, in probability]
	\label{cor:projection-close-to-density}
	Assume that~\ref{ass:flat-image} holds. 
	Assume further that $\kernelsize\geq 5$.
	Let $\epsilon > 0$. 
	Then, for any fixed $(i,j)\in\midim(\kernelwidth)$, 
	\[
	\proba{\abs{\frac{P_{i,j}-\expec{P_{i,j}}}{\sqrt{\var{P_{i,j}}}} -\frac{\proj_{i,j}-\smallexpec{\proj_{i,j}}}{\sqrt{\smallvar{\proj_{i,j}}}}} > \epsilon} \leq \frac{1}{\kernelsize^2\epsilon^2}
	\, .
	\]
\end{corollary}

\begin{proof}
	The assumptions of Proposition~\ref{prop:variance-ratio-convergence} are satisfied, therefore
	\[
	\abs{\frac{\smallvar{\proj_{i,j}}}{\var{P_{i,j}}} - 1} \leq \frac{4}{5\kernelsize^2}
	\, ,
	\]
	which is smaller than $1/30$ by our choice of $\kernelsize$. 
	Since $\sqrt{\cdot}$ is $3/5$-Lipschitz on $[1-1/30,1+1/30]$, it holds that 
	\begin{equation}
		\label{eq:aux-projection-close-1}
		\abs{\sqrt{\frac{\smallvar{\proj_{i,j}}}{\var{P_{i,j}}}} - 1} \leq \frac{3}{5}\cdot \frac{4}{5\kernelsize^2} \leq \frac{1}{2\kernelsize^2}
		\, .
	\end{equation}
	A careful reading of the proof of Theorem~11.2 in \cite{van2000asymptotic} reveals that a factor $2$ appears, which concludes the proof. 
\end{proof}

%%%%%%%%%%%%%%%%%%%%%%%%%%%%%%%%%%%%%%%%%%%%%%%%%%%%%%%%%%%%%%%%%%%%%%%%%%%%%%%%

\subsection{Main term}
\label{sec:main-term}

We now show that $R_{i,j}=\proj_{i,j}-Q_{i,j}$ is negligible in probability when $\sigma^2$ is small. 
More precisely, we derive a variance bound for~$R_{i,j}$. 

\begin{lemma}[$R_{i,j}$ is negligible]
	\label{lemma:Rij-variance}
	Assume that~\ref{ass:bicolor} holds. 
	Assume further that $\kernelsize\geq 5$. 
	Then 
	\[
	\var{R_{i,j}} \leq \frac{\sigma^2}{4}
	\, .
	\]
\end{lemma}

\begin{proof}
	It is clear that $\expec{R_{i,j}}=0$ and from the proof of Lemma~\ref{lemma:hajek-projection-variance-computation}, we see that
	\begin{align*}
		\var{R_{i,j}} &= \psi_2\left(\frac{\sigma^2}{\kernelsize^2} \right) \cdot \sum_{\substack{(u,v)\in C_{i,j} \\ (u,v)\neq (i,j)}} \delta_{u,v}^2 \\
		&\leq \frac{3\sigma^4}{2\kernelsize^4} \cdot \frac{(9\kernelsize+5)^2}{25} \tag{Eq.~\eqref{eq:sum-of-squares-upper-bound}}
		\, .
	\end{align*}
	Since $\sigma^2\leq \kernelsize^2/25$, we see that
	\[
	\var{R_{i,j}} \leq \frac{3(9\kernelsize+5)^2}{2\cdot 25\cdot 25\kernelsize^2}\cdot \sigma^2
	\, .
	\]
	We conclude by studying the previous display as a function of $\kernelsize$ on $[5,+\infty)$. 
\end{proof}

To conclude this section, we now prove Theorem~\ref{th:P-close-to-Q} of the paper, which is re-stated here for completeness' sake.  

\begin{theorem}[$P_{i,j}$ is close to $Q_{i,j}$, with high probability]
%	\label{th:P-close-to-Q}
	Assume that~\ref{ass:flat-image} holds. 
	Suppose furthermore that $\kernelsize\geq 5$. 
	Let $(i,j)\in\midim(\kernelwidth)$. 
	Then, for any $\epsilon >0$, 
	\[
	\proba{\abs{P_{i,j} - Q_{i,j}} > \epsilon} \leq \frac{71\sigma^2}{\epsilon^2}
	\, .
	\]
\end{theorem}

\begin{proof}
	We first write
	\begin{equation}
		\label{eq:PQ-aux-1}
		\proba{\abs{P_{i,j} - Q_{i,j}} > \epsilon} \leq \smallproba{\smallabs{P_{i,j}-\proj_{i,j}} > \epsilon/2} + \proba{\abs{R_{i,j} } > \epsilon / 2}
		\, .
	\end{equation}
	
	% term 2
	Let us focus on the second term, and notice that $\expec{R_{i,j}}=0$. 
	Since the assumptions of Lemma~\ref{lemma:Rij-variance} are satisfied, we know that $\var{R_{i,j}}\leq \sigma^2/4$. 
	Therefore, by Chebyshev's inequality, 
	\[
	\proba{\abs{R_{i,j} } > \epsilon / 2} \leq \frac{2^2\sigma^2}{4\epsilon^2} = \frac{\sigma^2}{\epsilon^2}
	\, .
	\]
	
	% term 1
	Now we turn back to the first term in Eq.~\eqref{eq:PQ-aux-1}. 
	Noting that $\expec{P_{i,j}}=\smallexpec{\proj_{i,j}}$, we have
	\begin{align}
		\smallproba{\smallabs{P_{i,j}-\proj_{i,j}} > \epsilon/2} &= \proba{\abs{\frac{P_{i,j} - \expec{P_{i,j}}}{\sqrt{\smallvar{\proj_{i,j}}}} -\frac{\proj_{i,j} - \expec{\proj_{i,j}}}{\sqrt{\smallvar{\proj_{i,j}}}}} > \frac{\epsilon}{2\sqrt{\smallvar{\proj_{i,j}}}}} \notag \\
		&\leq \proba{\abs{\frac{P_{i,j} - \expec{P_{i,j}}}{\sqrt{\var{P_{i,j}}}} -\frac{\proj_{i,j} - \expec{\proj_{i,j}}}{\sqrt{\smallvar{\proj_{i,j}}}}} > \frac{\epsilon}{4\sqrt{\smallvar{\proj_{i,j}}}}} \label{eq:PQ-aux-2} \\
		&+ \proba{\abs{P_{i,j}-\expec{P_{i,j}}} \cdot \abs{\frac{1}{\sqrt{\var{P_{i,j}}}} - \frac{1}{\sqrt{\smallvar{\proj_{i,j}}}}} > \frac{\epsilon}{4\sqrt{\smallvar{\proj_{i,j}}}}} \label{eq:PQ-aux-3}
		\, .
	\end{align}
	According to Corollary~\ref{cor:projection-close-to-density}, Eq.~\eqref{eq:PQ-aux-2} is upper bounded by
	\[
	\frac{16\smallvar{\proj_{i,j}}}{\kernelsize^2\epsilon^2} \leq \frac{16\var{P_{i,j}}}{\kernelsize^2\epsilon^2} \leq \frac{16\cdot 107\cdot \sigma^4}{\kernelsize^2\epsilon^2} \leq \frac{69\sigma^2}{\epsilon^2}
	\, ,
	\]
	where we used, successively, the fact that projection reduces variance, Lemma~\ref{lemma:Pij-variance-upper-bound}, and $\sigma^2\leq \kernelsize^2/25$. 
	Finally, we rewrite Eq.~\eqref{eq:PQ-aux-3} as
	\begin{align*}
		\proba{ \abs{P_{i,j}-\expec{P_{i,j}}} > \frac{\epsilon}{4\abs{ \sqrt{\frac{ \smallvar{\proj_{i,j}}}{\var{P_{i,j}}}} - 1}} } &\leq \frac{16\var{P_{i,j}}}{\epsilon^2}\cdot \abs{ \sqrt{\frac{ \smallvar{\proj_{i,j}}}{\var{P_{i,j}}}} - 1}^2 \tag{Chebyshev} \\
		&\leq \frac{16\cdot 107\cdot \sigma^4}{\epsilon^2} \cdot \frac{1}{4\kernelsize^4} \tag{Lemma~\ref{lemma:Pij-variance-upper-bound} and Eq.~\eqref{eq:aux-projection-close-1}} \\
		&\leq \frac{16\cdot 107 \cdot \sigma^2}{4\epsilon^2\cdot 25\cdot 25} \tag{$\sigma^2\leq \kernelsize^2/25$ and $\kernelsize\geq 5$} \\
		&\leq \frac{\sigma^2}{\epsilon^2}
		\, .
	\end{align*}
	We conclude by summing all bounds. 
\end{proof}

\begin{remark}
	Here, trying to obtain a uniform bound, that is, a meaningful upper bound for $\proba{\infnorm{P-Q}>\epsilon}$ seems challenging. 
	Indeed, since the typical deviations of $P$ are of order $\sigma$, we need $\epsilon$ to be of order $\sigma$ for Theorem~\ref{th:P-close-to-Q} to be useful. 
	If we were to extend the bound given by Theorem~\ref{th:P-close-to-Q} by a union bound argument, a factor $\height\cdot \width$ would appear. 
\end{remark}

%%%%%%%%%%%%%%%%%%%%%%%%%%%%%%%%%%%%%%%%%%%%%%%%%%%%%%%%%%%%%%%%%%%%%%%%%%%%%%%%%

\section{HOMOGENEOUS PATCHES}
\label{sec:homogeneous-patches}

In this section, we provide a detailed proof of Theorem~\ref{th:average-local-max} of the paper. 
We start by providing a proof of Lemma~\ref{lemma:connected-components} of the paper in Section~\ref{sec:connected-components}. 
We follow up by some area computations in Section~\ref{sec:area-computations}. 
In Section~\ref{sec:binding-lemma}, we show that the number of lattice points inside a square, a rounded square, or a disk, is close to the area of the geometric form. 
We conclude the proof in Section~\ref{sec:conclusion}.

%%%%%%%%%%%%%%%%%%%%%%%%%%%%%%%%%%%%%%%%%%%%%%%%%%%%%%%%%%%%%%%%%%%%%%%%%%%%%%%%%

\subsection{Proof of Lemma~\ref{lemma:connected-components}}
\label{sec:connected-components}

The main goal of this section is to prove the following (Lemma~\ref{lemma:connected-components} in the paper).

\begin{lemma}[Connected components and local maxima]
%	\label{lemma:connected-components}
	Let $\graph(A)$ be the directed graph produced by Algorithm~\ref{algo:simplified-quickshift} applied to an array $A\in\Reals^{\height\times\width}$. 
	Then to each connected component of $\graph(A)$ corresponds a \emph{unique} local maxima of $A$ in the sense of Definition~\ref{def:local-maximum}. 
\end{lemma}

\begin{proof}
	Let $\cc$ be a connected component of $\graph(A)$. 
	We split the proof in existence and uniqueness of the local maximum.
	
	\paragraph{Existence.}
	Let us pick any point in $\cc$. 
	If this point has no outgoing edge, then it is a local maximum by definition. 
	Otherwise, we follow outgoing edges until we meet a vertex that has no outgoing edge. 
	This always happens since there is a finite number of points in $\cc$ and loops are prohibited by construction of the outgoing edges.
	
	\paragraph{Uniqueness.}
	Let us now suppose that there are two distinct local maxima in $\cc$, say $(i_1,j_1)$ and $(i_2,j_2)$. 
	Since $\cc$ is a connected component of $\graph(A)$, there is a path between $(i_1,j_1)$ and $(i_2,j_2)$. 
	This path cannot have directed edges flowing out of $(i_1,j_1)$ (resp. $(i_2,j_2)$): by construction of $\graph(A)$, it would mean that there exists a vertex $(i',j')$ inside $E_{i_1,j_1}$ (resp. $E_{i_2,j_2}$) with higher value. 
	Thus the path connecting $(i_1,j_1)$ and $(i_2,j_2)$ has directed edges going towards them. 
	Let us call $(i_1',j_1')$ and $(i_2',j_2')$ the vertices at the origin of these edges. 
	Of course, one has $(i_1',j_1')\neq (i_2',j_2')$, since a given vertex has only one outgoing edge by construction of the graph. 
	Therefore, we can repeat the reasoning above and find two distinct vertices with edges flowing towards $(i_1',j_1')$ and $(i_2',j_2')$. 
	This is absurd, since there is only a finite number of vertices in $\cc$. 
	Therefore the local maximum is unique. 
\end{proof} 

%%%%%%%%%%%%%%%%%%%%%%%%%%%%%%%%%%%%%%%%%%%%%%%%%%%%%%%%%%%%%%%%%%%%%%%%%%%%%%%%%

\subsection{Area computations}
\label{sec:area-computations}

Computing the area of a square of half side $\kernelwidth$ or of a disk of radius $\maxdist$ is straightforward: we obtain $4\kernelwidth^2$ and $\pi\maxdist^2$, respectively. 
The only difficulty is computing the area of a rounded square, the case of interest when $\kernelwidth\leq \maxdist\leq \sqrt{2}\kernelwidth$, which we will assume in this section. 
We begin by the computation of a \emph{circle segment} (see Figure~\ref{fig:area-circle-segment}).

\begin{figure}
	\begin{center}
		\includegraphics[scale=0.3]{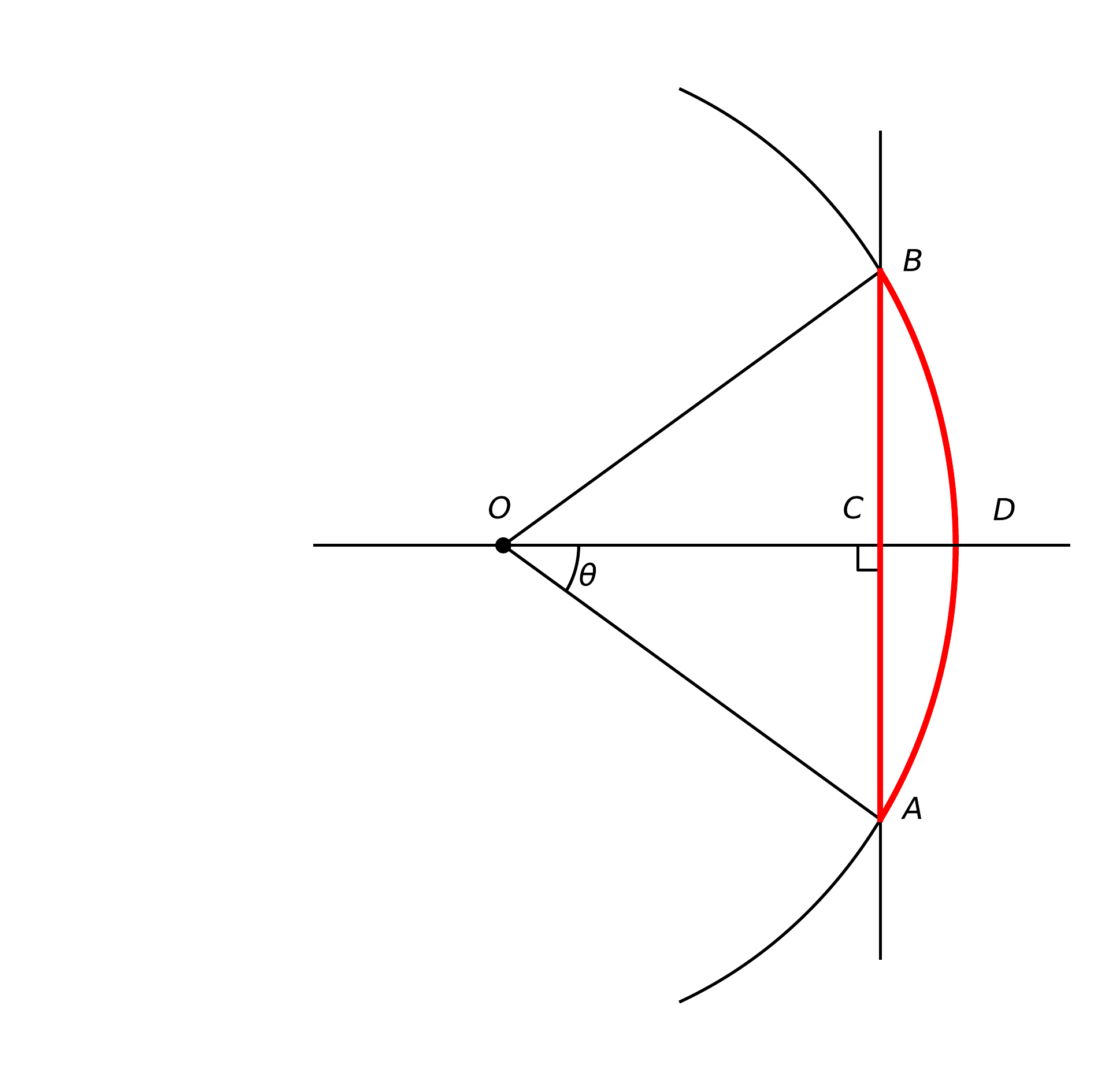}
	\end{center}
	\caption{\label{fig:area-circle-segment}
		A circle segment with parameters $(s,d)$ is outlined in red, obtained by intersecting of the disk of radius $OD=d$ centered in $O$ and the square of half-side $OC=s$. The area of the portion of the plane enclosed by the red curve is given by Lemma~\ref{lemma:circle-segment-area}. 
	}
\end{figure}

\begin{lemma}[Area of a circle segment]
	\label{lemma:circle-segment-area}
	Let $\maxdist$ and $\kernelwidth$ be two real numbers such that $\kernelwidth < \maxdist < \sqrt{2}\kernelwidth$. 
	Then the area of the circle segment  is given by $\gamma(\kernelwidth,\maxdist)$, with
	\[
	\forall s,d,\qquad \gamma(s,d) \defeq d^2\arctan \frac{\sqrt{d^2-s^2}}{s} - s\sqrt{d^2-s^2} 
	\, .
	\]
\end{lemma}

Before proving Lemma~\ref{lemma:circle-segment-area}, let us note that a direct consequence is that we can deduce that the area of the rounded square. 
Let us define
\begin{equation}
	\label{eq:def-rounded-square-area}
	B(s,d) \defeq \pi d^2 -4\gamma(s,d)= \pi d^2 - 4d^2\arctan \frac{\sqrt{d^2-s^2}}{s} + 4s\sqrt{d^2-s^2}
	\, ,
\end{equation}
that is, the area of a disk of radius $d$ to which we subtract four times the area of the circle segment of parameters $(s,d)$. 
Then the area of the rounded square of parameters $(\kernelwidth,\maxdist)$ is given by $B(\kernelwidth,\maxdist)$. 
Notice that we recover the limit cases: when $\maxdist=\kernelsize$, $B(\kernelwidth,\maxdist) = \pi\maxdist^2$, and when $\maxdist=\sqrt{2}\kernelsize$, $B(\kernelwidth,\maxdist)=4\kernelsize^2$, as expected. 

\begin{proof}
	The area of the circle segment is given by the difference between the area of the angular sector $OAB$ and the area of the triangle $OAB$. 
	Let us recall that $OC=s$ and $OA=d$. 
	According to Pythagoras theorem, $AC=\sqrt{d^2-s^2}$, and therefore $\tan \theta = \frac{\sqrt{d^2-s^2}}{s}$. 
	We deduce that the area of the angular sector is given by $d^2\arctan \frac{\sqrt{d^2-s^2}}{s}$. 
	Finally, the area of the triangle is $AB \times OC / 2$, that is, $s\sqrt{d^2-s^2}$. 
\end{proof}

%%%%%%%%%%%%%%%%%%%%%%%%%%%%%%%%%%%%%%%%%%%%%%%%%%%%%%%%%%%%%%%%%%%%%%%%%%%%%%%%%

\subsection{Counting lattice points}
\label{sec:binding-lemma}

The main goal of this section is to prove a binding lemma: in order to count the number of lattice points inside on of the three geometrical shapes, it is sufficient to compute the area of these figures, if one is ready to loose $\bigo{\kernelwidth+\maxdist}$ terms. 
The idea of the proof is rather simple: show that the area expressions are Lipschitz with respect to $\kernelwidth$ and $\maxdist$, and then use the idea of Gauss historical bound for the Gauss circle problem. 
As in Section~\ref{sec:area-computations}, the only challenging case is that of the rounded square. 
We thus focus our attention on this case for now on, and start by studying $\gamma$ more precisely. 

\begin{figure}
	\begin{center}
		\includegraphics[scale=0.3]{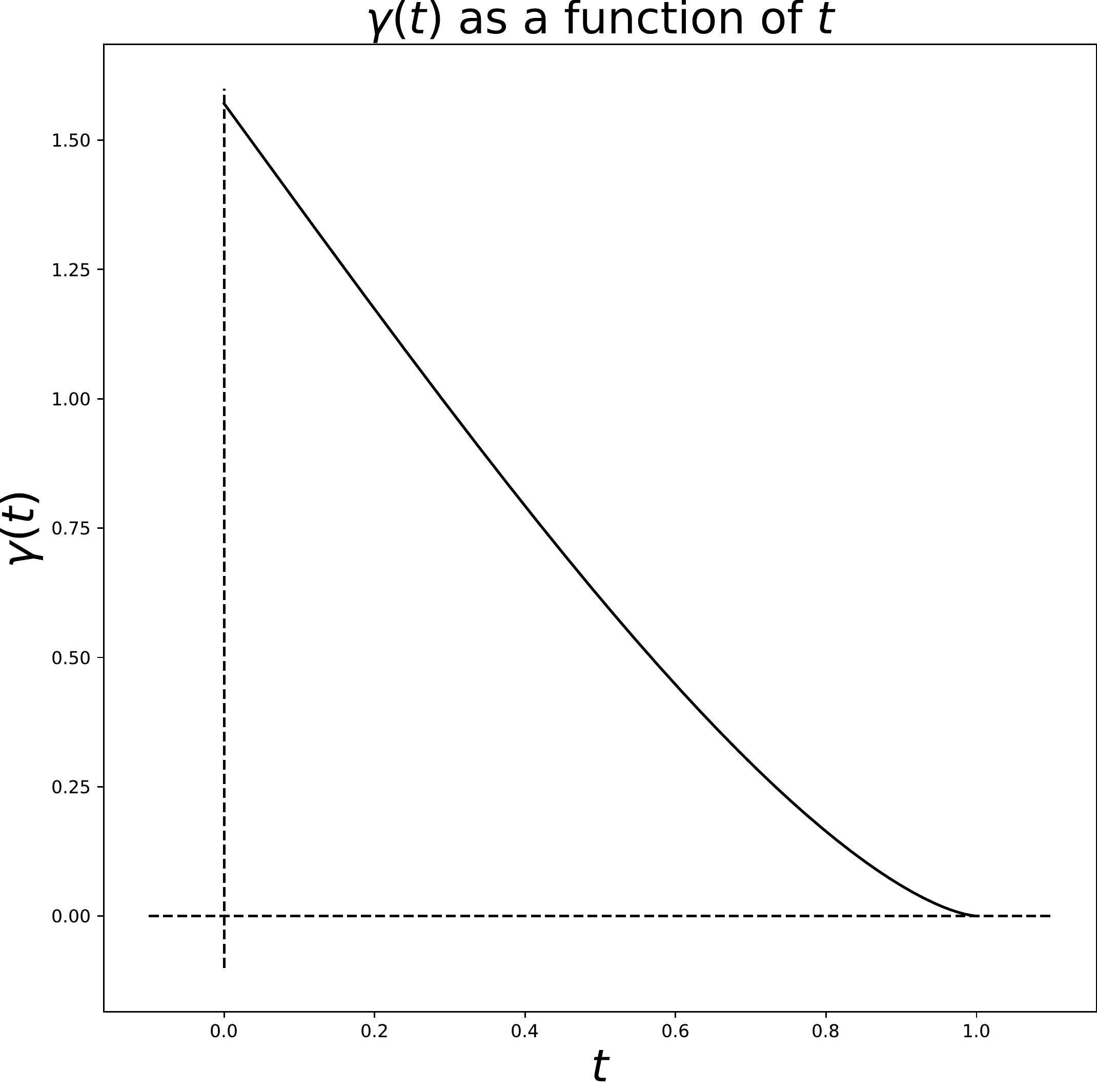}
	\end{center}
	\caption{\label{fig:gamma-function}
		Plot of the $\gamma$ function, the normalized version of the $\gamma$ function derived in Lemma~\ref{lemma:circle-segment-area}.
	}
\end{figure}

\begin{lemma}[$\gamma$ is Lipschitz in each coordinate]
	\label{lemma:gamma-lipschitz}
	Let $s$ and $d$ be fixed numbers such that $s<d<\sqrt{2}s$. 
	Then the function $\gamma(\cdot,d)$ is $d$-Lipschitz on $[0,d]$ and the function $\gamma(s,\cdot)$ is $3s$-Lipschitz on $[s,\sqrt{2}s]$. 
\end{lemma}

\begin{proof}
	Let us first consider the case where $d$ is fixed. 
	We can write
	\begin{align*}
		\gamma(ds,d) &= d^2 \arctan \frac{\sqrt{d^2-d^2s^2}}{ds} - ds \cdot \sqrt{d^2-d^2s^2} \\
		&= d^2 \left[\arctan \frac{\sqrt{1-s^2}}{s} - s\sqrt{1-s^2}\right] \\
		\gamma(ds,d) &= d^2 \gamma(s)
		\, ,
	\end{align*}
	where we let $\gamma(\cdot)$ denote $\gamma(\cdot,1)$. 
	As seen on Figure~\ref{fig:gamma-function}, $\gamma(\cdot)$ is well-behaved. 
	More precisely, one can check that $\gamma'(t)=-2t\sqrt{1-t^2}$. 
	The maximum of $\abs{\gamma'(t)}$ on $[0,1]$ is $1$, attained at $t=\sqrt{2}/2$. 
	As a consequence, $\gamma$ is $1$-Lipschitz. 
	For any $h$ small enough, we write
	\begin{align*}
		\abs{\gamma(s,d) - \gamma(s+h,d)} &= d^2 \abs{\gamma\left(\frac{s+h}{d}\right)-\gamma\left(\frac{s}{d}\right)} \\
		&\leq d^2 \abs{\frac{s+h}{d} - \frac{s}{d}} = dh
		\, ,
	\end{align*}
	where we used the $1$-Lipschitzness of $\gamma$ in the inequality. 
	We deduce that $\gamma(\cdot,d)$ is $d$-Lipschitz. 
	
	Next let us consider that $s$ is fixed. 
	We directly compute the partial derivative of $\gamma$ with respect to $d$ and we obtain $2d\arctan \frac{\sqrt{d^2-s^2}}{s}$. 
	It is an increasing function of $d$, thus taking its maximum at the rightmost possible value for $d$, which is $d=\sqrt{2}s$. 
	The extreme value is $2\sqrt{2}s\arctan 1 = \frac{\pi\sqrt{2}}{2}s (\approx 2.22 s)$, and we can deduce the result. 
\end{proof}

We can now state and prove the main result of this section:

\begin{proposition}[Counting lattice points]
	\label{prop:counting-lattice-points}
	Let $s$ and $d$ positive real numbers. 
	Then 
	\begin{itemize}
		\item if $d<s$, the number of lattice points inside a disk of radius $d$ is given by $\pi d^2 + \bigo{d}$;
		\item if $s<d<\sqrt{2}s$, the number of lattice points inside a rounded square of parameters $(s,d)$ is given by $B(s,d) + \bigo{s}$;
		\item if $\sqrt{2}s < d$, the number of lattice points inside a square of half side $s$ is given by $4s^2+\bigo{s}$.
	\end{itemize}
\end{proposition}

\begin{proof}
	We begin by the first case, which is known as the \emph{Gauss circle problem} in the literature. 
	Let $n(d)$ denote the number of lattice points inside the disk of radius $d$. 
	For each lattice point inside the disk, we can draw a square of side $1$ centered at the point. 
	These squares are non-overlapping, and there total area is less than that of a disk of radius $d+\sqrt{2}/2$: the limit case is that of a lattice point lying exactly on the boundary of the disk. 
	In the same fashion, the total area cannot be less than that of a disk of radius $d-\sqrt{2}/2$ (if $d\leq \sqrt{2}/2$, then there is no need to consider this case). 
	Since the total area coincide with the number of lattice points inside the disk, we have obtained the following bound:
	\[
	\pi(d-\sqrt{2}/2)^2 \leq n(d) \leq \pi (d+\sqrt{2}/2)^2
	\, .
	\]
	From this last display, we immediately deduce that $\abs{n(d)-\pi d^2} = \bigo{d}$. 
	This line of proof is actually the historical one, proposed by Gauss himself \citep{hardy1999ramanujan}. 
	We directly extend this reasoning to the square by considering an inner square of side $2s-1$ and an outer square of side $2s+1$. 
	
	The rounded square case is slightly more involved: essentially, one has to look at an inner rounded square of parameters $(s-1/2,d-\sqrt{2}/2)$ and an outer rounded square of parameters $(s+1/2,d+\sqrt{2}/2)$. 
	Controlling the error amounts to bounding $B(s,d) - B(s-1/2,d-\sqrt{2}/2)$ (the outer case is similar). 
	This is where Lemma~\ref{lemma:gamma-lipschitz} comes into play: we write
	\begin{align*}
		\abs{B(s,d) - B(s-1/2,d-\sqrt{2}/2)} & = \abs{B(s,d) - B(s-1/2,d) + B(s-1/2,d) - B(s-1/2,d-\sqrt{2}/2)} \\
		&\leq \abs{B(s,d) - B(s-1/2,d)} + \abs{B(s-1/2,d) - B(s-1/2,d-\sqrt{2}/2)} \\
		&\leq \frac{d}{2} + 4s^2 - (2s-1)^2 + \frac{3s\sqrt{2}}{2} = \bigo{s}
		\, ,
	\end{align*}
	where we used the (bi-)Lipschitzness of $\gamma$ in the second inequality. 
\end{proof}

\begin{remark}
	Let us call $E(d)$ the difference between $n(d)$, the number of lattice points inside a disk of radius $d$, and $\pi d^2$, the area of that disk. 
	Following the classical argument, we showed that $\abs{E(d)}\leq 2\sqrt{2}\pi d$. 
	This is not the best bound, which is currently $\abs{E(d)} =\bigo{d^{\frac{131}{208}}}$ \citep{huxley2000rational}. 
	Using this bound would improve the error made in Theorem~\ref{th:average-local-max} by considering $\pi \maxdist^2$ instead of $N_{i,j}$, but additional work is required to extend the argument to the rounded square. 
\end{remark}

%%%%%%%%%%%%%%%%%%%%%%%%%%%%%%%%%%%%%%%%%%%%%%%%%%%%%%%%%%%%%%%%%%%%%%%%%%%%%%%%%

\subsection{Average number of local maxima}
\label{sec:conclusion}

We are now able to prove Theorem~\ref{th:average-local-max} of the paper. 
Let us recall the statement of Theorem~\ref{th:average-local-max}: 

\begin{theorem}[Average number of local maxima]
	Assume that~\ref{ass:flat-image} holds. 
	Let $R\subseteq\image$ be a rectangle of height $h$ and width $w$ at distance greater than $2\kernelwidth$ from the border. 
	Then 
	\[
	\expec{N_R(Q)} = 
	\begin{cases}
		hw\cdot\left(\frac{1}{\pi \maxdist^2}+\bigo{\frac{1}{\maxdist}}\right) \text{ if } \maxdist \leq \kernelwidth \\
		hw\cdot \left(\frac{1}{\pi(3\kernelwidth\maxdist-\kernelwidth^2-\maxdist^2)}+\bigo{\frac{1}{\kernelwidth}}\right) \text{ if } \kernelwidth < \maxdist \leq \sqrt{2}\kernelwidth \\
		hw\cdot \left(\frac{1}{4\kernelwidth^2}+\bigo{\frac{1}{\kernelwidth}}\right) \text{ otherwise. }
	\end{cases}
	\]
\end{theorem}

\begin{proof}
	Let us first focus on the disk case, that is, $\maxdist \leq \kernelwidth$. 
	According to Lemma~\ref{lemma:key-lemma} of the paper, 
	\begin{equation}
		\label{eq:key-th-average}
		\expec{N_R(Q)} = \sum_{(i,j)\in R} \frac{1}{N_{i,j}}
		\, .
	\end{equation}
	Since $R$ is at distance greater than $2\kernelwidth > \maxdist$ from the border of the image, $E_{i,j}$ is not intersecting the boundaries of the image, and in particular $N_{i,j}$ is constant. 
	Moreover, according to Proposition~\ref{prop:counting-lattice-points}, we know that $N_{i,j}=\pi\maxdist^2 + \bigo{\maxdist}$. 
	We deduce the result since there are $hw$ terms in the sum in Eq.~\eqref{eq:key-th-average}, and since
	\[
	\frac{1}{N_{i,j}} = \frac{1}{\pi\maxdist^2 + \bigo{\maxdist}} = \frac{1}{\pi\maxdist^2(1+\bigo{1/d})} = \frac{1}{\pi\maxdist^2}\left( 1 + \bigo{\frac{1}{\maxdist}}\right)
	\, .
	\]
	The square case is similar. 
	Finally, in the rounded square case, the only difference is that we replaced $B(\kernelwidth,\maxdist)$ by the more readable $\pi(3\kernelwidth\maxdist-\kernelwidth^2-\maxdist^2)$. 
	This is justified by Lemma~\ref{lemma:Bsd-technical}. 
	Indeed, 
	\[
	\abs{\frac{1}{B(\kernelwidth,\maxdist)} - \frac{1}{\pi(3\kernelwidth\maxdist-\kernelwidth^2-\maxdist^2)}} = \frac{\abs{B(\kernelwidth,\maxdist) - \pi(3\kernelwidth\maxdist-\kernelwidth^2-\maxdist^2)}}{B(\kernelwidth,\maxdist)\cdot \pi(3\kernelwidth\maxdist-\kernelwidth^2-\maxdist^2)} \leq \frac{0.04\maxdist^2}{\pi \left(\kernelwidth^2\right)^2 }
	\, ,
	\]
	since a disk of radius $\kernelwidth$ is always included in the intersection in this configuration. 
	Since $\kernelwidth\geq \maxdist/\sqrt{2}$, we deduce that 
	\[
	\abs{\frac{1}{B(\kernelwidth,\maxdist)} - \frac{1}{\pi(3\kernelwidth\maxdist-\kernelwidth^2-\maxdist^2)}} = \bigo{\frac{1}{\maxdist^3}}
	\, ,
	\]
	and the leading term in the approximation comes from replacing $N_{i,j}$ by $B$. 
\end{proof}

\begin{remark}
	The main reason for substituting $\pi(3\kernelwidth\maxdist-\kernelwidth^2-\maxdist^2)$ to $B(\kernelwidth,\maxdist)$ is the clarity of the exposition. 
	Both expressions are $2$-homogeneous in $(\kernelwidth,\maxdist)$ and we could have stated Theorem~3.2 of the paper with $B$ at the denominator in the second case. 
\end{remark}

%%%%%%%%%%%%%%%%%%%%%%%%%%%%%%%%%%%%%%%%%%%%%%%%%%%%%%%%%%%%%%%%%%%%%%%%%%%%%%%%%

\section{SHARP BOUNDARIES}
\label{sec:sharp-boundaries}

In this section, we prove Theorem~\ref{th:decreasing-density-estimates-bicolor} of the paper, which is true in the bicolor setting (Assumption~\ref{ass:bicolor}). 
The organization of this section follows the sketch of the proof: 
in Section~\ref{sec:expec-bicolor} , we compute the expected value of $P_{i,j}$ under Assumption~\ref{ass:bicolor}, and the difference $P_{i,j+1}-P_{i,j}$. 
In Section~\ref{sec:variance-bicolor}, we study the variance of $P_{i,j}$ and show that the same bound holds. 
We conclude in Section~\ref{sec:everything-bicolor}. 

%%%%%%%%%%%%%%%%%%%%%%%%%%%%%%%%%%%%%%%%%%%%%%%%%%%%%%%%%%%%%%%%%%%%%%%%%%%%%%%%%

\subsection{Expectation computation}
\label{sec:expec-bicolor}

The expected value of $P_{i,j}$ under Assumption~\ref{ass:bicolor} differs notably from the one computed under Assumption~\ref{ass:flat-image} (in Section~\ref{sec:elementary} of this Appendix). 

\begin{lemma}[Expected density, bi-color setting]
	\label{lemma:expected-density-bicolor}
	Assume that~\ref{ass:bicolor} holds. 
	Then, for all $(i,j)\in \leftim$, we have
	\[
	\expec{P_{i,j}} = \normcst_2 \cdot\left[ \sum_{(u,v)\in C_{i,j}\cap \leftim } \delta_{u,v} + \exp{\frac{-\norm{\oc_1-\oc_2}^2}{2(\kernelsize^2+2\sigma^2)}}\cdot\sum_{(u,v)\in C_{i,j}\cap \rightim } \delta_{u,v}
	\right] 
	\, .
	\]
\end{lemma}

\begin{proof}
	Straightforward from Lemma~\ref{lemma:Xuv-moments-generic}. 
\end{proof}

We now show that the expected density decreases near the boundary, provided that the color change is large enough. 

\begin{lemma}[Expected density decreases near the boundary]
	\label{lemma:expected-density-decreases}
	Assume that~\ref{ass:bicolor} holds. 
	Assume further that $\norm{\oc_1-\oc_2}\geq 3\kernelsize$ and $\kernelsize\geq 5$. 
	Then, for any $(i,j)\in \leftim$, 
	\[
	\expec{P_{i,j}} - \expec{P_{i,j+1}} \geq \frac{3\kernelsize}{2}
	\, .
	\]
\end{lemma}

\begin{proof}
	We start by noticing that another way to read Lemma~\ref{lemma:expected-density-bicolor} is
	\[
	\expec{P_{i,j}} = \normcst_2 \cdot \left[\Delta_{i,j} -\left(1-\exp{\frac{-\norm{\oc_1-\oc_2}^2}{2(\kernelsize^2+2\sigma^2)} }\right) \cdot \sum_{(u,v)\in C_{i,j}\cap \rightim } \delta_{u,v}\right]
	\, .
	\]
	Thus the difference that interests us here can be written
	\[
	\expec{P_{i,j}} - \expec{P_{i,j+1}} = \normcst_2 \cdot \left(1-\exps{\frac{-\norm{\oc_1-\oc_2}^2}{2(\kernelsize^2+2\sigma^2)} }\right) \cdot \left[\sum_{(u,v)\in C_{i,j+1}\cap \rightim } \delta_{u,v}^{i,j+1} -\sum_{(u,v)\in C_{i,j}\cap \rightim } \delta_{u,v}^{i,j} \right]
	\, .
	\]
	Since $\delta_{u,v}^{i,j+1}=\delta_{u,v-1}^{i,j}$, most of the terms cancel out in the right-hand side of the last display, and we are left with
	\begin{equation}
		\label{eq:expected-density-decrease-main}
		\expec{P_{i,j}} - \expec{P_{i,j+1}} = \normcst_2 \cdot \left(1-\exps{\frac{-\norm{\oc_1-\oc_2}^2}{2(\kernelsize^2+2\sigma^2)} }\right) \cdot \sum_{u=i-\kernelwidth}^{i+\kernelwidth}\delta_{u,j_0}^{i,j}
		\, .
	\end{equation}
	We now proceed to lower bound each term in Eq.~\eqref{eq:expected-density-decrease-main}. 
	First, since we assumed $\sigma^2\leq \kernelsize^2/25$, it is easy to check that $\normcst_2 \geq 4/5$. 
	Next, using Eq.~\eqref{eq:exp-sum-lower-bound} (we assumed that $\kernelsize\geq 5$), we see that 
	\[
	\sum_{u=i-\kernelwidth}^{i+\kernelwidth}\delta_{u,v_0}^{i,j} \geq 2\cdot \kernelsize + 1 \geq 2\kernelsize
	\, ,
	\]
	since the worst case is at the border, where $j=j_0$. 
	The last term requires a bit more attention. 
	We start by writing
	\begin{align*}
		\frac{\norm{\oc_1-\oc_2}^2}{2(\kernelsize^2+2\sigma^2)} &\geq \frac{9\kernelsize^2}{2(\kernelsize^2+2\sigma^2)} \tag{since $\norm{\oc_1-\oc_2} \geq 3\kernelsize$} \\
		&\geq \frac{9\cdot 25}{2\cdot 27} \tag{since $\sigma^2\leq \kernelsize^2/25$}
		\, .
	\end{align*}
	We deduce that
	\[
	1-\exps{\frac{-\norm{\oc}^2}{2(\kernelsize^2+2\sigma^2)} } \geq 1-\exps{\frac{-9\cdot 25}{2\cdot 27}} \geq \frac{49}{50}
	\, .
	\]
	Multiplying together the individual numerical bounds, we find the promised result. 
\end{proof}

%%%%%%%%%%%%%%%%%%%%%%%%%%%%%%%%%%%%%%%%%%%%%%%%%%%%%%%%%%%%%%%%%%%%%%%%%%%%%%%%%%%%%

\subsection{Variance computation}
\label{sec:variance-bicolor}

We now prove that the variance of $P_{i,j}$ in the bicolor case is upper bounded by the variance in the homogeneous case. 

\begin{lemma}[Variance of $P_{i,j}$, bicolor setting]
	\label{lemma:variance-bound-bicolor}
	Assume that~\ref{ass:bicolor} holds. 
	Assume further that $\kernelsize\geq 5$.
	Then 
	\[
	\var{P_{i,j}} \leq 107\sigma^4
	\, .
	\]
\end{lemma} 

\begin{remark}
	There is no point in trying to derive a specific bound since the worst case is far from the border when $P_{i,j}$ coincides with its homogeneous version. 
	In that event, both variance coincides. 
	Lemma~\ref{lemma:variance-bound-bicolor} simply states that this is the worst case scenario. 
\end{remark}

\begin{proof}
	When computing the variance of the density estimate, all terms are identical to the homogeneous case when looking at the individual variances of $X_{u,v}$ for $(u,v)\in C_{i,j}\cap\leftim$.  
	The covariance terms are also the same if both points lie in the same part of $C_{i,j}$. 
	
	% first difficult case
	Let us compute $\var{X_{u,v}}$ when $(u,v)\in C_{i,j}\cap\rightim$.
	Using Lemma~\ref{lemma:Xuv-moments-generic}, we find that
	\begin{equation}
		\label{eq:aux-variance-bicolor-1}
		\expec{X_{u,v}} = \left(\frac{\kernelsize^2}{\kernelsize^2+2\sigma^2}\right)^{\frac{3}{2}}\cdot\exp{\frac{-\norm{\oc_1-\oc_2}^2}{2(\kernelsize^2+2\sigma^2)}}\cdot \delta_{u,v}
		\, .
	\end{equation}
	and
	\[
	\expec{X_{u,v}^2} = \left(\frac{\kernelsize^2}{\kernelsize^2+4\sigma^2}\right)^{\frac{3}{2}}\cdot\exp{\frac{-\norm{\oc_1-\oc_2}^2}{\kernelsize^2+4\sigma^2}}\cdot \delta_{u,v}^2
	\, .
	\]
	Putting the last two displays together, we see that
	\[
	\var{X_{u,v}} = \left[\left(\frac{\kernelsize^2}{\kernelsize^2+4\sigma^2}\right)^{\frac{3}{2}}\cdot\exps{\frac{-\norm{\oc_1-\oc}^2}{\kernelsize^2+4\sigma^2}} - \left(\frac{\kernelsize^2}{\kernelsize^2+2\sigma^2}\right)^3\cdot\exps{\frac{-\norm{\oc_1-\oc_2}^2}{\kernelsize^2+2\sigma^2}}\right]\cdot \delta_{u,v}^2
	\, .
	\]
	According to Lemma~\ref{lemma:technical-variance-bound}, the left term is upper bounded by $\psi_1(\sigma^2/\kernelsize^2)$, provided that $\norm{\oc}\geq 3\kernelsize$ and $\sigma^2\leq \kernelsize^2/25$, which we assumed. 
	
	% second difficult case
	Finally, let us look at covariance terms with point $(u,v)\in C_{i,j}\cap \leftim$ and $(u',v')\in C_{i,j}\cap\rightim$. 
	In that case, we first write
	\begin{align*}
		\condexpec{X_{u,v}X_{u',v'}}{\xi_{i,j}} &= \left(\frac{\kernelsize^2}{\kernelsize^2+\sigma^2}\right)^3 \cdot \exp{\frac{-\norm{\xi_{i,j}-\oc_1}^2}{2(\kernelsize^2+\sigma^2)}}
		\cdot \exp{ \frac{-\norm{\xi_{i,j}-\oc_2}^2}{2(\kernelsize^2+\sigma^2)}}
		\, ,
	\end{align*}
	by independence. 
	The key computation here is
	\[
	\int \exp{\frac{-(x-\oc_{1,k})^2}{2(\kernelsize^2+\sigma^2)} + \frac{-(x-\oc_{2,k})^2}{2(\kernelsize^2+\sigma^2)} + \frac{-(x-\oc_{1,k})^2}{2\sigma^2}}\frac{\Diff x}{\sigma\sqrt{2\pi}} = \left(\frac{\kernelsize^2+\sigma^2}{\kernelsize^2+3\sigma^2}\right)^{\frac{1}{2}} \cdot \exp{\frac{-(\oc_{1,k}-\oc_{2,k})^2}{\kernelsize^2+3\sigma^2}}
	\, ,
	\]
	for $k\in\{1,2,3\}$, since $(i,j)\in\leftim$. 
	In definitive, we obtain that
	\[
	\expec{X_{u,v}X_{u',v'}} = \frac{\kernelsize^6}{(\kernelsize^2+\sigma^2)^{\frac{3}{2}}(\kernelsize^2+3\sigma^2)^{\frac{3}{2}}} \cdot \exp{\frac{-\norm{\oc_1-\oc_2}^2}{\kernelsize^2+3\sigma^2}} \cdot \delta_{u,v}\delta_{u',v'}
	\, .
	\]
	Now, the expectation of $X_{u,v}$ is unchanged:
	\[
	\expec{X_{u,v}} = \left(\frac{\kernelsize^2}{\kernelsize^2+2\sigma^2}\right)^{\frac{3}{2}}\cdot \delta_{u,v}
	\, .
	\]
	The expectation of $X_{u',v'}$ is as in Eq.~\eqref{eq:aux-variance-bicolor-1}. 
	Putting all of this together, we have
	\[
	\cov{X_{u,v}}{X_{u',v'}} =  \left[\frac{\kernelsize^6}{(\kernelsize^2+\sigma^2)^{\frac{3}{2}}(\kernelsize^2+3\sigma^2)^{\frac{3}{2}}}\cdot \exps{\frac{-\norm{\oc_1-\oc_2}^2}{\kernelsize^2+3\sigma^2}} 
	-\left(\frac{\kernelsize^2}{\kernelsize^2+2\sigma^2}\right)^3\cdot\exps{\frac{-\norm{\oc_1-\oc_2}^2}{2(\kernelsize^2+2\sigma^2)}} \right]\cdot \delta_{u,v}\delta_{u',v'}
	\, .
	\]
	In Lemma~\ref{lemma:technical-covariance-bound}, we show that the bracketed term is smaller than $\psi_2\left(\frac{\sigma^2}{\kernelsize^2} \right)$. 
	Thus $\var{P_{i,j}}$ is smaller than in the unicolor case, and the bound is thus identical to that of Lemma~\ref{lemma:Pij-variance-upper-bound}. 
\end{proof}

%%%%%%%%%%%%%%%%%%%%%%%%%%%%%%%%%%%%%%%%%%%%%%%%%%%%%%%%%%%%%%%%%%%%%%%%%%%%%%%%%%

\subsection{Putting everything together}
\label{sec:everything-bicolor}

We are now able to prove Theorem~\ref{th:decreasing-density-estimates-bicolor} of the paper, which we re-state here:

\begin{theorem}[Decreasing density estimates]
	Assume that~\ref{ass:bicolor} holds.  
	Assume further that $\kernelsize \geq 5$ and that $\norm{\oc_1-\oc_2}\geq 3\kernelsize$. 
	Then, for any $(i,j)\in\midim\cap \leftim$ such that $\abs{j-j_0}\leq \kernelwidth$, 
	\[
	\proba{P_{i,j} > P_{i,j+1}} \geq 1- 16\sigma^2
	\, .
	\]
\end{theorem}

\begin{proof}
	We first use the Chebyshev's inequality and Lemma~\ref{lemma:variance-bound-bicolor} to show that
	\[
	\proba{\abs{P_{i,j}-\expec{P_{i,j}}} \geq \frac{3\kernelsize}{4}} \leq \frac{\var{P_{i,j}}}{3^2\kernelsize^2/4^2} \leq \frac{107\sigma^4\cdot 4^2}{3^2\kernelsize^2} 
	\, .
	\]
	Since $\sigma^2\leq \kernelsize^2/25$,
	\[
	\proba{\abs{P_{i,j}-\expec{P_{i,j}}} \geq \frac{107\cdot 16\kernelsize}{9\cdot 25}} \leq 8\sigma^2
	\, .
	\]
	We have the same bound for $P_{i,j+1}$. 
	Noting that $\frac{3\kernelsize}{4}$ is exactly half the minimal gap between both expectations given by Lemma~\ref{lemma:expected-density-decreases}, we conclude by a union bound argument. 
\end{proof}

%%%%%%%%%%%%%%%%%%%%%%%%%%%%%%%%%%%%%%%%%%%%%%%%%%%%%%%%%%%%%%%%%%%%%%%%%%%%%%%%

\section{TECHNICAL RESULTS}
\label{sec:technical-results}

We collect in this section technical lemmatas used throughout the proofs. 

%%%%%%%%%%%%%%%%%%%%%%%%%%%%%%%%%%%%%%%%%%%%%%%%%%%%%%%%%%%%%%%%%%%%%%%%%%%%%%%%

\subsection{Expectation computations}
\label{sec:gaussian-computations}

In this section we collect technical facts related to Gaussian computations. 

\begin{lemma}[Key Gaussian computation]
	\label{lemma:key-gaussian-computation}
	Let $a,b$ be real numbers, and $c,d$ be positive numbers. 
	Then it holds that
	\[
	\int \exp{\frac{-(x-a)^2}{2c^2}+\frac{-(x-b)^2}{2d^2}} \frac{\Diff x}{d\sqrt{2\pi}} =  \left(\frac{c^2}{c^2+d^2}\right)^{\frac{1}{2}} \cdot  \exp{\frac{-(a-b)^2}{2(c^2+d^2)} }
	\, .
	\]
\end{lemma}

\begin{proof}
See for instance Lemma~11.1 in \citet{pmlr-v108-garreau20a}. 
\end{proof}

%%%%%%%%%%%%%%%%%%%%%%%%%%%%%%%%%%%%%%%%%%%%%%%%%%%%%%%%%%%%%%%%%%%%%%%%%%

\subsection{Deterministic part}

In this section, we collect some technical facts about $\delta_{u,v}$ and $\Delta_{i,j}$. 

\begin{lemma}[Bounding $\Delta_{i,j}$]
	\label{lemma:Delta-bounds}
	Assume that $\kernelsize\geq 5$.
	Then, for all $(i,j)\in\midim(\kernelwidth)$, 
	\[
	(2\kernelsize +1)^2\leq \Delta_{i,j} \leq \frac{(5\kernelsize+2)^2}{4}
	\, .
	\]
\end{lemma}

\begin{proof}
	By definition of $\Delta_{i,j}$, since $(i,j)\in\midim(\kernelwidth)$, we have
	\begin{equation}
		\label{eq:aux-bound-Delta}
		\Delta_{i,j} = \left( 2\sum_{u=1}^{\kernelwidth} \exp{\frac{-u^2}{2\kernelsize^2}} +1 \right)^2
		\, .
	\end{equation}
	The key idea of the proof is a series-integral comparison of the sum appearing in the previous display. 
	Since the mapping $u\mapsto\exps{-u^2/(2\kernelsize^2)}$ is decreasing on $\Reals_+$, we write
	\[
	\forall u \in \{1,\ldots,\kernelwidth\}, \qquad \int_{u}^{u+1} \exps{\frac{-t^2}{2\kernelsize^2}}\Diff t \leq \exp{\frac{-u^2}{2\kernelsize^2}} \leq \int_{u-1}^{u} \exps{\frac{-t^2}{2\kernelsize^2}}\Diff t  
	\, .
	\]
	We then sum these inequalities over $u$. 
	On one side, 
	\begin{align*}
		\sum_{u=1}^{\kernelwidth} \exp{\frac{-u^2}{2\kernelsize^2}} &\geq \int_1^{\kernelwidth+1} \exps{\frac{-t^2}{2\kernelsize^2}}\Diff t \\
		&= \kernelsize\cdot \int_{1/\kernelsize}^{\frac{\kernelwidth+1}{\kernelsize}} \exps{\frac{-s^2}{2}}\Diff s \tag{$s=t/\kernelsize$}\\
		&\geq \kernelsize \cdot \int_{1/5}^3 \exps{\frac{-s^2}{2}}\Diff s \tag{$\kernelwidth=3\kernelsize$ and $\kernelsize\geq 5$}
		\, .
	\end{align*}
	
	Evaluating numerically the integral in last display, we find
	\begin{equation}
		\label{eq:exp-sum-lower-bound}
		\sum_{u=1}^{\kernelwidth} \exp{\frac{-u^2}{2\kernelsize^2}} \geq 1.0512 \cdot \kernelsize\geq  \kernelsize
		\, .
	\end{equation}
	Coming back to Eq.~\eqref{eq:aux-bound-Delta}, we get the promised lower bound. 
	
	In the other direction, we find that
	\begin{align*}
		\sum_{u=1}^{\kernelwidth}\exps{\frac{-u^2}{2\kernelsize^2}} &\leq \int_0^{\kernelwidth}\exps{\frac{-u^2}{2\kernelsize^2}}\Diff u \\
		&= \kernelsize \cdot \int_0^3 \exps{\frac{-u^2}{2}}\Diff u \approx 1.24 \cdot \kernelsize \\
		&\leq \frac{5\kernelsize}{2\cdot 2}
		\, .
	\end{align*}
	Coming back to Eq.~\eqref{eq:aux-bound-Delta}, we get the promised upper bound. 
\end{proof}

\begin{lemma}[Sum of squares is negligible]
	\label{lemma:sum-of-squares-negligible}
	Assume that $\kernelsize\geq 5$. 
	Then, for any $(i,j)\in \midim(\kernelwidth)$, it holds that
	\[
	\sum_{(u,v)\in C_{i,j}} \delta_{u,v}^2 \leq \frac{\Delta_{i,j}^2}{4\kernelsize^2} 
	\, .
	\]
\end{lemma}

\begin{proof}
	Similarly to the proof of Lemma\ref{lemma:Delta-bounds}, we first write that 
	\[
	\sum_{(u,v)\in C_{i,j}} \delta_{u,v}^2 = \left(2\cdot \sum_{u=1}^{\kernelwidth}\exps{\frac{-u^2}{\kernelsize^2}} +1\right)^2
	\, ,
	\]
	since we consider that $(i,j)\in\midim(\kernelwidth)$. 
	Since the mapping $u\mapsto \exps{-u^2/\kernelsize^2}$ is decreasing on $\Reals_+$, we have
	\begin{align*}
		\sum_{u=1}^{\kernelwidth}\exps{\frac{-u^2}{\kernelsize^2}} &\leq \int_0^{\kernelwidth} \exps{\frac{-u^2}{\kernelsize^2}}\Diff u \\
		&= \kernelsize \cdot \int_0^{3} \exps{-t^2}\Diff t
		\, ,
	\end{align*}
	where we used $\kernelwidth=3\kernelsize$. 
	Numerically, we find this integral to be smaller than $9/10$, and we deduce that
	\begin{equation}
		\label{eq:sum-of-squares-upper-bound}
		\sum_{(u,v)\in C_{i,j}} \delta_{u,v}^2 \leq \frac{(9\kernelsize+5)^2}{25}
		\, .
	\end{equation}
	
	% now other direction
	Recall that, since we assumed $(i,j)\in\midim(\kernelwidth)$ and $\kernelsize\geq 5$, according to Lemma~\ref{lemma:Delta-bounds},
	\[
	\Delta_{i,j}^2 \geq (2\kernelsize + 1)^4
	\, .
	\]
	We notice that, for any $\kernelsize\geq 5$, 
	\[
	\frac{\kernelsize^2(18\kernelsize+10)^2}{100(2\kernelsize+1)^4} \leq \frac{81}{400}
	\, ,
	\]
	and we deduce the result. 
\end{proof}

%%%%%%%%%%%%%%%%%%%%%%%%%%%%%%%%%%%%%%%%%%%%%%%%%%%%%%%%%%%%%%%%%%%%%%%%%%%%%%%

\subsection{Facts about the $\psi$ functions}
\label{sec:psi-functions-facts}

In this section, we collect facts about $\psi_1$ and $\psi_2$ that are used throughout the proofs. 

\begin{lemma}[Bounds on $\psi_1$]
	\label{lemma:psi-1-bounds}
	For any $t\geq 0$, $\psi_1(t)\geq 0$. 
	Moreover,
	\[
	\forall t\in\left[0,\frac{1}{25}\right],\qquad 4t^2 \leq \psi_1(t) \leq 6t^2
	\, .
	\]
\end{lemma}

\begin{proof}
	For any $t\geq 0$, we have
	\[
	(1+2t)^6-(1+4t)^3 = 12t^2+96t^3+240t^4+192t^5+64t^6 \geq 0
	\, ,
	\]
	which proves the first statement. 
	The second statement follows trough a function study. 
\end{proof}

\begin{lemma}[Bounds on $\psi_2$]
	\label{lemma:psi-2-bounds}
	For any $t\geq 0$, $\psi_2(t)\geq 0$. 
	Moreover,
	\[
	\forall t\in\left[0,\frac{1}{25}\right],\qquad t^2\leq \psi_2(t) \leq \frac{3t^2}{2}
	\, .
	\]
\end{lemma}

\begin{proof}
	For any $t\geq 0$, we have
	\[
	(1+2t)^6 - (1+t)^3(1+3t)^3 = 3t^2+24t^3+69t^4+84t^5+37t^6 \geq 0
	\, ,
	\]
	which proves the first statement. 
	The second statement follows trough a function study. 
\end{proof}

\begin{lemma}[A technical bound]
	\label{lemma:psi-technical}
	For any $t\in [0,1/25]$,
	\[
	\abs{2\psi_1(t) - \psi_2(t)} \leq 3t^2
	\, .
	\]
\end{lemma}

\begin{proof}
	Straightforward function study.
\end{proof}

%%%%%%%%%%%%%%%%%%%%%%%%%%%%%%%%%%%%%%%%%%%%%%%%%%%%%%%%%%%%%%%%%%%%%%%%%%%%%%%

\subsection{Variance bounds}

We detail here the variance bounds used in Section~\ref{sec:variance-bicolor}. 

\begin{lemma}[Upper bounding the variance in the bicolor case]
	\label{lemma:technical-variance-bound}
	Assume that $\sigma^2\leq \kernelsize^2/25$ and that $\norm{\oc}\geq 3\kernelsize$. 
	Then 
	\begin{equation}
		\label{eq:technical-variance-bound}
		\left(\frac{\kernelsize^2}{\kernelsize^2+4\sigma^2} \right)^{\frac{3}{2}} \exps{\frac{-\norm{\oc}^2}{\kernelsize^2+4\sigma^2}} - \left(\frac{\kernelsize^2}{\kernelsize^2+2\sigma^2} \right)^3 \exps{\frac{-\norm{\oc}^2}{\kernelsize^2+2\sigma^2}} \leq \psi_1\left(\frac{\sigma^2}{\kernelsize^2}\right)
		\, .
	\end{equation}
\end{lemma}

\begin{proof}
	Let us first study the limit case: $\norm{\oc}=3\kernelsize$. 
	In that case, we see Eq.~\eqref{eq:technical-variance-bound} as a difference of functions of $s=\sigma^2/\kernelsize^2$.
	Studying this difference on $[0,1/25]$, we see that it is negative on this interval, and thus Eq.~\eqref{eq:technical-variance-bound} is true when $\norm{\oc}=3\kernelsize$. 
	Next, we see the left term of Eq.~\eqref{eq:technical-variance-bound} as a function of $t=\norm{\oc}^2$. 
	We are going to prove that this function is decreasing. 
	Indeed, the derivative at $t$ is given by
	\[
	\Phi(t)\defeq \frac{-\kernelsize^3}{(\kernelsize^2+4\sigma^2)^{\frac{5}{2}}} \exps{\frac{-t}{\kernelsize^2+4\sigma^2}} + \frac{\kernelsize^6}{(\kernelsize^2+2\sigma^2)^4} \exps{\frac{-t}{\kernelsize^2+2\sigma^2}}
	\, .
	\]
	Asking for $\Phi(t)$ to be negative is equivalent to asking for
	\[
	\exp{t\left(\frac{-1}{\kernelsize^2+2\sigma^2} +\frac{1}{\kernelsize^2+4\sigma^2}\right)} < \frac{(\kernelsize^2+2\sigma^2)^4}{\kernelsize^3(\kernelsize^2+4\sigma^2)^{\frac{5}{2}}}
	\, .
	\]
	Since $\kernelsize^2+2\sigma^2 < \kernelsize^2+4\sigma^2$, the left-hand side of the last display is a decreasing function of $t$, and we just have to consider the limit case $t=9\kernelsize^2$. 
	Setting again $s=\sigma^2/\kernelsize^2$, one can check that 
	\[
	\exp{9\left(\frac{-1}{1+2s}+\frac{1}{1+4s}\right)} < \frac{(1+2s)^4}{(1+4s)^{5/2}}
	\]
	for any $s\in [0,1/25]$. 
	This concludes the proof. 
\end{proof}

\begin{lemma}[Upper bounding the covariance in the bicolor case]
	\label{lemma:technical-covariance-bound}
	Assume that $\sigma^2\leq \kernelsize^2/25$ and that $\norm{\oc}\geq 3\kernelsize$. 
	Then
	\[
	\frac{\kernelsize^6}{(\kernelsize^2+\sigma^2)^{\frac{3}{2}}(\kernelsize^2+3\sigma^2)^{\frac{3}{2}}}\cdot \exps{\frac{-\norm{\oc}^2}{\kernelsize^2+3\sigma^2}} 
	-\left(\frac{\kernelsize^2}{\kernelsize^2+2\sigma^2}\right)^3\cdot\exps{\frac{-\norm{\oc}^2}{2(\kernelsize^2+2\sigma^2)}}  \leq \psi_2\left(\frac{\sigma^2}{\kernelsize^2} \right)
	\, .
	\]
\end{lemma}

\begin{proof}
	This time the proof is much simpler. 
	Indeed, since $\kernelsize^2+3\sigma^2 < 2(\kernelsize^2+2\sigma^2)$, one finds that 
	\[
	\exps{\frac{-\norm{\oc}^2}{\kernelsize^2+3\sigma^2}} < \exps{\frac{-\norm{\oc}^2}{2(\kernelsize^2+2\sigma^2)}}
	\, ,
	\]
	and we deduce the result since this last display is less than $1$.
\end{proof}

%%%%%%%%%%%%%%%%%%%%%%%%%%%%%%%%%%%%%%%%%%%%%%%%%%%%%%%%%%%%%%%%%%%%%%%%%%%%%%%%%%

\subsection{Approximate expression for the rounded square area}

The exact expression of $B$ can be cumbersome to use. 
We sometimes prefer to use the following approximation, which gives close enough values for all practical purposes: 

\begin{lemma}[Approximate expression for $B$]
	\label{lemma:Bsd-technical}
	Assume that $s<d<\sqrt{2}s$. 
	Then 
	\[
	\abs{B(s,d) - \pi(3sd-s^2-d^2)} \leq 0.04d^2
	\, .
	\]
\end{lemma}

\begin{proof}
	Recall that, for any $s,d$ such that $s<d<\sqrt{2}s$, 
	\[
	\gamma(s,d) = d^2 \gamma(t), \quad \text{with}\quad \gamma(t) = \arctan \frac{\sqrt{1-t^2}}{t} - t\sqrt{1-t^2}
	\, .
	\]
	Numerically, we find that 
	\[
	\abs{\gamma(t) - \frac{\pi}{2}\left(1-\frac{3t}{2} +\frac{t^2}{2}\right)} \leq 0.04
	\, .
	\]
	We deduce that 
	\[
	\abs{\gamma(s,d) - \frac{\pi}{2}\left(d^2-\frac{3sd}{2} +\frac{s^2}{2}\right)}\leq 0.04d^2
	\, .
	\]
	The result follows by definition of $B$. 
\end{proof}

%%%%%%%%%%%%%%%%%%%%%%%%%%%%%%%%%%%%%%%%%%%%%%%%%%%%%%%%%%%%%%%%%%%%%%%%%%%%%%%%%%%

\section{ADDITIONAL EXPERIMENTS}
\label{sec:additional-experiments}

In this section, we present additional experimental results. 
We start with further verification of the validity of Theorem~3.2 of the paper in Section~\ref{sec:average-further-check}. 
In Section~\ref{sec:real-image}, we provide additional insights on the distribution of pixel values for real images. 
Finally, we showcase more qualitative results for the scaling use-case in Section~\ref{sec:scaling-qualitative}. 

%%%%%%%%%%%%%%%%%%%%%%%%%%%%%%%%%%%%%%%%%%%%%%%%%%%%%%%%%%%%%%%%%%%%%%%%%%%%%%%%%%%

\subsection{Checking Theorem~3.2}
\label{sec:average-further-check}

In this section, we provide some additional experimental checks of Theorem~3.2 on synthetic data, namely with lower and higher $\kernelsize$. 
They are summarized in Figure~\ref{fig:additional-evolution}. 
As in the paper, we generated ten random images with increasing shape and counted the number of superpixels in the central square area. 
The paper presents results for $\kernelsize=5$ (which is the default choice), which we reproduce here for easier comparison (middle row of Figure~\ref{fig:additional-evolution}).

\begin{figure}[ht]
	\begin{center}
		\includegraphics[scale=0.28]{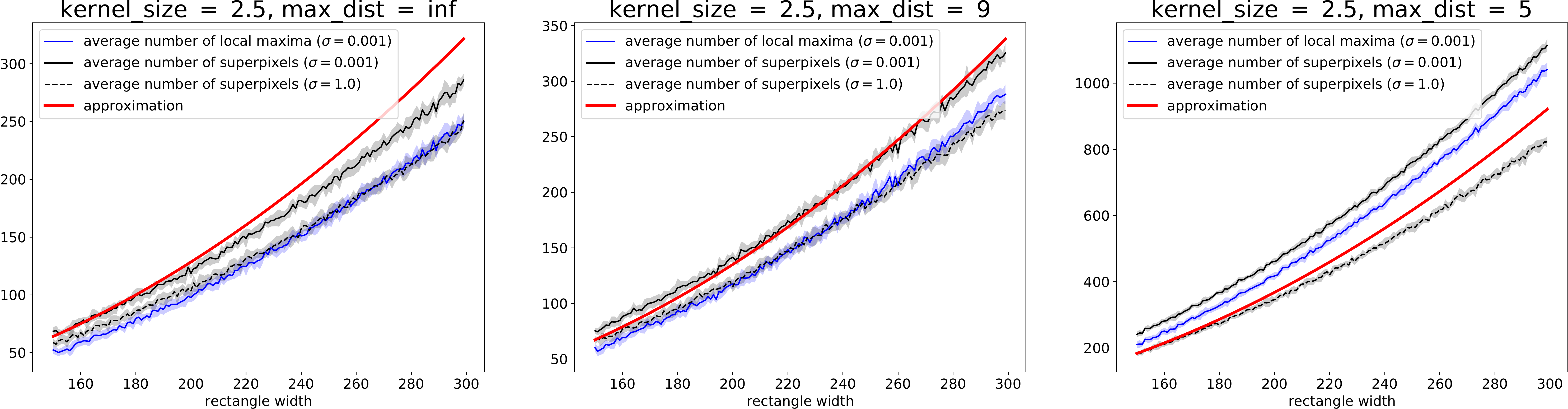} \\
		\includegraphics[scale=0.28]{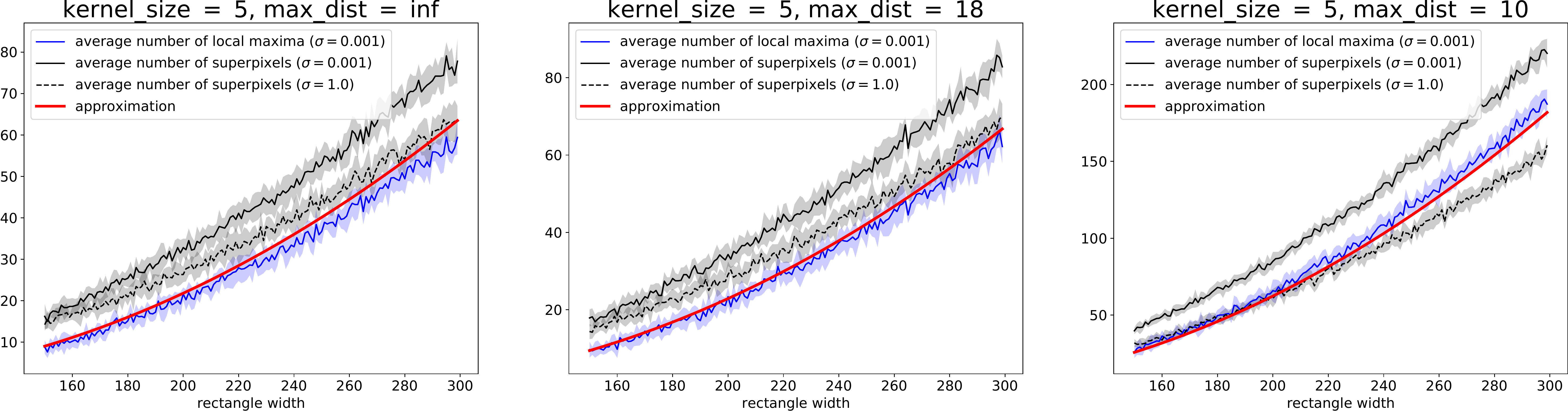} \\
		\includegraphics[scale=0.28]{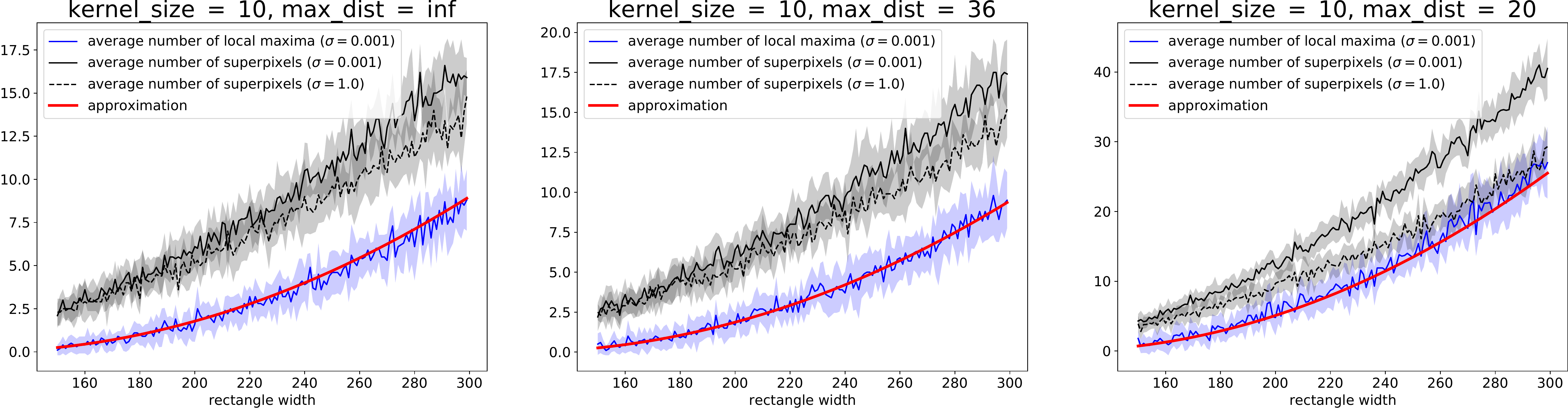}
	\end{center}
	\vspace{-0.1in}
	\caption{\label{fig:additional-evolution}For small values of $\kernelsize$, the averaging effect is not strong enough and the fit is not perfect between Theorem~3.2 and empirical runs of quickshift. However, we still recover the right scale for the number of superpixels. The fit improves as $\kernelsize$ increases. Notice the difference in the scale of the $y$-axes.}
\end{figure}
%

%%%%%%%%%%%%%%%%%%%%%%%%%%%%%%%%%%%%%%%%%%%%%%%%%%%%%%%%%%%%%%%%%%%%%%%%%%%%%%%%%%%

\subsection{Flat portions of real images}
\label{sec:real-image}

In this section, we showcase the limits of Assumption~\ref{ass:flat-image} when dealing with real images, expanding the discussion at the beginning of Section~5 in the paper. 
In Figure~\ref{fig:limits}, we display the histogram of pixel values for an image of our dataset. 
More specifically, we took the same image as in Figure~1 of the paper, which we converted to CIELAB. 
Subsequently, we selected the pixels contained in the red rectangle and reported for each channel the histogram of the values. 
The main differences with our model are:
\begin{itemize}
	\item \textbf{higher variance:} though rather small, the variance on each channel is higher than the values with which Theorem~3.2 of the paper is concerned. Moreover, it is not constant throughout all channels. 
	\item \textbf{distribution is not Gaussian:} the pixel values distribution is not Gaussian. A Kolmogorov-Smirnov test rejects $H_0:$ ``$L$ is normally distributed'' at any level. Moreover, the $b$ channel is clearly multimodal (hinting that there are actually several colors in the rectangle).
	\item \textbf{spacial dependencies:} often, there is a source of light somewhere in the image drawing the $L$ values. In our example, the Pearson correlation between the first line of $L$ values and their indices is $-0.68$, far from~$0$, the theoretical value under Assumption~\ref{ass:flat-image}.
\end{itemize}
This is a consistent behavior on all the flat regions of images that we tested. 

\begin{figure}[ht]
	\begin{center}
		\includegraphics[scale=0.45]{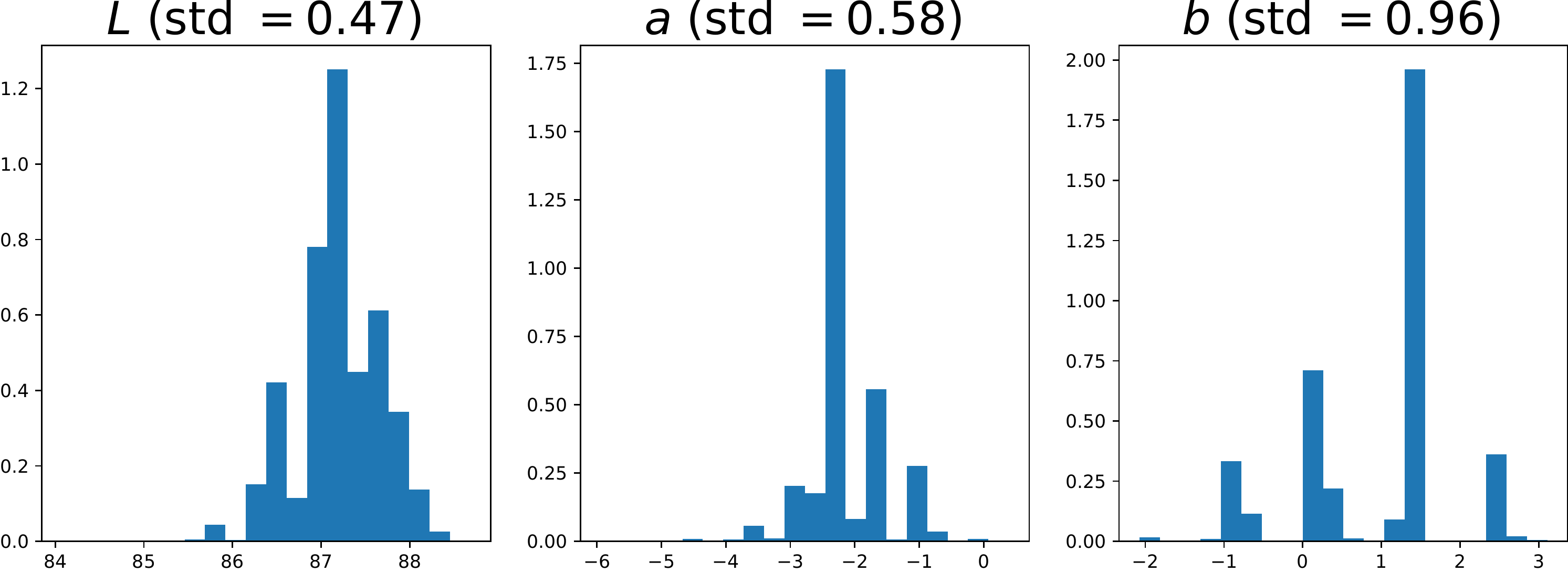}	
	\end{center}
	\vspace{-0.1in}
	\caption{\label{fig:limits}Histogram of CIELAB pixel values in the rectangle outlined in Figure~1 of the paper. Under Assumption~\ref{ass:flat-image}, we should observe three bell curves with same width.}
\end{figure}

%%%%%%%%%%%%%%%%%%%%%%%%%%%%%%%%%%%%%%%%%%%%%%%%%%%%%%%%%%%%%%%%%%%%%%%%%%%%%%%%%%%

\subsection{Scaling hyperparemeters}
\label{sec:scaling-qualitative}

In this section, we present more qualitative results for the rescaling use-case described in Section~5.3 of the paper. 
All experiments presented here are done with the ILSVRC dataset. 
We took $\kernelsize=5$ and a downsizing ratio equals to $2$, that is, the downsized image has size $\height/2\times \width/2$, rounded to the closest integer. 
In Figure~\ref{fig:use-case-1} we take $\maxdist=+\infty$, in Figure~\ref{fig:use-case-2} $\maxdist=18$, and $\maxdist=10$ in Figure~\ref{fig:use-case-3}. 
For each image, we produce superpixels with the parameters indicated in the title. The \textbf{left image is the downsized version}, while the \textbf{middle and right image have the original size}. 
We denote by~$N$ the number of superpixels in the image. 
The heuristic proposed in the paper amounts to multiply both~$\kernelsize$ and~$\maxdist$ by the downsizing ratio (here, $2$), in order to get \emph{approximately} the same number of superpixels in downsized image as in the original image. 

\begin{figure}[h]
	\begin{center}
		\includegraphics[scale=0.35]{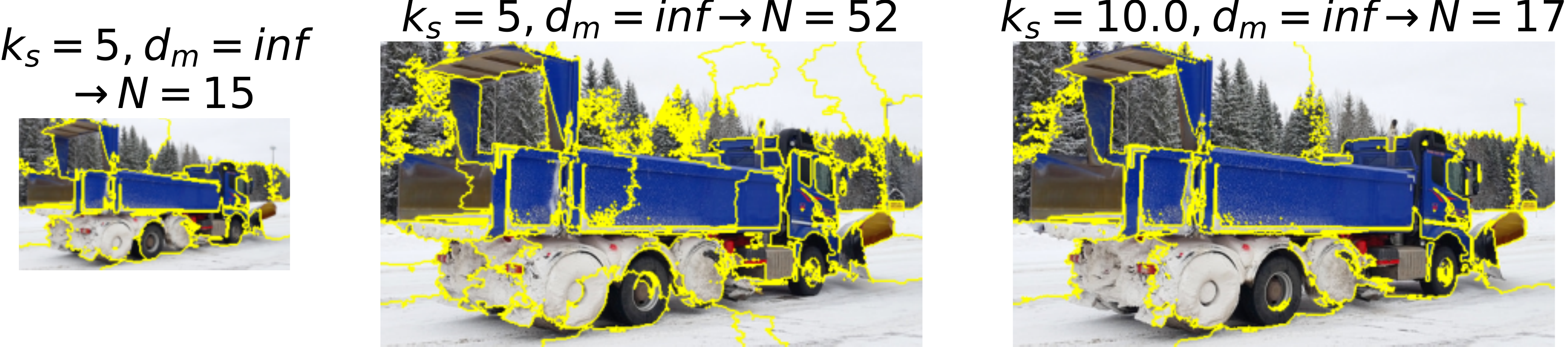} \\
		\vspace{0.2cm}
		\includegraphics[scale=0.35]{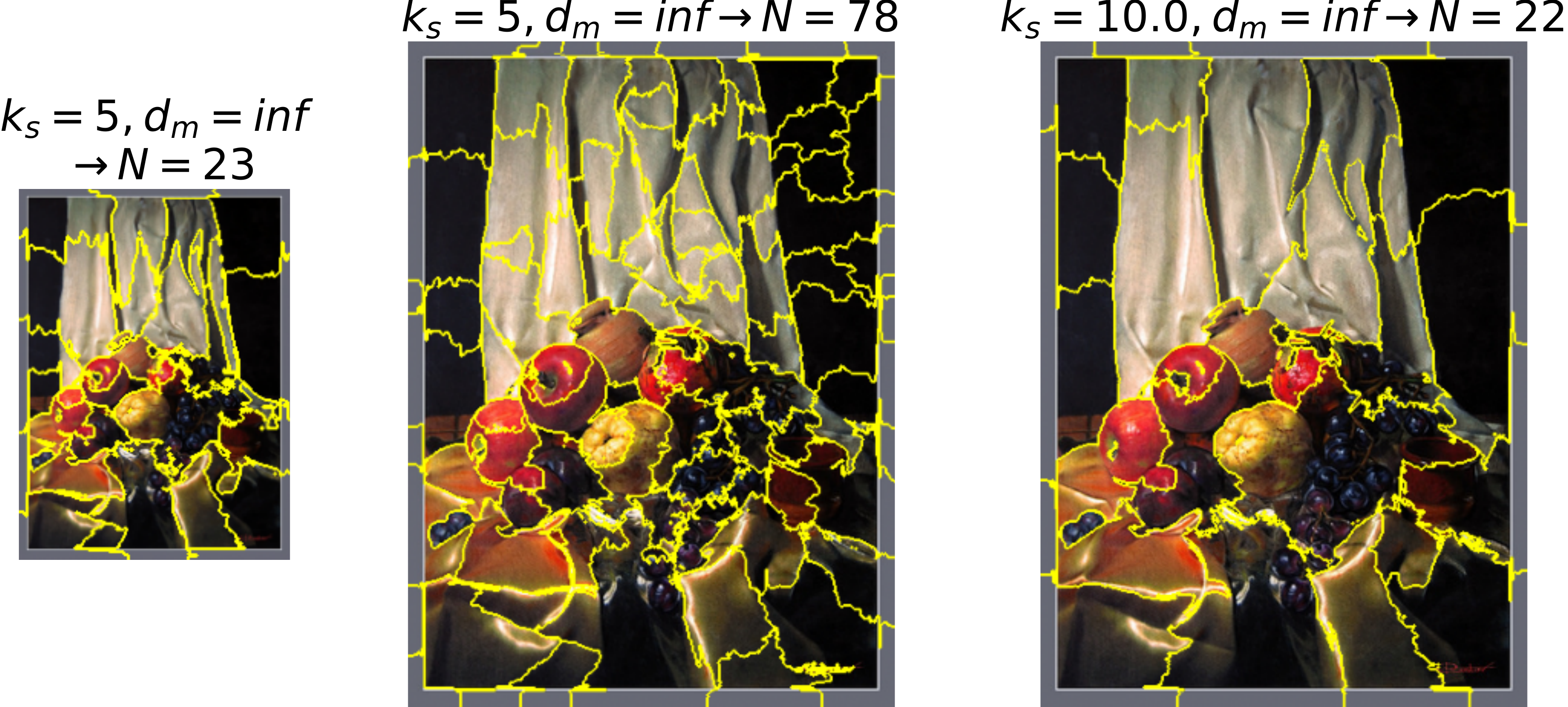} \\
		\vspace{0.2cm}
		\includegraphics[scale=0.35]{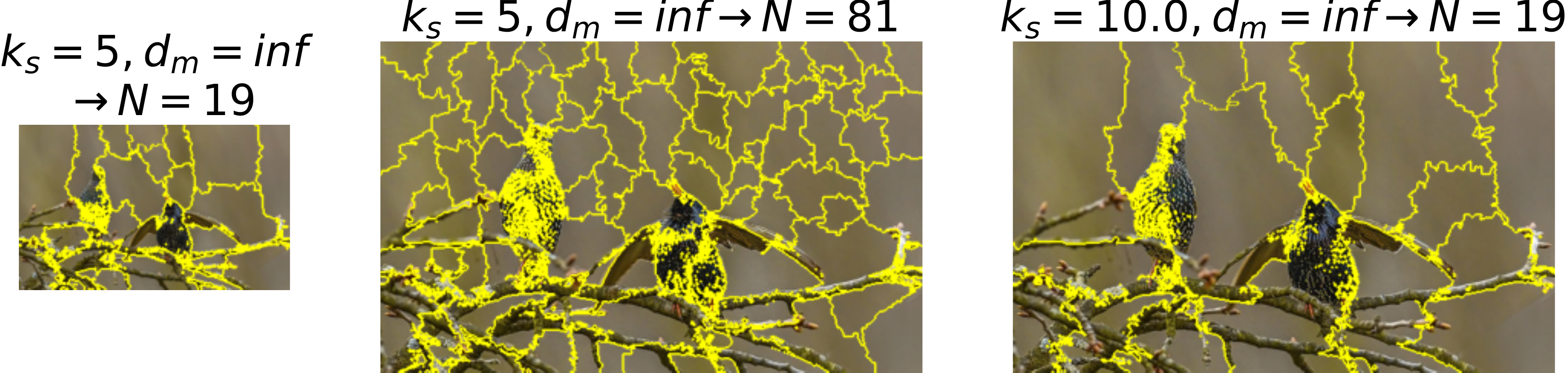} \\
		\vspace{0.2cm}
		\includegraphics[scale=0.35]{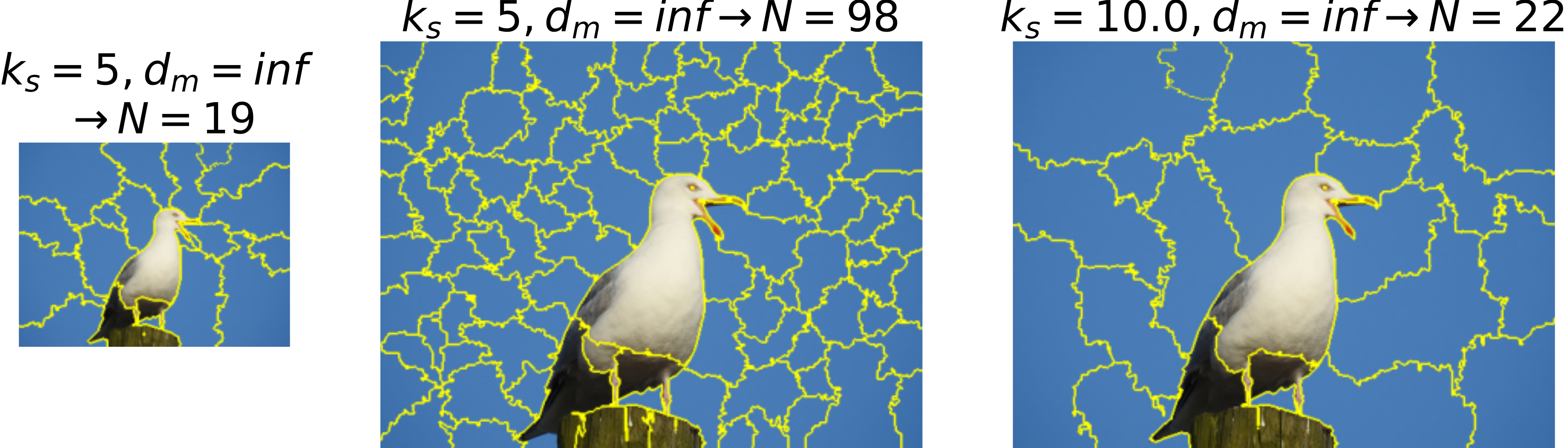} \\
		\vspace{0.2cm}
		\includegraphics[scale=0.35]{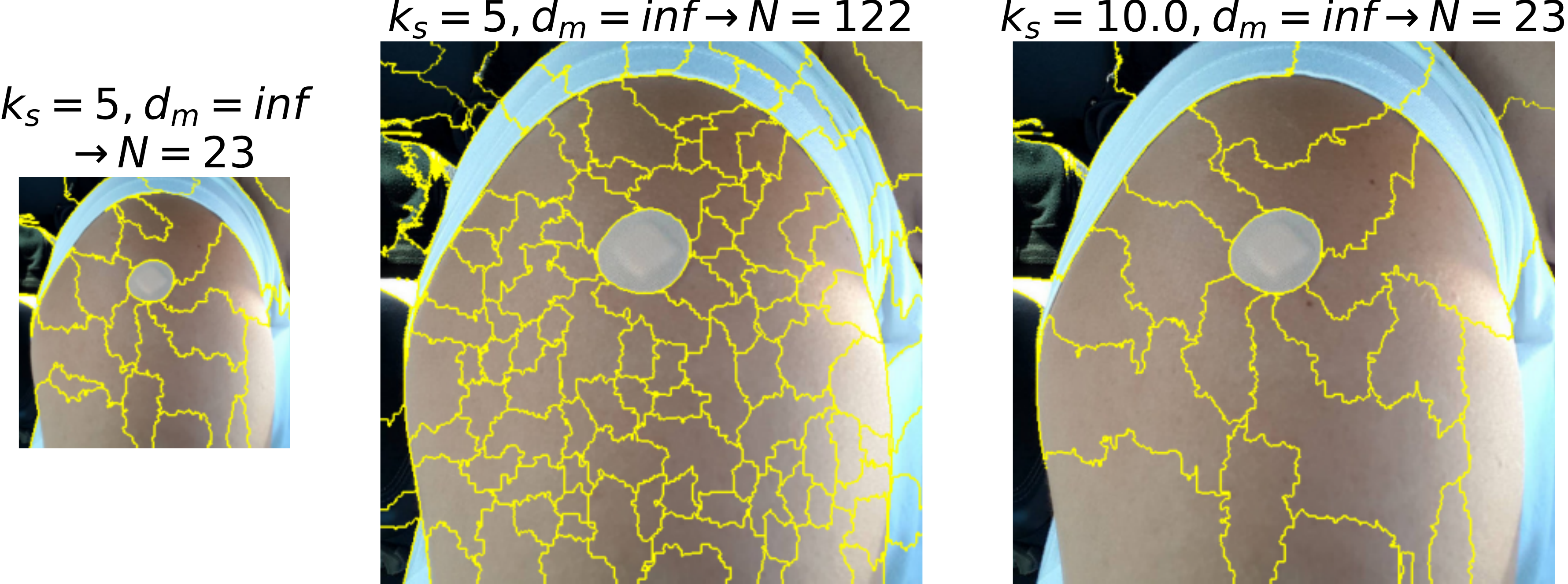} \\
		\vspace{0.2cm}
		\includegraphics[scale=0.35]{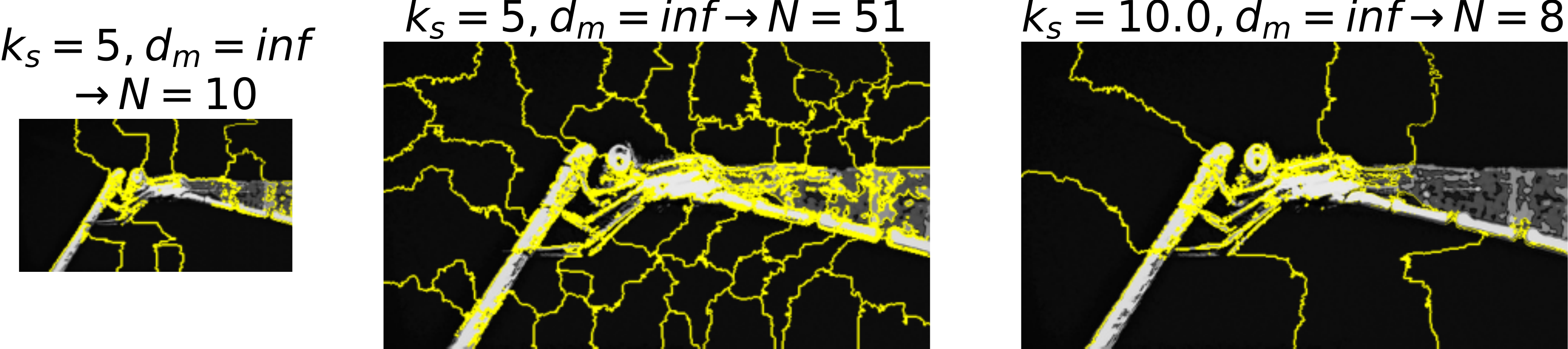}
	\end{center}
	\vspace{-0.1in}
	\caption{\label{fig:use-case-1}Additional results for the rescaling experiments. Original hyperparameters: $\kernelsize=5$ and $\maxdist=+\infty$.}
\end{figure}

\begin{figure}[h]
	\begin{center}
		\includegraphics[scale=0.35]{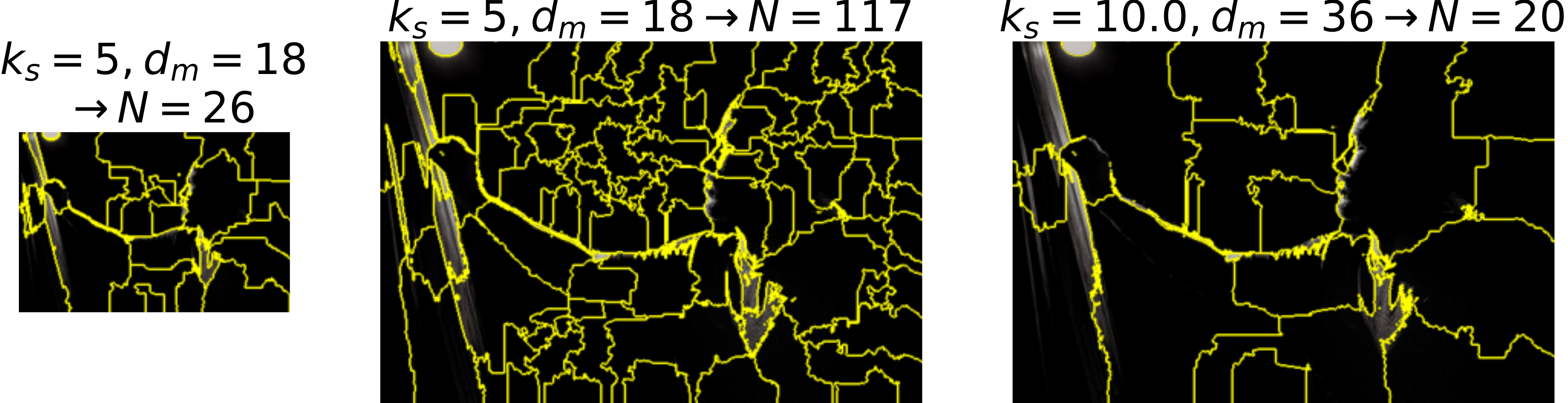} \\
		\vspace{0.2cm}
		\includegraphics[scale=0.35]{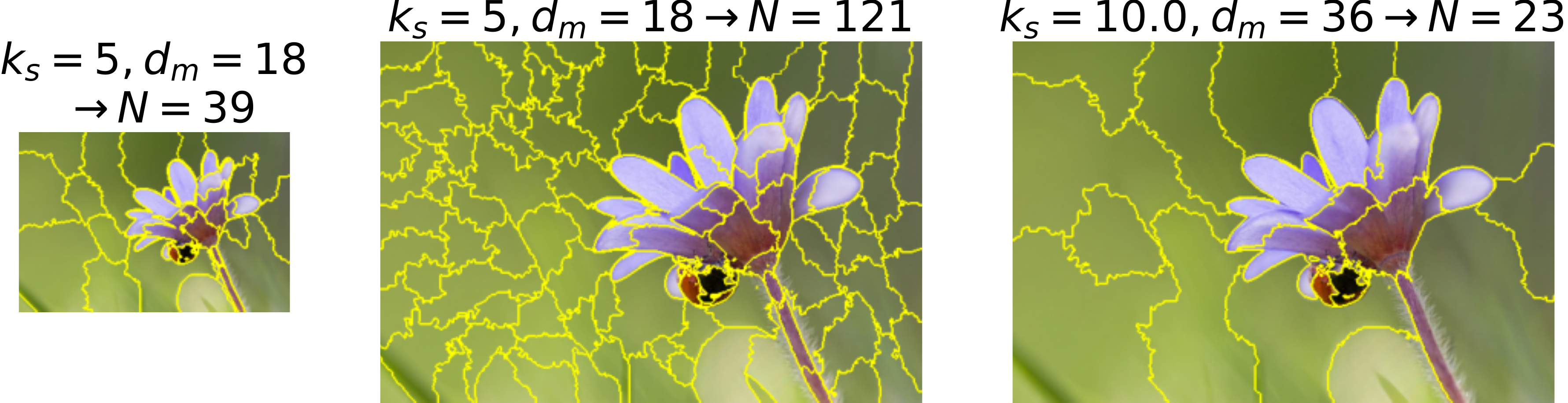} \\
		\vspace{0.2cm}
		\includegraphics[scale=0.35]{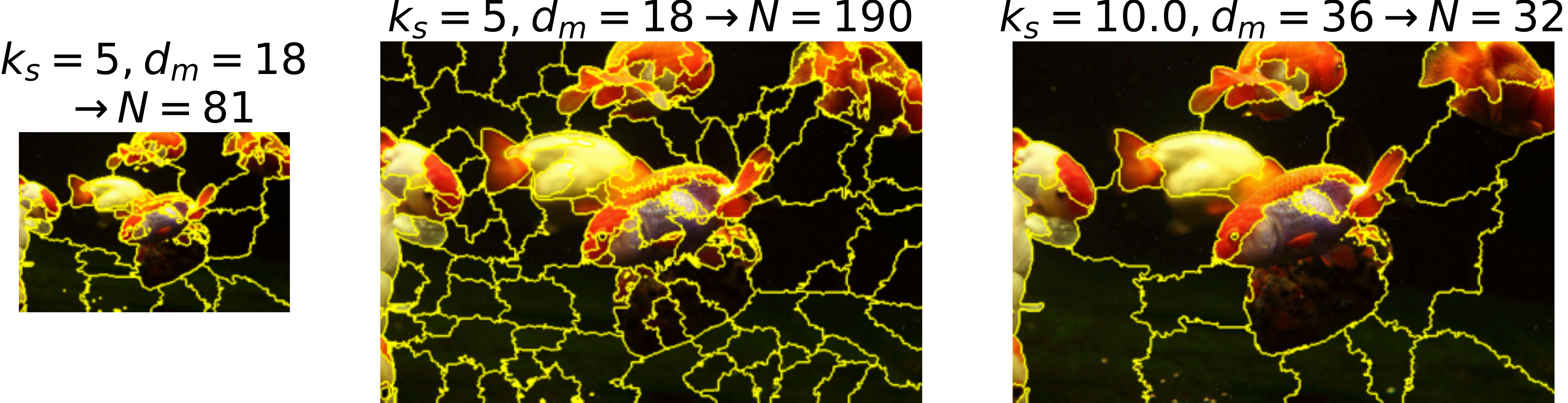} \\
		\vspace{0.2cm}
		\includegraphics[scale=0.35]{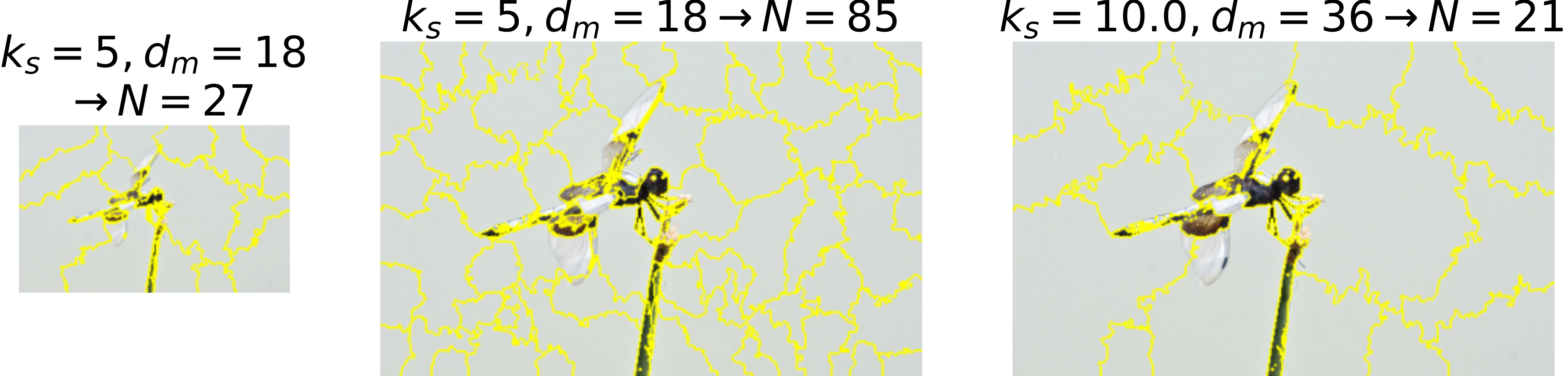} \\
		\vspace{0.2cm}
		\includegraphics[scale=0.35]{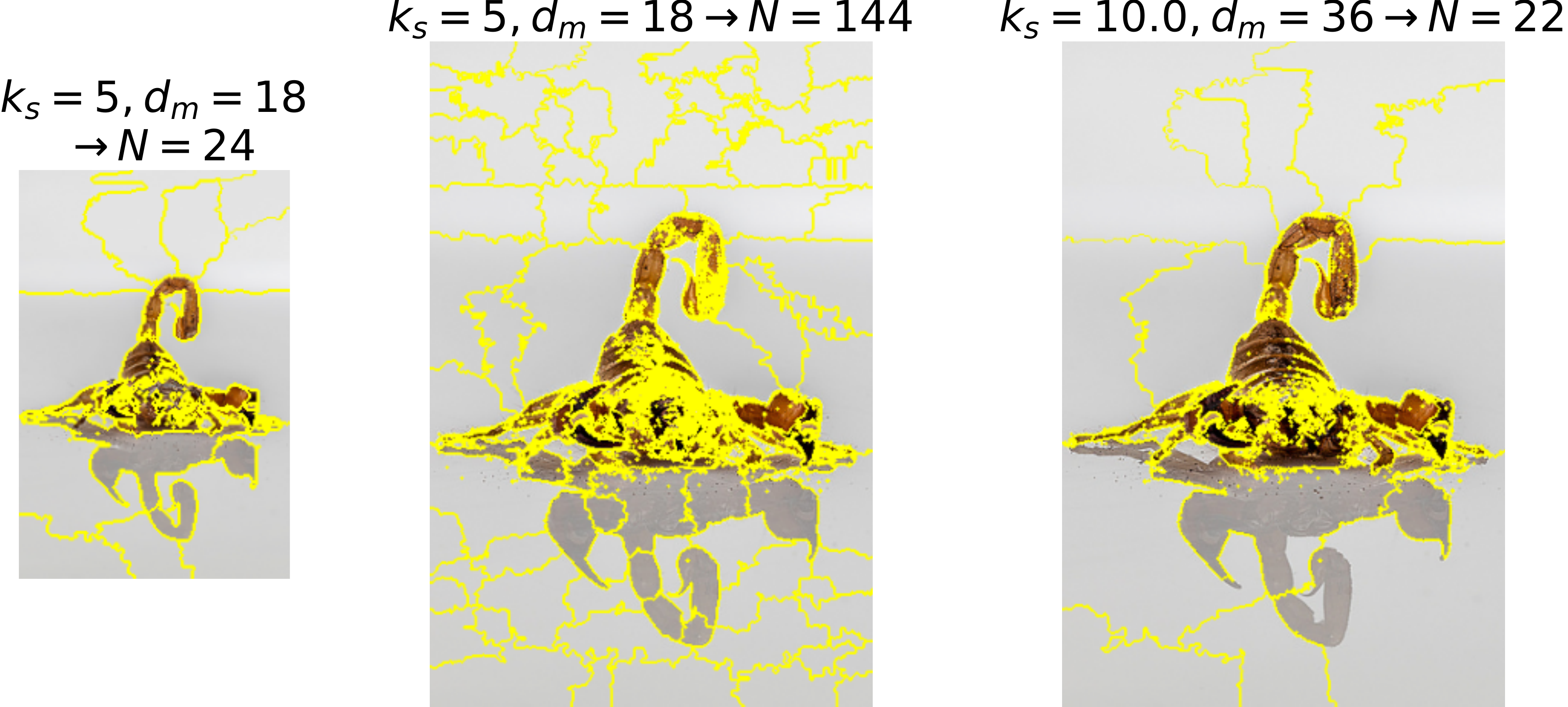} \\
		\vspace{0.2cm}
		\includegraphics[scale=0.35]{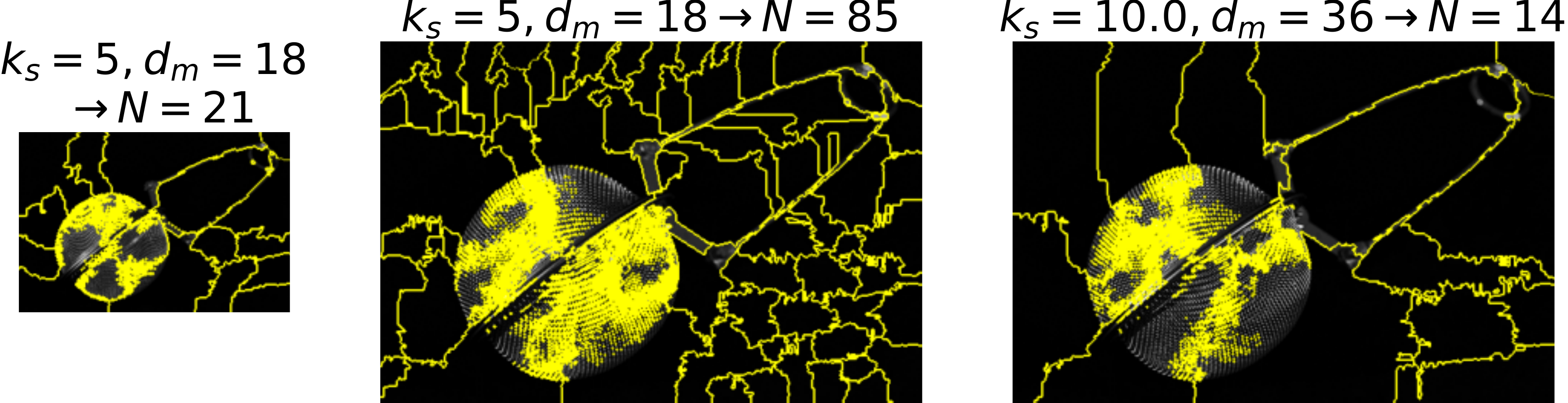}
	\end{center}
	\vspace{-0.1in}
	\caption{\label{fig:use-case-2}Additional results for the rescaling experiments. Original hyperparameters: $\kernelsize=5$ and $\maxdist=18$.}
\end{figure}

\begin{figure}[h]
	\begin{center}
		\includegraphics[scale=0.35]{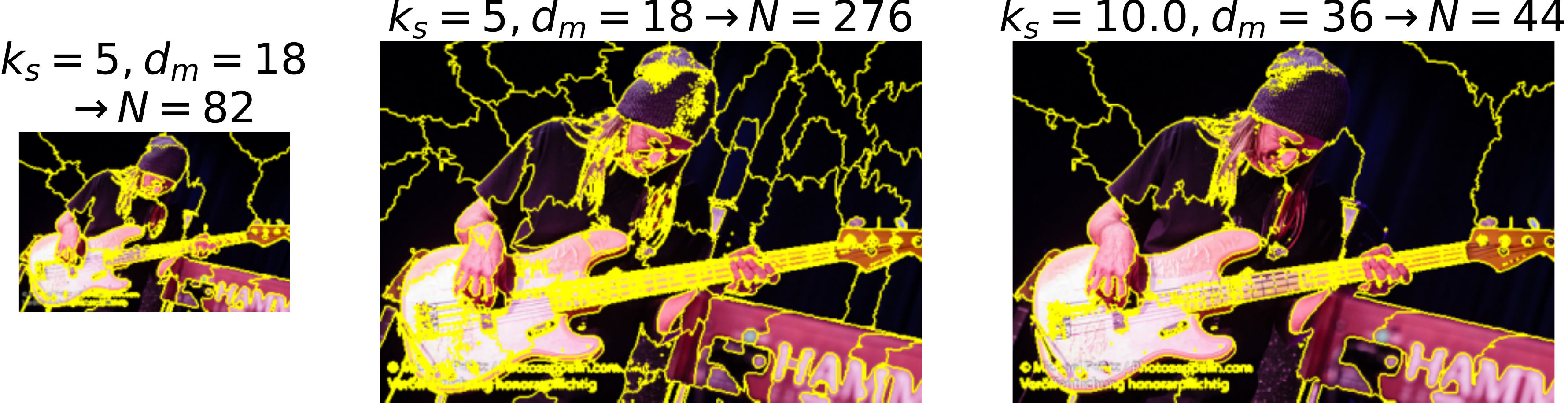} \\
		\vspace{0.2cm}
		\includegraphics[scale=0.35]{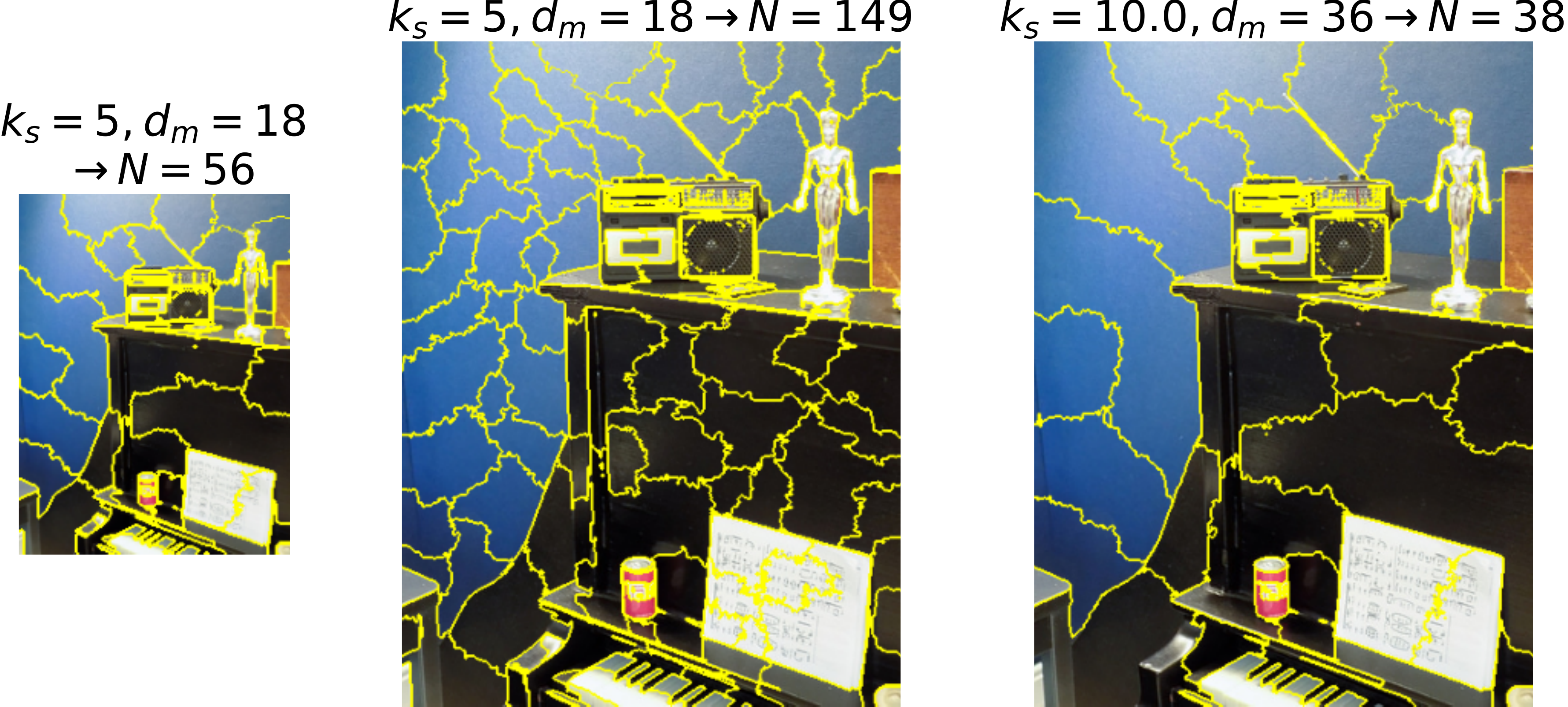} \\
		\vspace{0.2cm}
		\includegraphics[scale=0.35]{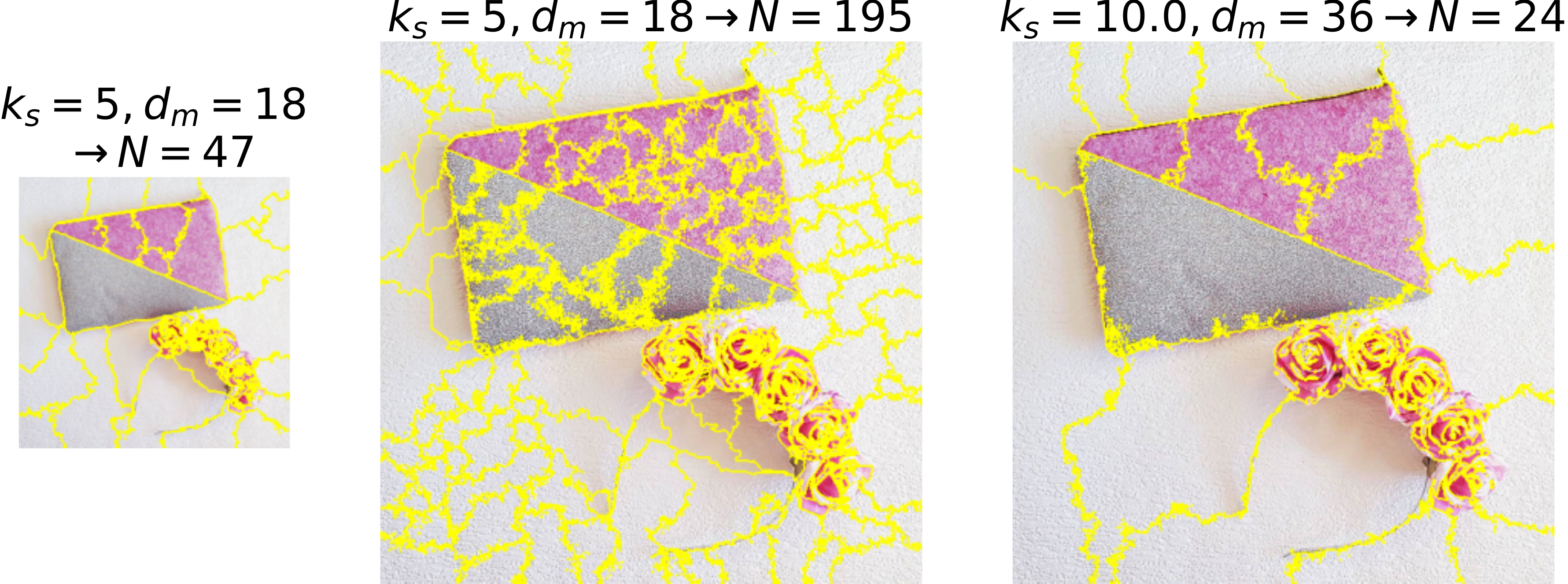} \\
		\vspace{0.2cm}
		\includegraphics[scale=0.35]{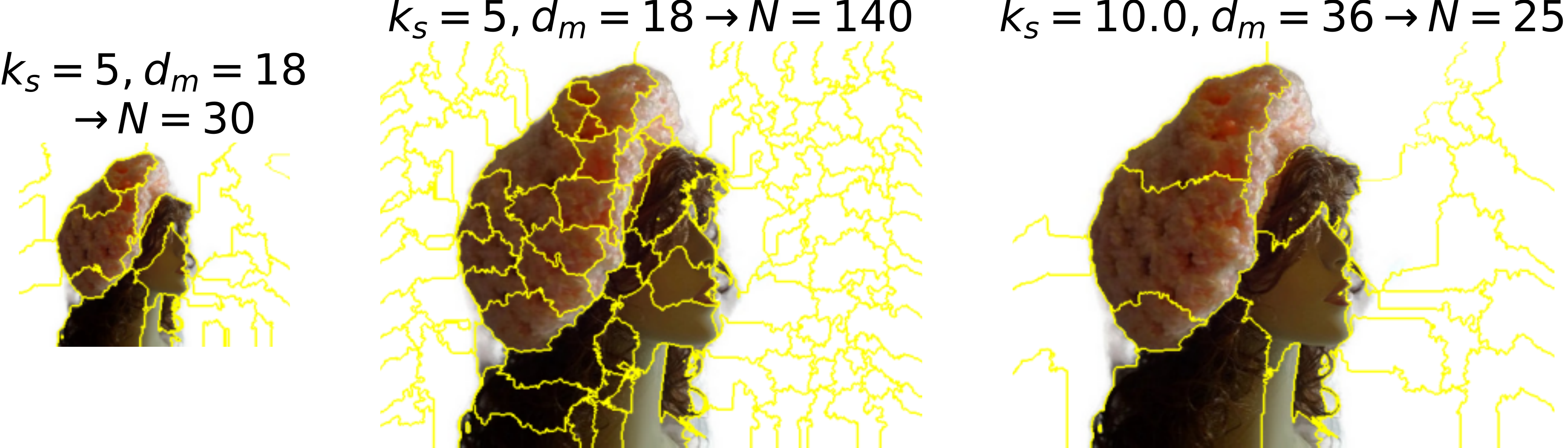} \\
		\vspace{0.2cm}
		\includegraphics[scale=0.35]{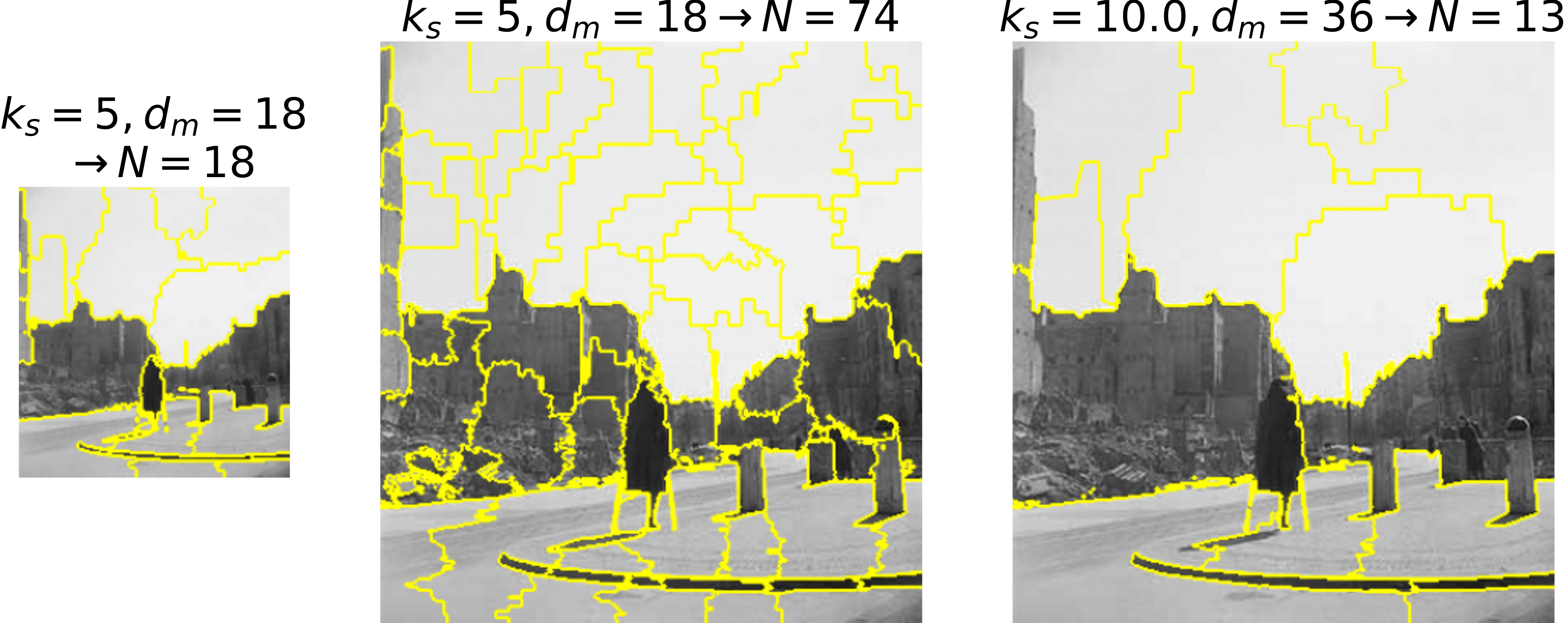} 
	\end{center}
	\vspace{-0.1in}
	\caption{\label{fig:use-case-3}Additional results for the rescaling experiments. Original hyperparameters: $\kernelsize=5$ and $\maxdist=10$.}
\end{figure}

%%%%%%%%%%%%%%%%%%%%%%%%%%%%%%%%%%%%%%%%%%%%%%%%%%%%%%%%%%%%%%%%%%%%%%%%%%%%%%%%%%%

\end{document}